\documentclass[11pt]{article}

\usepackage[margin=1in]{geometry}

\usepackage{microtype}
\usepackage{graphicx}
\usepackage{subcaption}
\usepackage{booktabs} % for professional tables
\usepackage{colortbl} % for table row colors

\usepackage{hyperref}
\usepackage{url}

\usepackage{natbib}

\usepackage{amsmath}
\usepackage{amssymb}
\usepackage{mathtools}
\usepackage{amsthm}
\usepackage{bm}  % For bold math symbols
\usepackage{algorithm}
\usepackage{algorithmic}
\usepackage{multirow}
\usepackage{float}  % For [H] positioning parameter
\usepackage{tikz}
\usepackage{tcolorbox}
\tcbuselibrary{skins}

\newtcolorbox{pinkbox}{
    colback=pink!15,
    colframe=pink!15,
    arc=8pt,
    boxrule=0pt,
    left=4pt, right=4pt, top=4pt, bottom=4pt
}
\newtcolorbox{greenbox}{
    colback=green!10,
    colframe=green!10,
    arc=8pt,
    boxrule=0pt,
    left=4pt, right=4pt, top=4pt, bottom=4pt
}
\newtcolorbox{yellowbox}{
    colback=yellow!12,
    colframe=yellow!12,
    arc=8pt,
    boxrule=0pt,
    left=4pt, right=4pt, top=4pt, bottom=4pt
}

\newcommand{\RETURN}{\STATE \textbf{return} }

\usepackage[capitalize,noabbrev]{cleveref}

\theoremstyle{plain}
\newtheorem{theorem}{Theorem}[section]
\newtheorem{proposition}[theorem]{Proposition}
\newtheorem{lemma}[theorem]{Lemma}

\theoremstyle{definition}
\newtheorem{definition}[theorem]{Definition}

\theoremstyle{remark}

\newcommand{\mat}[1]{\mathbf{#1}}

\renewcommand{\vec}[1]{\bm{#1}}

\begin{document}

\title{Noise is All You Need: Solving Linear Inverse Problems by Noise Combination Sampling with Diffusion Models}

\author{
Xun Su\thanks{Corresponding author: \texttt{suxun\_opt@asagi.waseda.jp}} \\
Waseda University, Tokyo, Japan
\and
Hiroyuki Kasai\thanks{\texttt{hiroyuki.kasai@waseda.jp}} \\
Waseda University, Tokyo, Japan
}

\date{}

\maketitle

\begin{abstract}
    Pretrained diffusion models have demonstrated strong capabilities in zero-shot inverse problem solving by incorporating observation information into the generation process of the diffusion models. However, this presents an inherent dilemma: excessive integration can disrupt the generative process, while insufficient integration fails to emphasize the constraints imposed by the inverse problem. To address this, we propose \emph{Noise Combination Sampling}, a novel method that synthesizes an optimal noise vector from a noise subspace to approximate the measurement score, replacing the noise term in the standard Denoising Diffusion Probabilistic Models process. This enables conditional information to be naturally embedded into the generation process without reliance on step-wise hyperparameter tuning. Our method can be applied to a wide range of inverse problem solvers, including image compression, and, particularly when the number of generation steps $T$ is small, achieves superior performance with negligible computational overhead, significantly improving robustness and stability.
\end{abstract}

\begin{figure*}[t]
    \centering
    \includegraphics[width=1.0\linewidth]{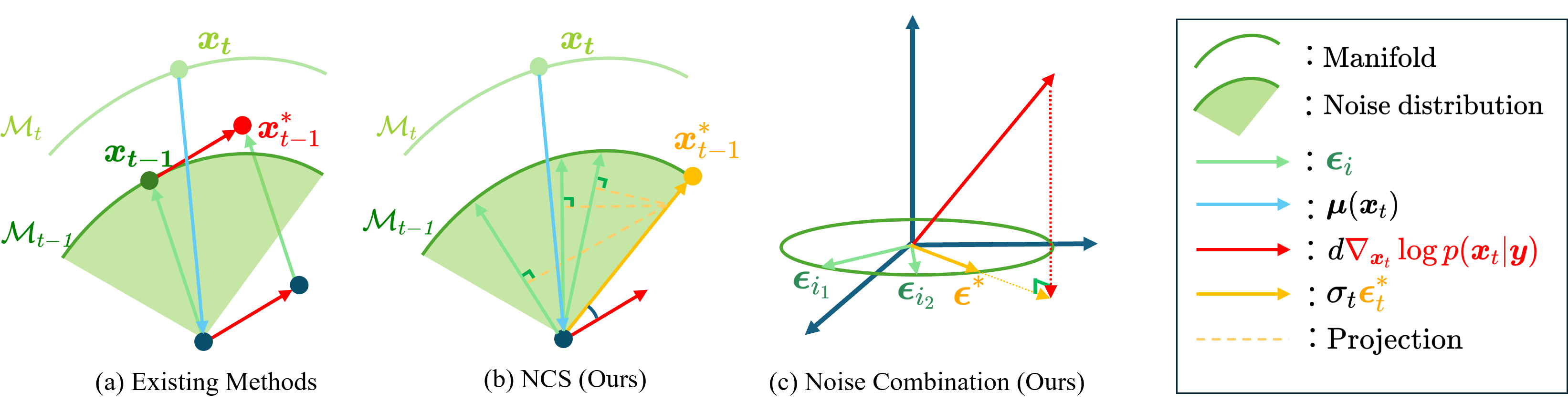}
    \vspace{-0.5cm}
    \caption{An illustration showing the difference between exact approximation methods and NCS. The intervention, i.e., the \textcolor{red}{measurement score} in the existing methods pushes the trajectory off the manifold $\mathcal{M}_{t-1}$ of $\vec{x}_{t-1}$. In contrast, NCS embeds the intervention into the optimal noise within an ellipsoidal subspace, defined by the span of the noise codebook. This allows NCS to naturally preserve both the position of $\vec{x}_{t-1}$ on its manifold and the consistency of the diffusion process.}
    \label{fig:illustration}
\end{figure*}

\section{Introduction}
Diffusion models (DMs) have achieved state-of-the-art results in image generation, audio synthesis, video modeling, and language modeling~\citep{ho2020denoisingdiffusionprobabilisticmodels, rombach2022highresolutionimagesynthesislatent, podell2024sdxl, ho2022video, nie2025largelanguagediffusionmodels, gat2024discrete}, and show strong zero-shot potential for tasks such as inpainting, depth estimation, and segmentation~\citep{lugmayr2022repaint, tian2024diffuse, li2023your}. Without additional training, DMs can be used to solve various inverse problems by injecting observational information during the stochastic denoising process~\citep{wang2022zero, chung2023diffusion, cardosomonte, dou2024diffusion, he2024manifold, kim2025regularization}. Such powerful plug-and-play approaches have further been extended to more specialized settings, including deep learning-guided generation tasks such as style transfer, image editing~\citep{yu2023freedom, he2024manifold, shi2024dragdiffusion, ye2024tfg}, and scientific inverse problems~\citep{zheng2025inversebench}. However, existing methods often suffer from sampling instability and therefore rely on complex hyperparameter tuning and long sampling schedules to compensate for the limited effectiveness of the guidance~\citep{wang2022zero, song2023solving, yang2024guidance, zhang2025improving, zheng2025inversebench}.

We propose \emph{Noise Combination Sampling} (NCS), a framework that approximates the measurement score $\nabla_{\vec{x}_t} \log p(\vec{y} \mid \vec{x}_t)$ via a linear combination of Gaussian noise vectors, replacing the noise term in the Denoising Diffusion Probabilistic Models (DDPM) process. By adjusting the sampling trajectory through synthesized noise vectors rather than explicit gradients, the plug-and-play NCS stably improves mainstream diffusion-based inverse problem solvers such as Diffusion Posterior Sampling (DPS)~\citep{chung2023diffusion} and Manifold-Preserving Gradient Descent (MPGD)~\citep{he2024manifold}, with negligible computational overhead. Across a wide range of tasks, NCS consistently delivers substantial performance gains and, in several cases, surpasses the state of the art, while remaining robust under a single default setting.

Denoising Diffusion Codebook Models (DDCM)~\citep{ohayon2025compressed}, a generative image compression method, can be interpreted as a special case of NCS, in which one noise vector is selected from a noise codebook, corresponding to an extremely low-rate setting. To approximate the measurement score with more than one noise vectors, DDCM relies on a greedy search whose computational complexity grows exponentially with noise combination. In contrast, we show that NCS achieves competitive quantization performance with only linear complexity. Concurrently, Turbo-DDCM~\citep{vaisman2025turboddcmfastflexiblezeroshot} reports a similar observation and further improves compression effectiveness through engineering optimizations.

Our contributions are summarized as follows:
\begin{itemize}
    \item We show that, NCS formulates measurement score approximation as an inner-product minimization problem, which admits a closed-form solution that is nearly equivalent to cosine-direction fitting in high-dimensional settings with a relatively limited noise basis, enabling stable and affordable conditional control of the DDPM process.

    \item We extend NCS to generative image compression and show that DDCM~\citep{ohayon2025compressed} is a special case of NCS. By combining multiple noise vectors from a large codebook and using significantly fewer diffusion steps, NCS achieves competitive quantization quality with linear-time complexity in $m$, enabling fast compression and decompression.

    \item We conduct extensive experiments on both compression and inverse problems. For compression, NCS attains comparable rate--distortion performance at the same bitrate using roughly $1/10$ of the diffusion steps $T$. For inverse problems, we show that augmenting DPS, a canonical and widely used inverse problem solver, with NCS leads to consistent and substantial performance improvements across diverse inverse problems and sampling regimes, without introducing additional hyperparameters or computational overhead, and with particularly strong gains in low-step settings.

\end{itemize}

\begin{figure*}[!t]
    \centering
    % Top row: Motion Blur and Gaussian Deblur side by side (pink background)
    \begin{pinkbox}
    \centering
    % % First row: Inpainting and SR x4
    % \begin{minipage}[t]{0.47\textwidth}
    %     \centering
    %     \includegraphics[width=\linewidth]{fig/dps/inp.png}
    %     \small (a) Inpainting
    % \end{minipage}
    % \hfill
    % \begin{minipage}[t]{0.47\textwidth}
    %     \centering
    %     \includegraphics[width=\linewidth]{fig/dps/srx4.png}
    %     \small (b) SR $\times$4
    % \end{minipage}

    % Second row: Motion Blur and Gaussian Deblur
    \begin{minipage}[t]{0.47\textwidth}
        \centering
        \includegraphics[width=\linewidth]{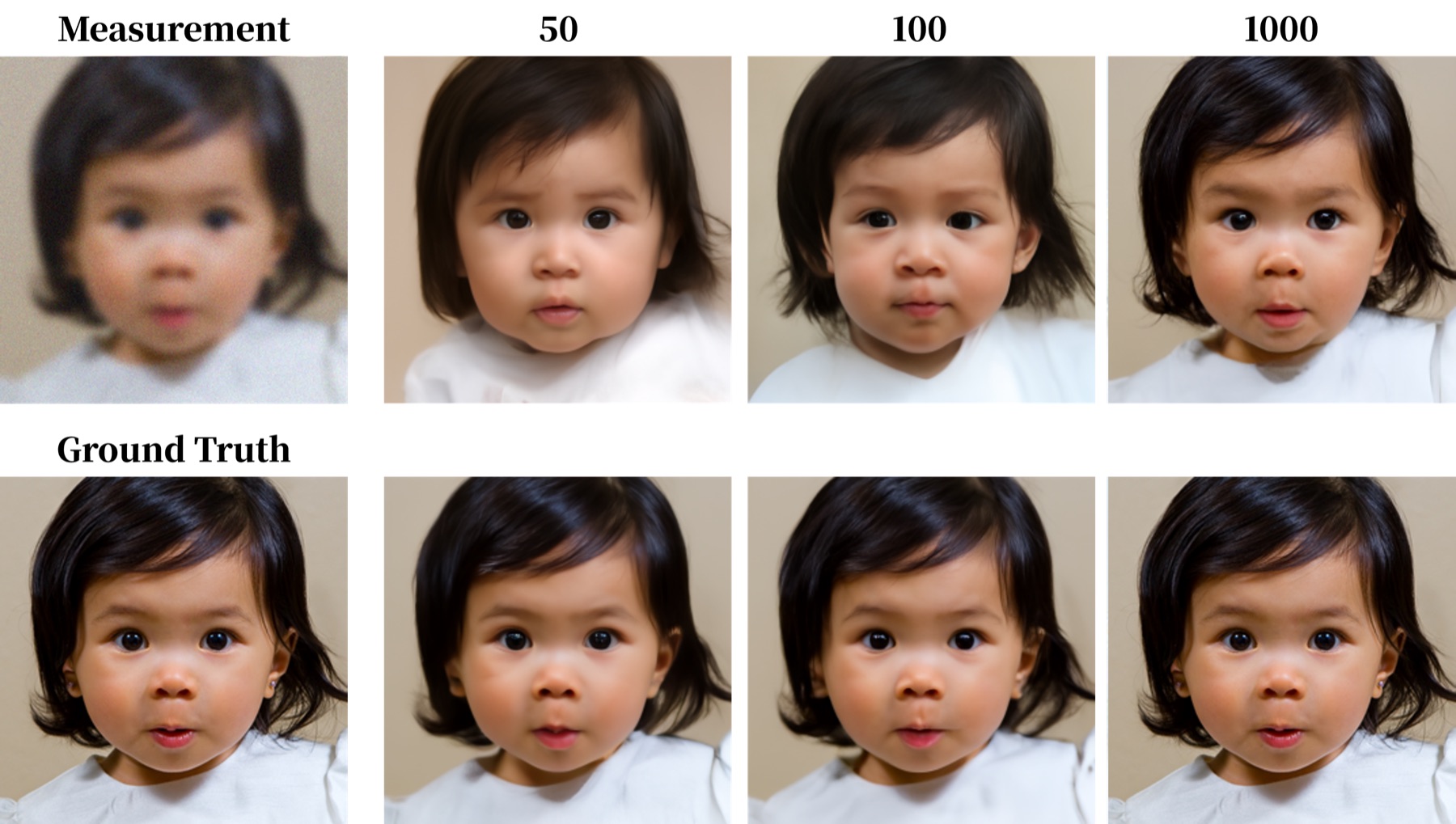}
        \small (a) Motion Blur
    \end{minipage}
    \hfill
    \begin{minipage}[t]{0.47\textwidth}
        \centering
        \includegraphics[width=\linewidth]{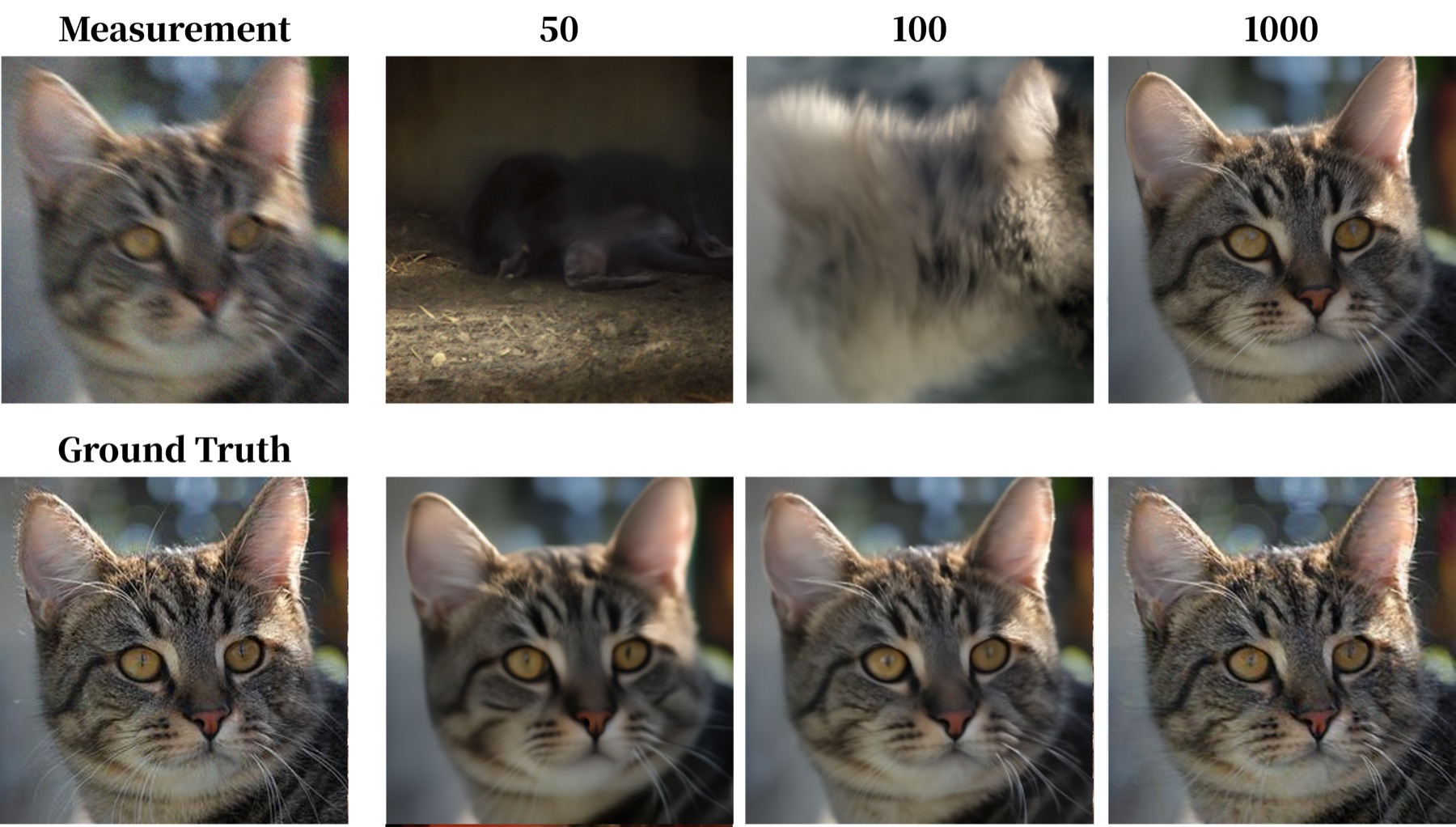}
        \small (b) Gaussian Deblur
    \end{minipage}
    \end{pinkbox}
    % Bottom: Three scientific inverse problems (green background)
    \begin{greenbox}
    \centering
    \setlength{\tabcolsep}{4pt}
    \begin{tabular}{c@{\hspace{12pt}}c@{\hspace{12pt}}cccc}
        & \small\textbf{GT} & \small\textbf{DPS} & \small\textbf{DAPS} & \small\textbf{REDDiff} & \small\textcolor{red}{\textbf{NCS-DPS (Ours)}} \\[1pt]
        \raisebox{-.35\height}{\footnotesize\textcolor{orange}{\textbf{\shortstack{Black Hole\\Imaging}}}}\hspace{6pt} & 
        \raisebox{-.5\height}{\includegraphics[width=0.12\textwidth]{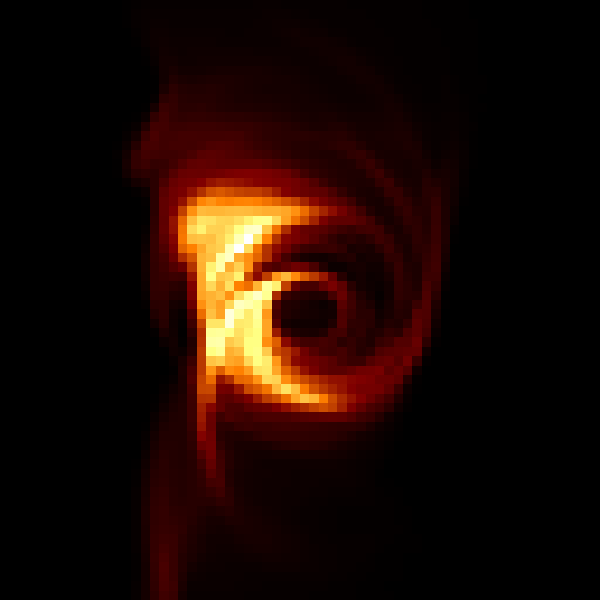}} &
        \raisebox{-.5\height}{\includegraphics[width=0.12\textwidth]{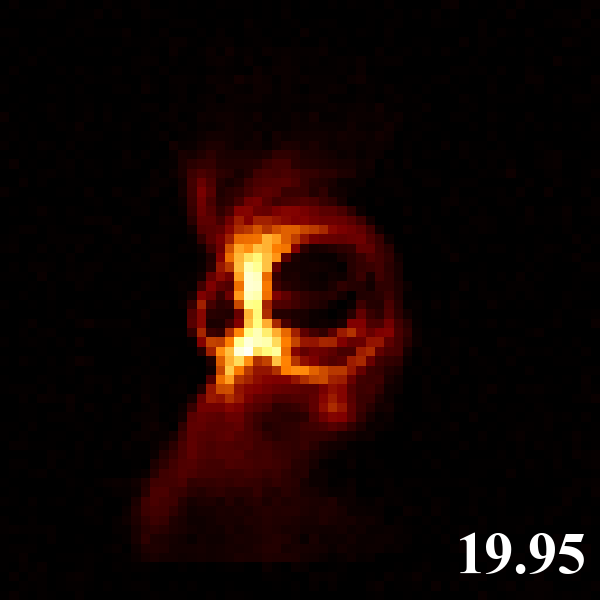}} &
        \raisebox{-.5\height}{\includegraphics[width=0.12\textwidth]{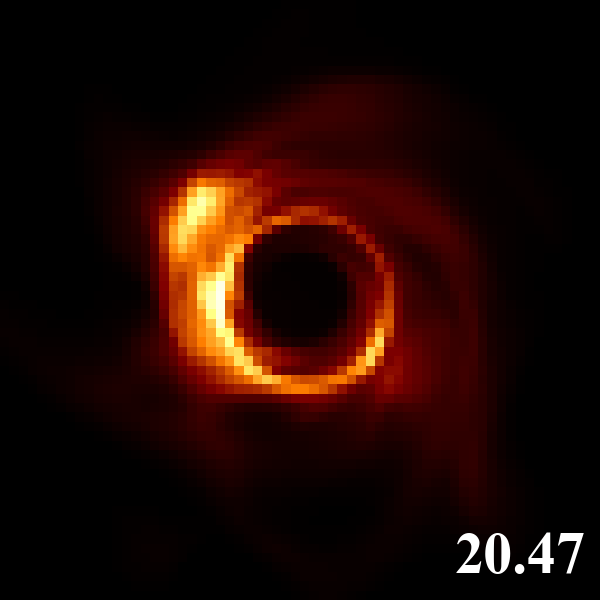}} &
        \raisebox{-.5\height}{\includegraphics[width=0.12\textwidth]{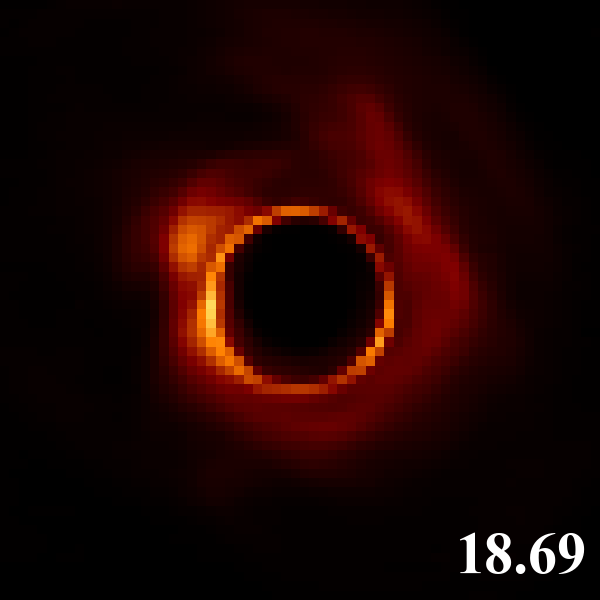}} &
        \raisebox{-.5\height}{\includegraphics[width=0.12\textwidth]{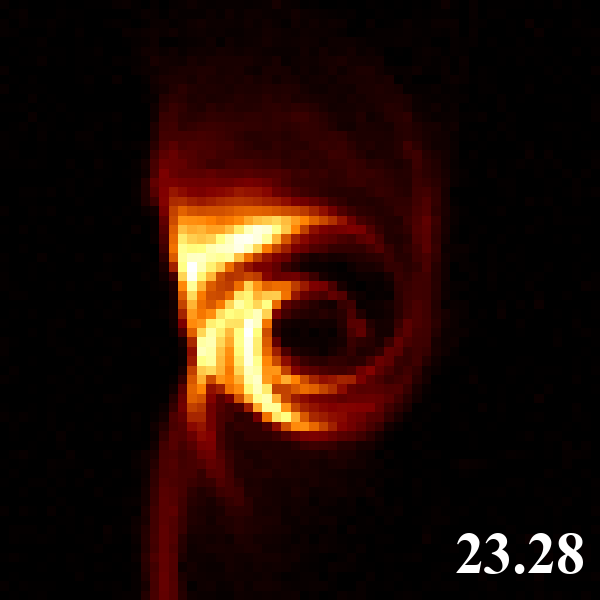}} \\
        \noalign{\vspace{5pt}}
        \raisebox{-.35\height}{\footnotesize\textcolor{violet}{\textbf{\shortstack{Inverse\\Scattering}}}}\hspace{6pt} & 
        \raisebox{-.5\height}{\includegraphics[width=0.12\textwidth]{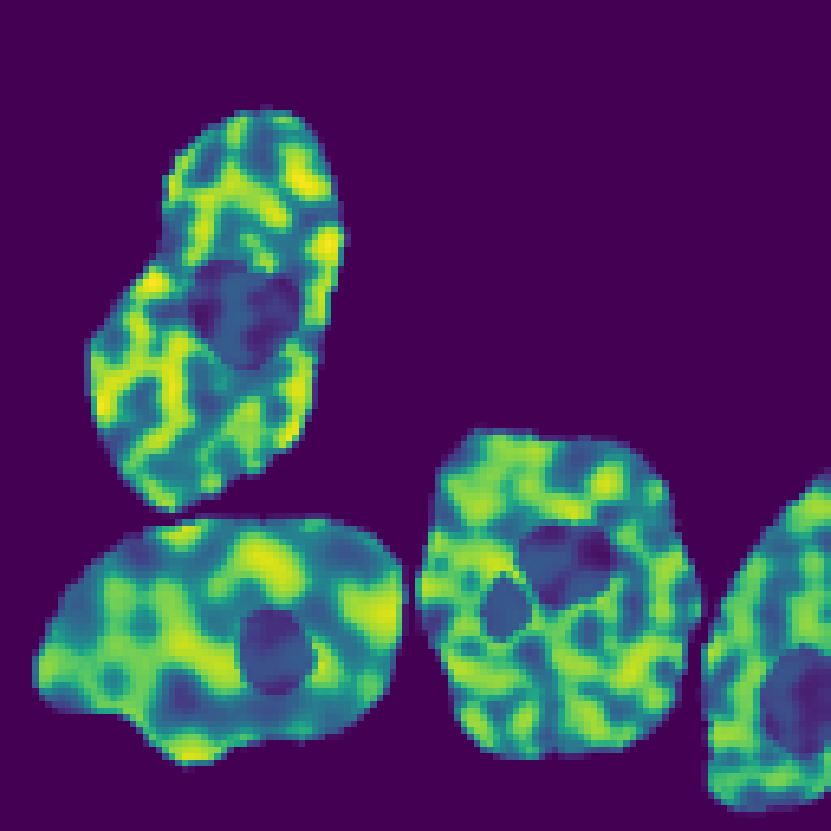}} &
        \raisebox{-.5\height}{\includegraphics[width=0.12\textwidth]{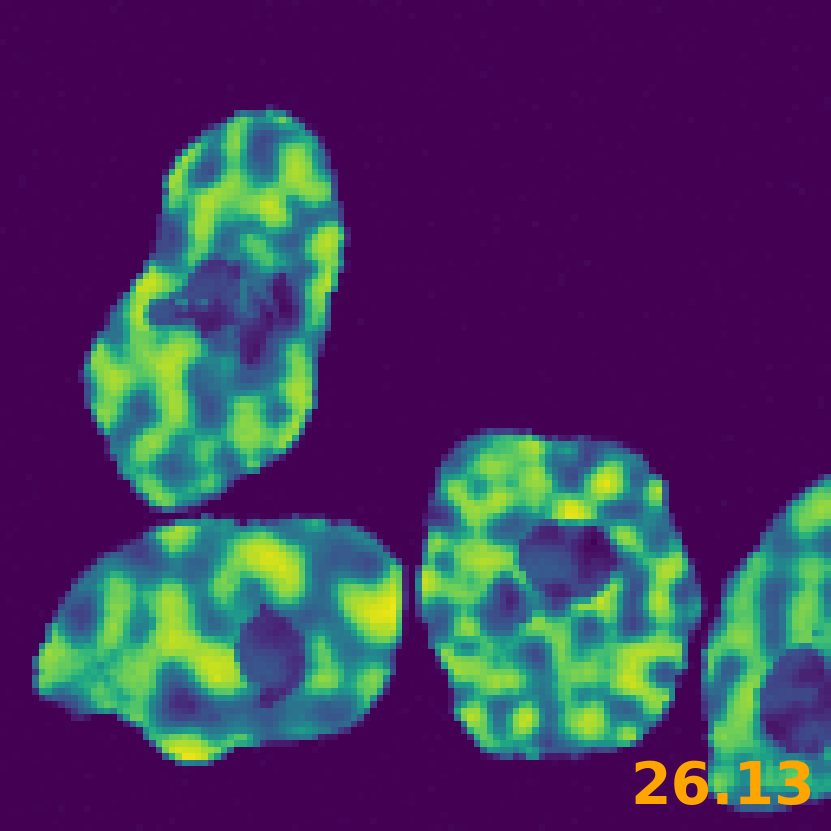}} &
        \raisebox{-.5\height}{\includegraphics[width=0.12\textwidth]{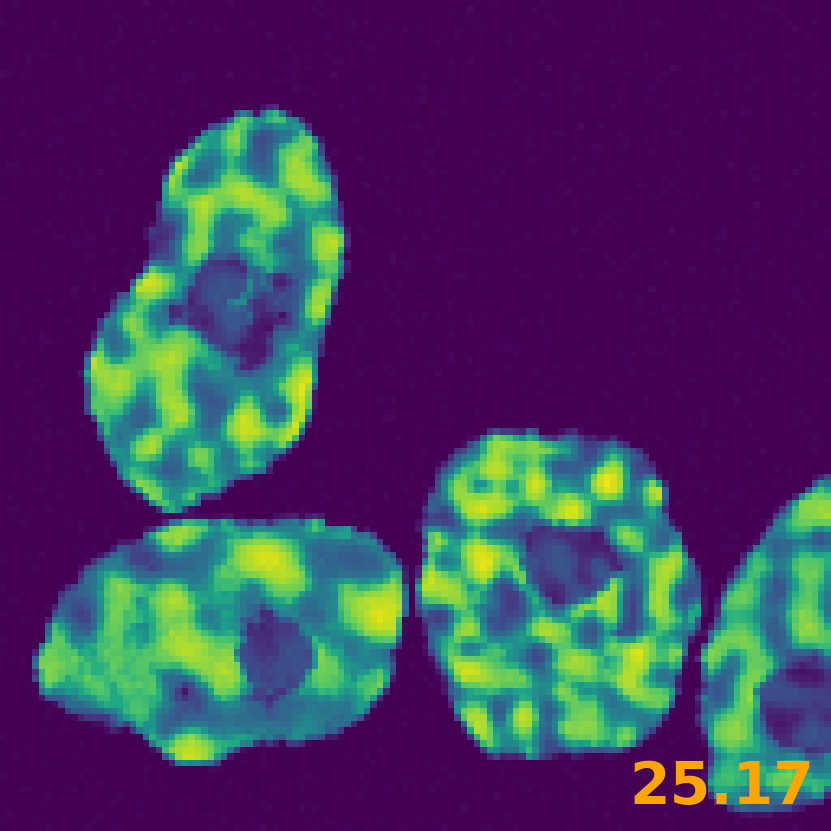}} &
        \raisebox{-.5\height}{\includegraphics[width=0.12\textwidth]{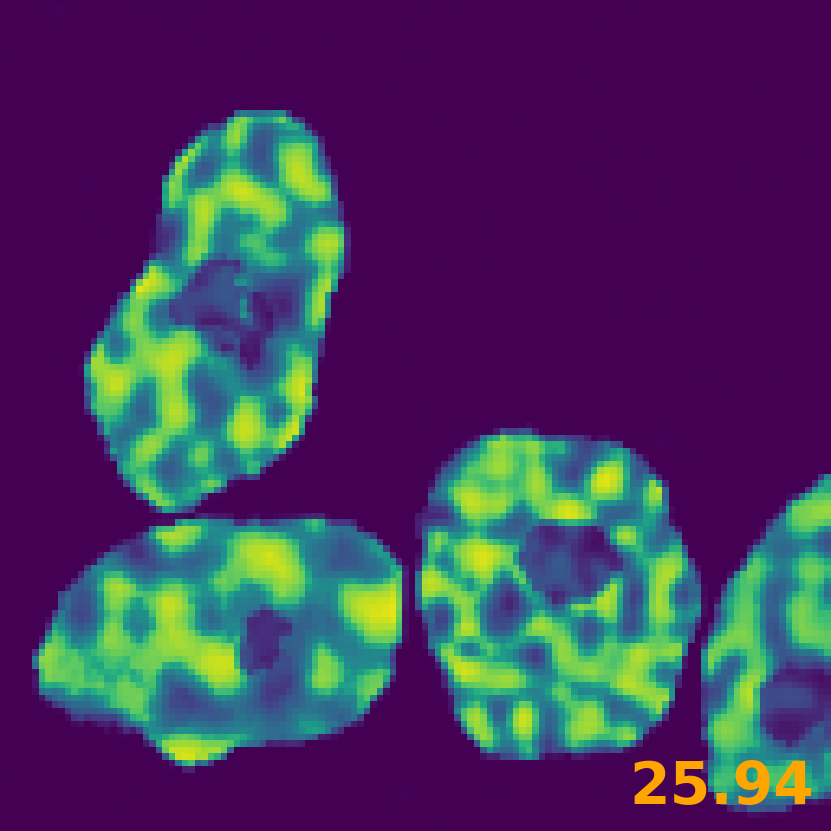}} &
        \raisebox{-.5\height}{\includegraphics[width=0.12\textwidth]{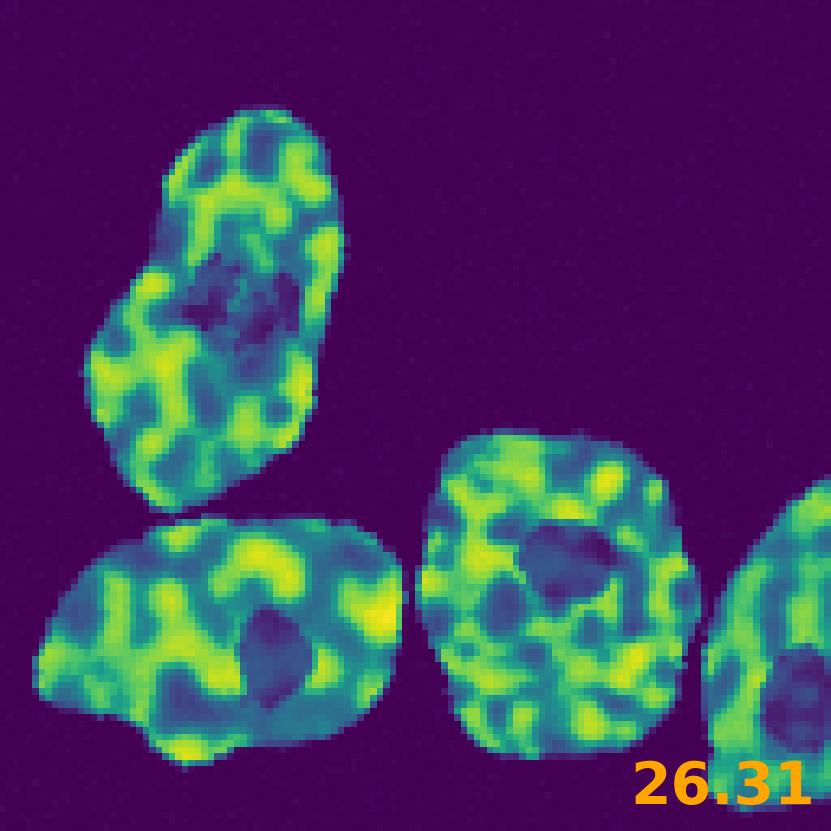}} \\
        \noalign{\vspace{5pt}}
        \raisebox{-.35\height}{\footnotesize\textcolor{teal}{\textbf{\shortstack{Compressed\\Sensing\\MRI}}}}\hspace{6pt} & 
        \raisebox{-.5\height}{\includegraphics[width=0.12\textwidth]{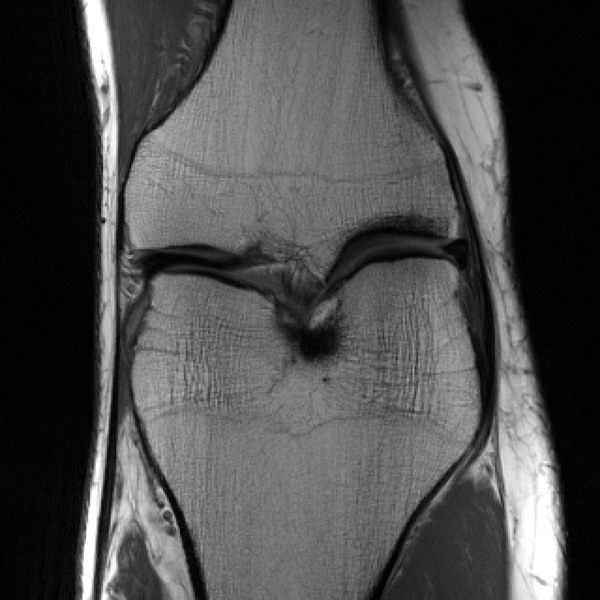}} &
        \raisebox{-.5\height}{\includegraphics[width=0.12\textwidth]{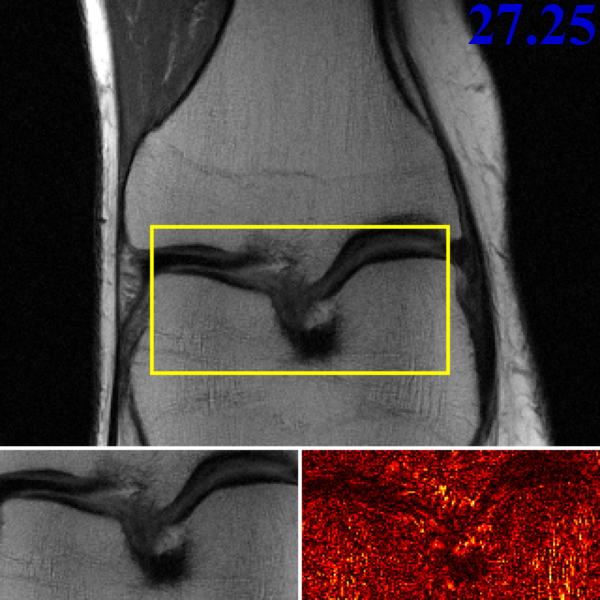}} &
        \raisebox{-.5\height}{\includegraphics[width=0.12\textwidth]{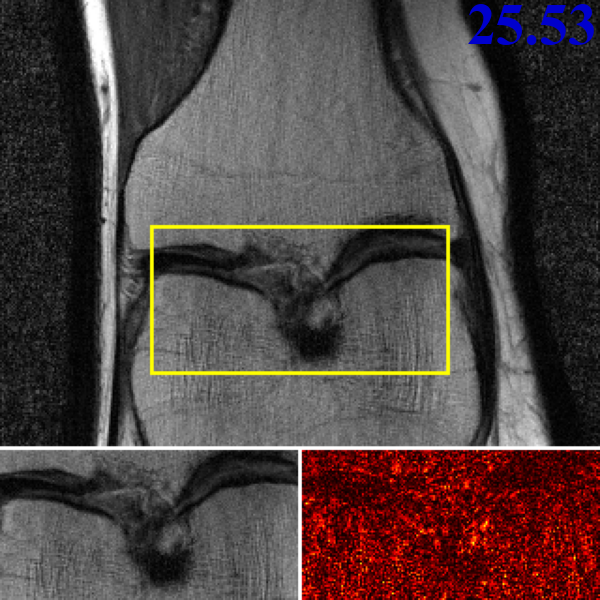}} &
        \raisebox{-.5\height}{\includegraphics[width=0.12\textwidth]{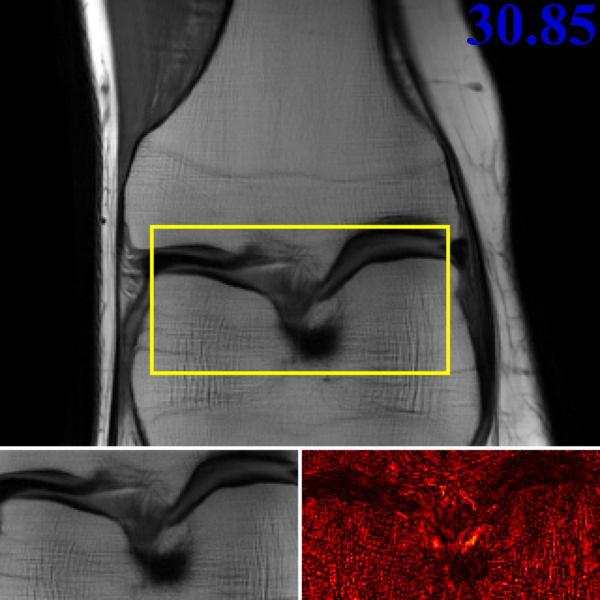}} &
        \raisebox{-.5\height}{\includegraphics[width=0.12\textwidth]{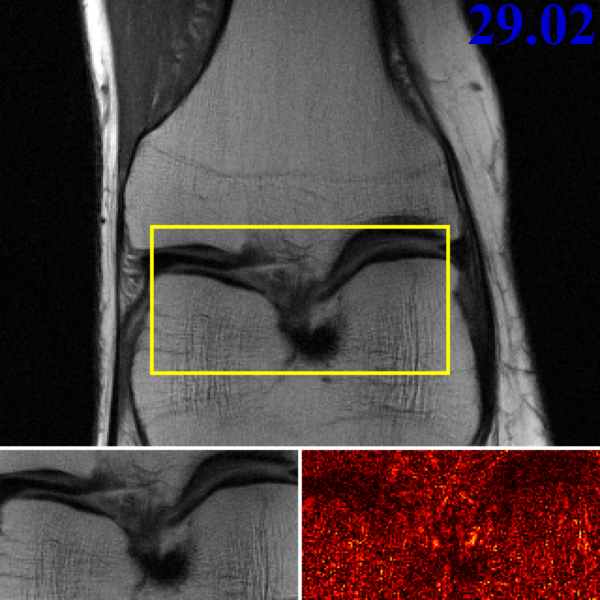}} \\
    \end{tabular}
    \end{greenbox}
    \caption{Visual comparison: \textbf{Top}: DPS vs.\ NCS-DPS on two inverse problems (motion blur, Gaussian deblur) under varying sampling steps. \textbf{Bottom}: Comparison across three scientific inverse problems with different methods.}
    \label{fig:main_results}
\end{figure*}

\section{Background}

\subsection{Denoising Diffusion Probabilistic Models (DDPMs)}

DMs define the generative process as the reverse of a predefined noising process. Following the formulation of~\citet{songscore}, we describe the forward (noising) process using an Itô stochastic differential equation (SDE), where $\vec{x}_t \in \mathbb{R}^d$ and $t \in [0, T]$:
\[
    d\vec{x}_t = f(\vec{x}_t, t)\,dt + g(t)\,d\vec{w}_t,
    \]
where $f$ and $g$ denote the drift function and diffusion coefficient, respectively, and $\vec{w}_t$ is a standard Wiener process.

With $\vec{x}_0 \sim p_{\text{data}}$ and $\vec{x}_T \sim \mathcal{N}(0, \mathbf{I})$, the reverse process recovers samples by solving~\citep{anderson1982reverse}:
\[
    d\vec{x}_t = \left( f(\vec{x}_t, t) - g^2(t)\nabla_{\vec{x}_t} \log p_t(\vec{x}_t) \right) dt + g(t)\,d\vec{w}_t,
\]
initialized at $\vec{x}_T \sim \mathcal{N}(0, \mathbf{I})$.

We follow \citet{songscore}'s definition to choose a Variance-Preserving (VP)-SDE, or DDPM schedule to show the discrete update rule. Researchers use a neural network $\vec{s}_\theta(\vec{x}_t, t)$ to approximate the score function $\nabla_{\vec{x}_t} \log p_t(\vec{x}_t)$ and makes it possible to use the reverse process to generate the data. We consider the general condition by discretizing the whole process into T bins,
\begin{equation}
    \label{eq:discrete}
    \vec{x}_{t-1} = \vec{x}_t - f(\vec{x}_t, t) +  g^2(t) \vec{s}_\theta(\vec{x}_t, t) +  g(t) \vec{\epsilon}_t
\end{equation}
where $\vec{\epsilon}_t \sim \mathcal{N}(0, \mathbf{I})$. \citet{ho2020denoisingdiffusionprobabilisticmodels} consider the marginal distribution of $\vec{x}_t$ given $\vec{x}_0$ is Gaussian:
\[
    q(\vec{x}_t \mid \vec{x}_0) = \mathcal{N}\left(\vec{x}_t; \sqrt{\bar{\alpha}_t}\, \vec{x}_0,\ (1 - \bar{\alpha}_t)\, \mathbf{I}\right),
\]
where the noise schedule is defined via $\beta_t = g(t)^2 = {-2f(t)}$, $\alpha_t = 1 - \beta_t$, and $\bar{\alpha}_t = \prod_{s=1}^t \alpha_s$. This leads to the update rule of the DDPM process:
\begin{equation}
    \label{eq:ddpm_update}
    \vec{x}_{t-1} = \vec{\mu}(\vec{x}_t, t) + \sigma_t\, \vec{\epsilon}, \quad \vec{\epsilon} \sim \mathcal{N}(0, \mathbf{I}),
\end{equation}
where $\vec{\mu}(\vec{x}_t, t) = \frac{1}{\sqrt{\alpha_t}} \left( \vec{x}_t - \frac{\beta_t}{\sqrt{1 - \bar{\alpha}_t}}\, \vec{s}_\theta(\vec{x}_t, t) \right)$,
and $\sigma_t = \sqrt{\beta_t}$ is the variance parameter governing the stochasticity of the reverse process.

Tweedie's formula~\citep{kadkhodaie2021stochastic} estimates the original signal $\vec{x}_0$ from $\vec{x}_t$. In practice, DDPMs approximate this using the score network:
\[
    \tilde{\vec{x}}_{0 \mid t}(\vec{x}_t, t) = \mathbb{E}[\vec{x}_0 \mid \vec{x}_t] \approx \frac{1}{\sqrt{\bar{\alpha}_t}} \left( \vec{x}_t - \sqrt{1 - \bar{\alpha}_t}\, \vec{s}_\theta(\vec{x}_t, t) \right).
\]

\subsection{Linear Inverse Problems and Conditional Generation}

Conditional generation addresses scenarios where only partial observations or measurements $\vec{y} \in \mathbb{R}^n$, derived from the original signal $\vec{x}_0 \in \mathbb{R}^d$, are available. The corresponding \emph{inverse problem} is typically formulated as:
\[
    \vec{y} = \mathcal{A}(\vec{x}_0) + \vec{n}, \quad \vec{x}_0 \in \mathbb{R}^d,\ \vec{y}, \vec{n} \in \mathbb{R}^n, 
\]
where $\mathcal{A}: \mathbb{R}^d \rightarrow \mathbb{R}^n$ is a known linear degradation operator, and $\vec{n} \sim \mathcal{N}(0, \sigma^2 \mathbf{I})$ denotes additive white Gaussian noise. It covers a wide range of tasks \citep{tarantola2005inverse}.

Solving the conditional generation problem with a pretrained DM requires replacing the score function $\nabla_{\vec{x}_t} \log p_t(\vec{x}_t)$ in Eq.~\eqref{eq:discrete} with the conditional score function $\nabla_{\vec{x}_t} \log p_t(\vec{x}_t \mid \vec{y})$. Under a Bayesian formulation, the conditional distribution at $t$ is given by $ p(\vec{x}_t \mid \vec{y}) = {p(\vec{y} \mid \vec{x}_t) p(\vec{x}_t)}/{p(\vec{y})} $, which indicates that $\nabla_{\vec{x}_t} \log p(\vec{x}_t \mid \vec{y}) = \nabla_{\vec{x}_t} \log p(\vec{y} \mid \vec{x}_t) + \nabla_{\vec{x}_t} \log p(\vec{x}_t)$. Substituting the $\nabla_{\vec{x}_t} \log p(\vec{x}_t \mid \vec{y})$ into the discrete update rule in Eq.~\eqref{eq:discrete} introduces an additional term that is not learned by the score network. The modified discrete update becomes:
\begin{align}
    \label{eq:conditional_discrete}
    \vec{x}_{t-1} &= \vec{x}_t - f(\vec{x}_t, t) + g^2(t)\, \nabla_{\vec{x}_t} \log p_t(\vec{x}_t) + g(t)\, \vec{\epsilon}_t \nonumber \\
    & \quad + \textcolor{red}{g^2(t)\, \nabla_{\vec{x}_t} \log p_t(\vec{y} \mid \vec{x}_t)}.
\end{align}

\begin{figure}[t]
    \centering
    \begin{minipage}[c]{0.49\linewidth}
        \centering
        \begin{minipage}[t]{0.49\linewidth}
            \centering
            \includegraphics[width=\linewidth]{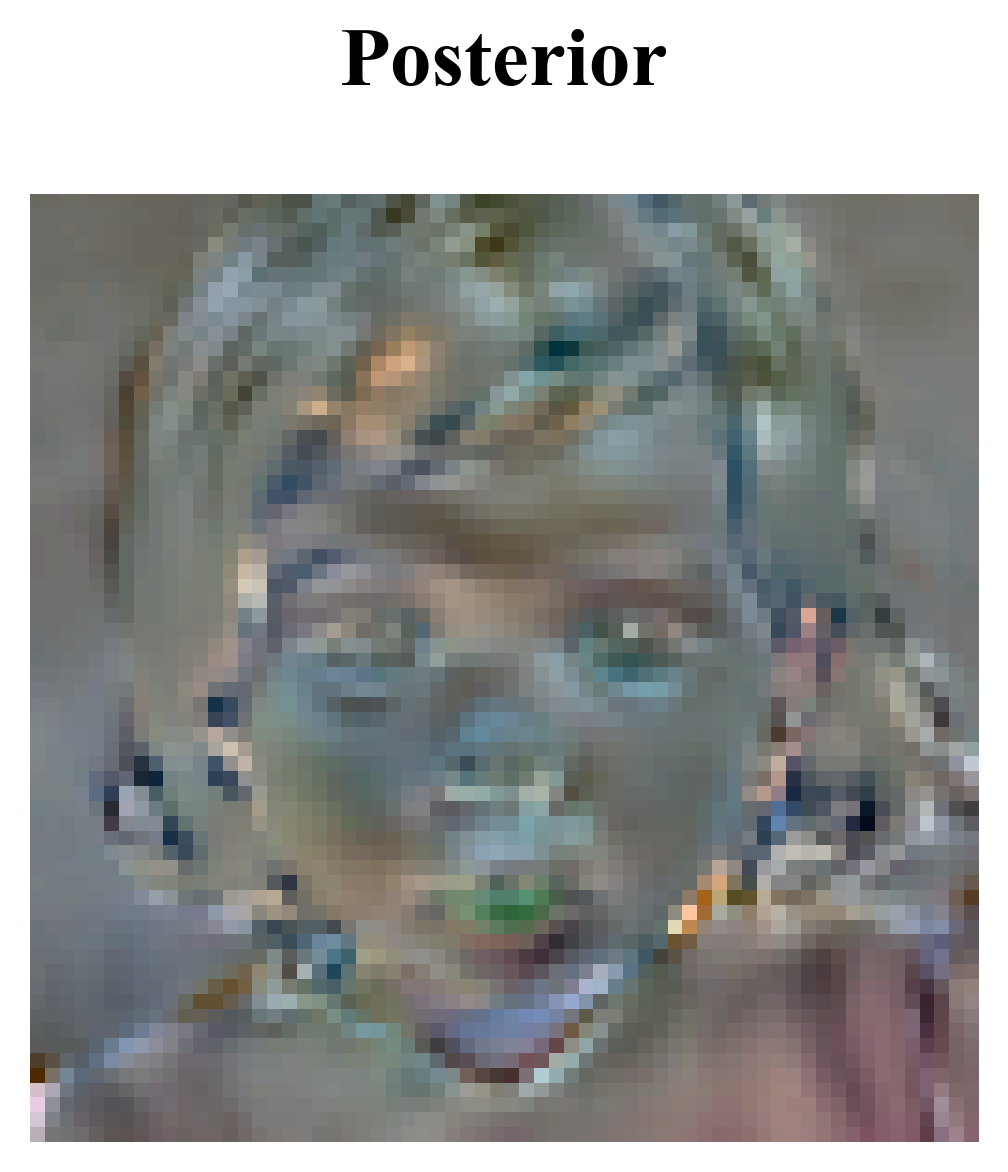}\\
            % \vspace{1mm}
            \includegraphics[width=\linewidth]{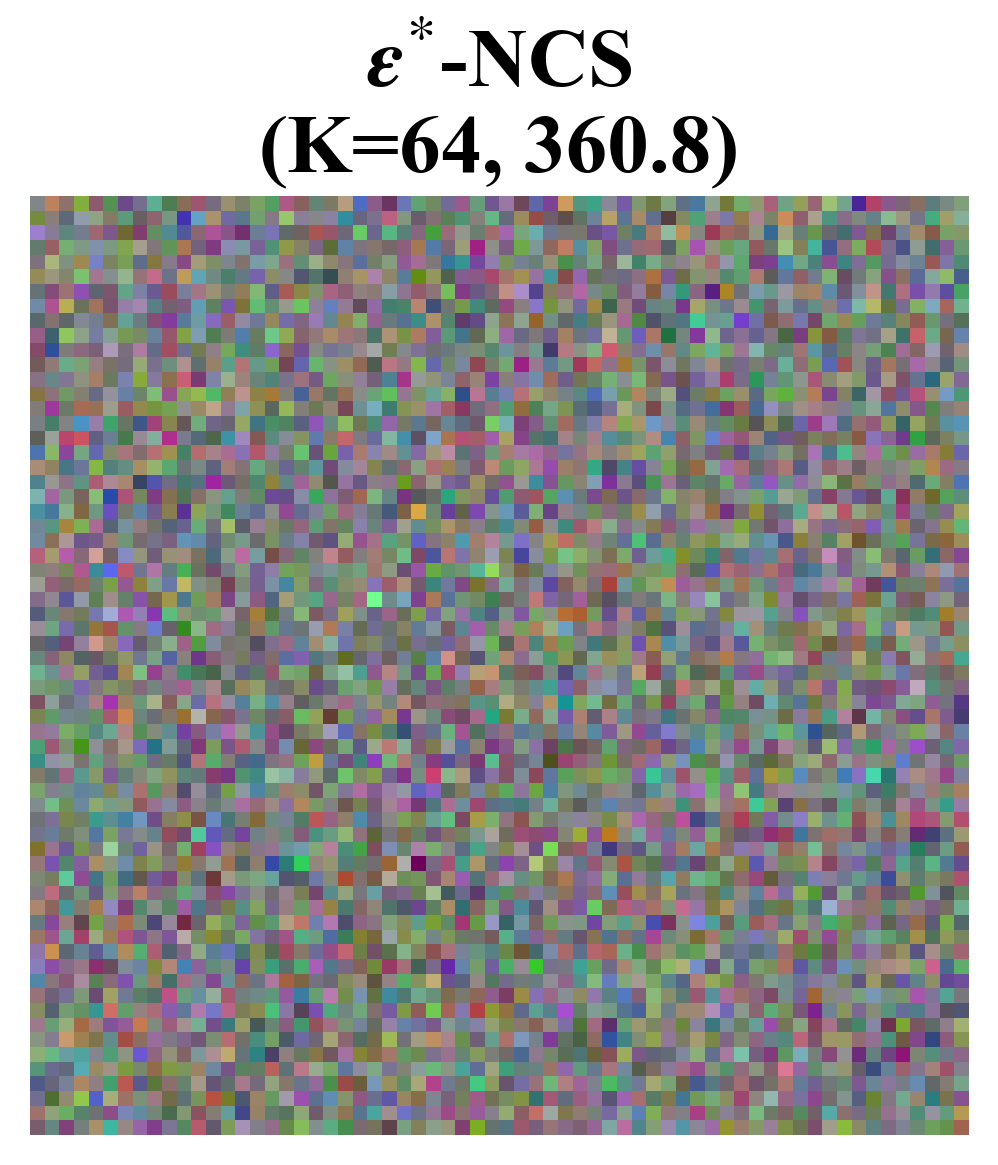}
        \end{minipage}
        \hfill
        \begin{minipage}[t]{0.49\linewidth}
            \centering
            \includegraphics[width=\linewidth]{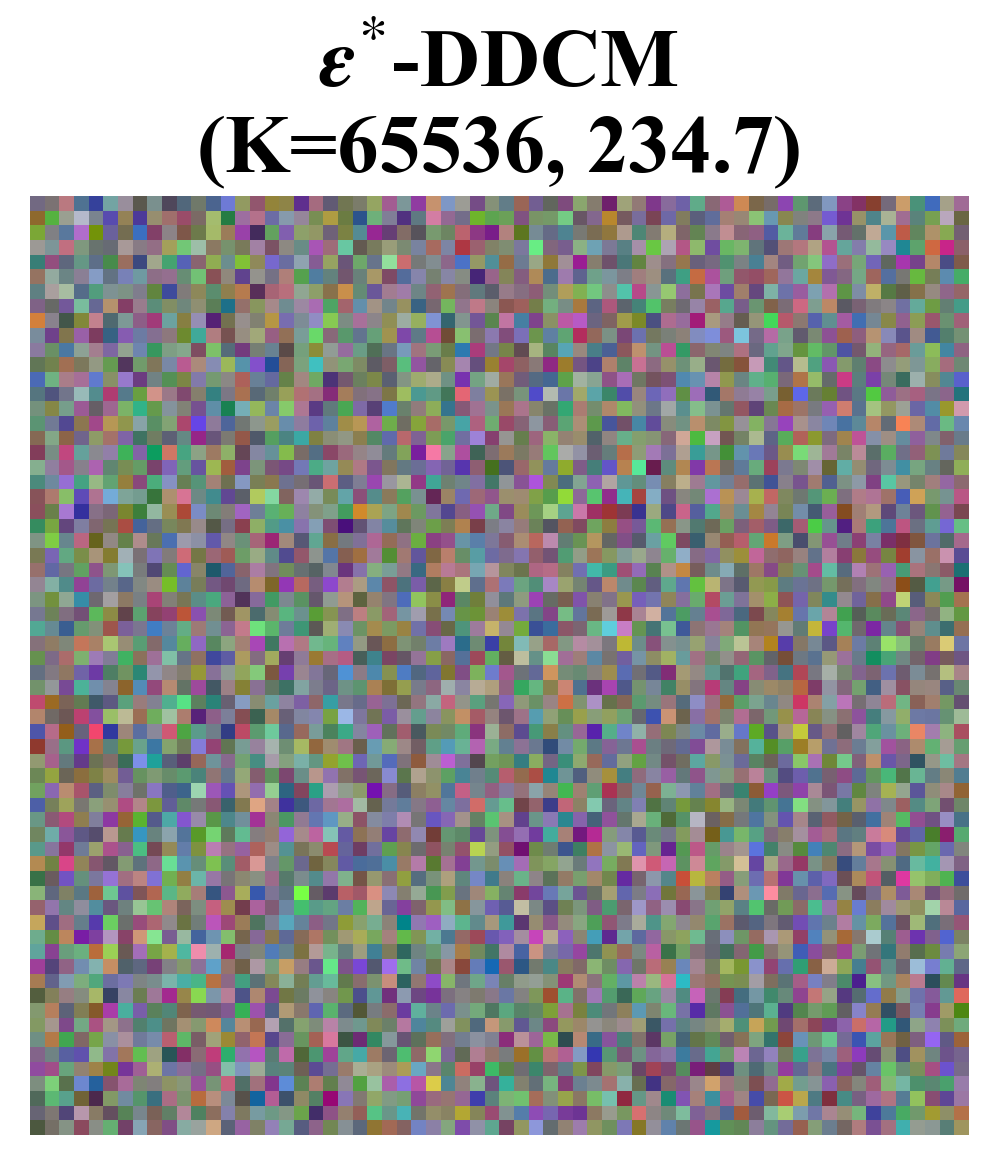}\\
            % \vspace{1mm}
            \includegraphics[width=\linewidth]{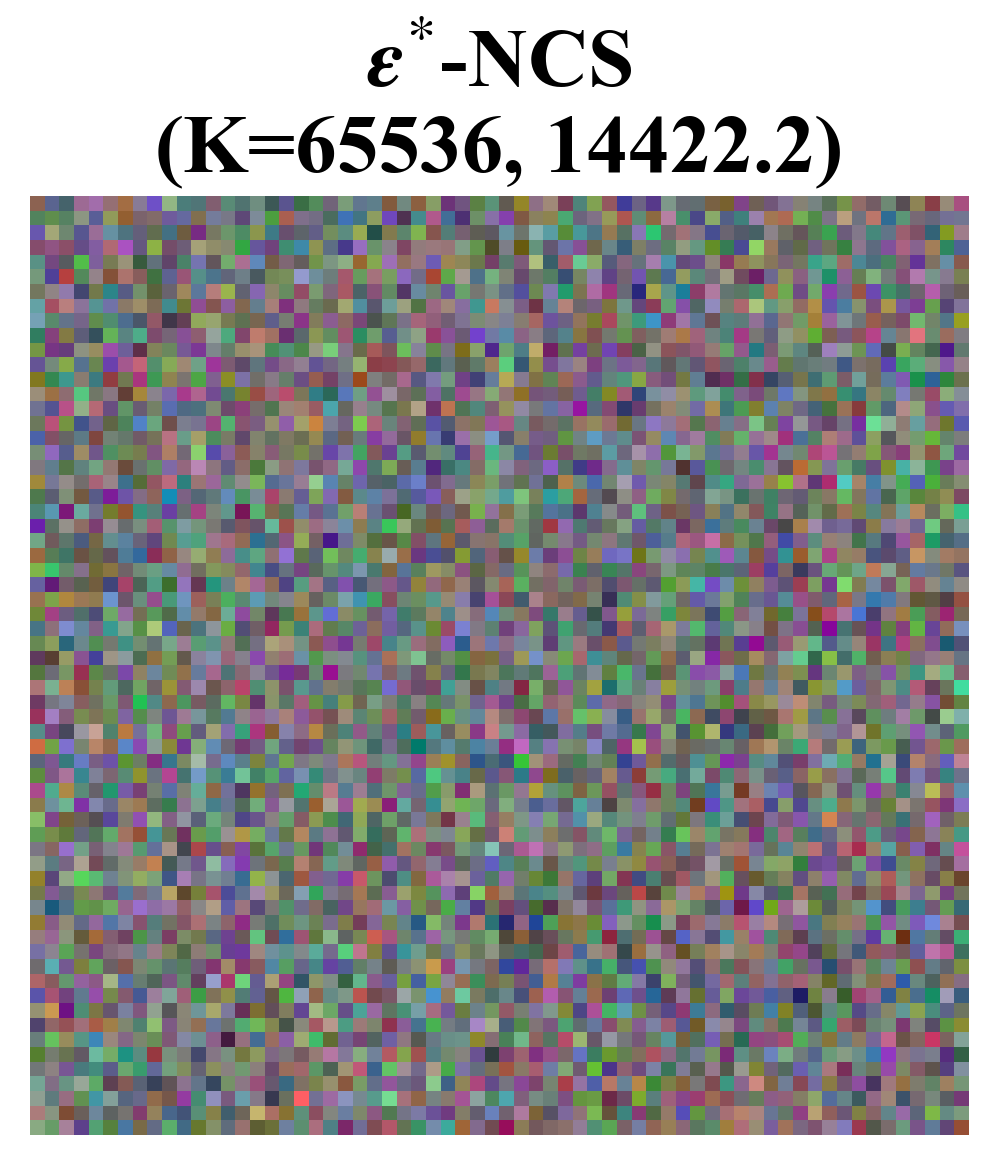}
        \end{minipage}
    \end{minipage}%
    \hfill
    \begin{minipage}[c]{0.49\linewidth}
        \centering
        \includegraphics[width=\linewidth]{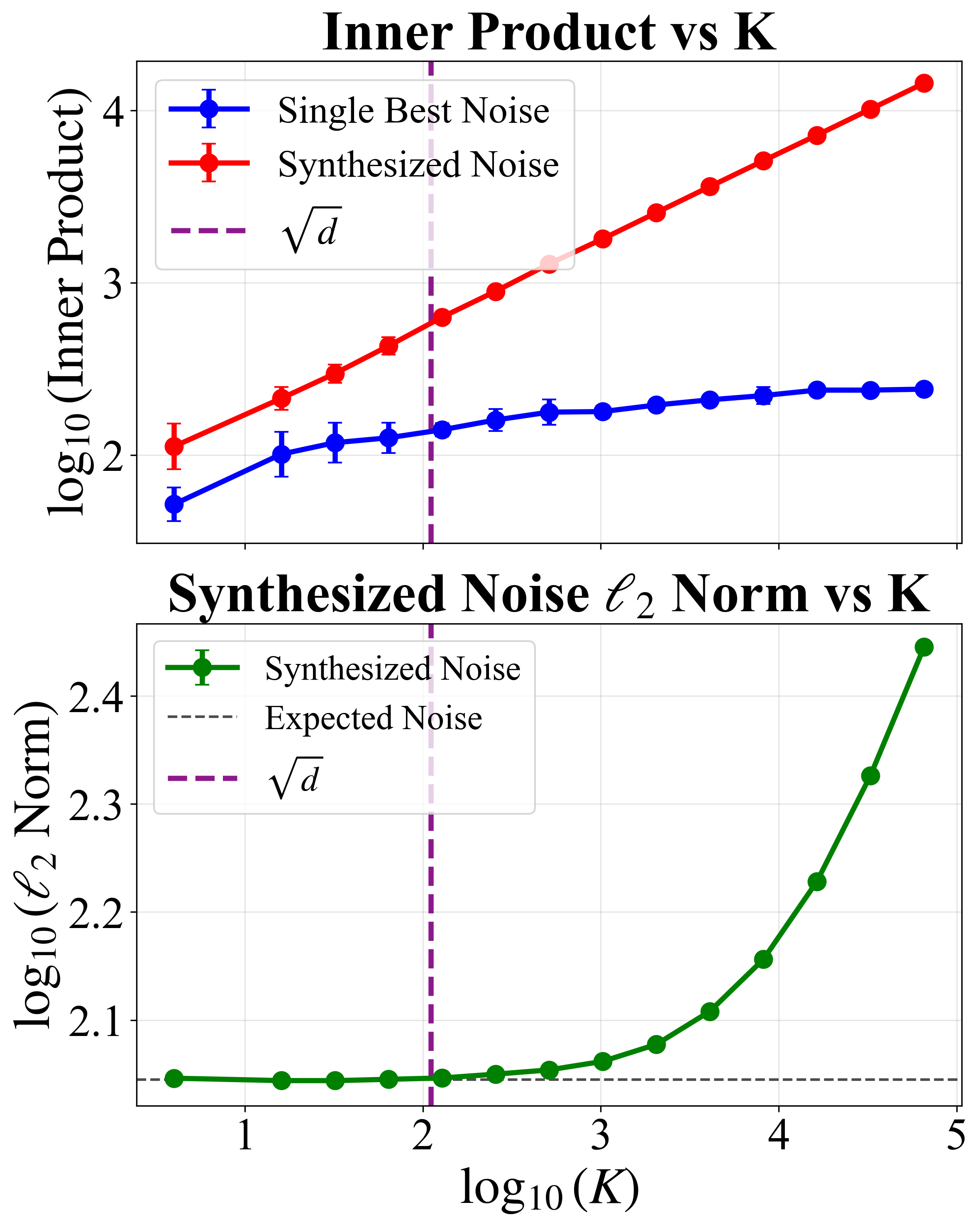}
    \end{minipage}
    \caption{Comparison of the selected noise on a downsampled $64\times64$ image under NCS optimization}
    \label{fig:inner_product_and_norm_vs_k}
\end{figure}
The red term in Eq.~\eqref{eq:conditional_discrete} captures the effect of the observation $\vec{y}$ on the sampling trajectory. Since the exact posterior score $\nabla_{\vec{x}_t} \log p(\vec{y} \mid \vec{x}_t)$ is generally intractable, learning it typically requires additional training~\citep{dhariwal2021diffusion, ho2022classifier}. The central challenge therefore lies in accurately approximating this term. Explicit approximation methods often suffer from sampling instability and long step inference. Variational inference approaches~\citep{mardani2024a, zhang2025improving} aim to approximate a more accurate posterior through more sophisticated optimization procedures, but require complex searches over multiple hyperparameters. Sampling-based methods~\citep{dou2024diffusion, cardosomonte, wu2023practical} avoid explicit posterior modification by selecting candidate samples via importance sampling; however, they typically suffer from weak guidance effects and high computational cost.

\begin{figure*}[t]
    \centering
    \includegraphics[width=0.9\linewidth]{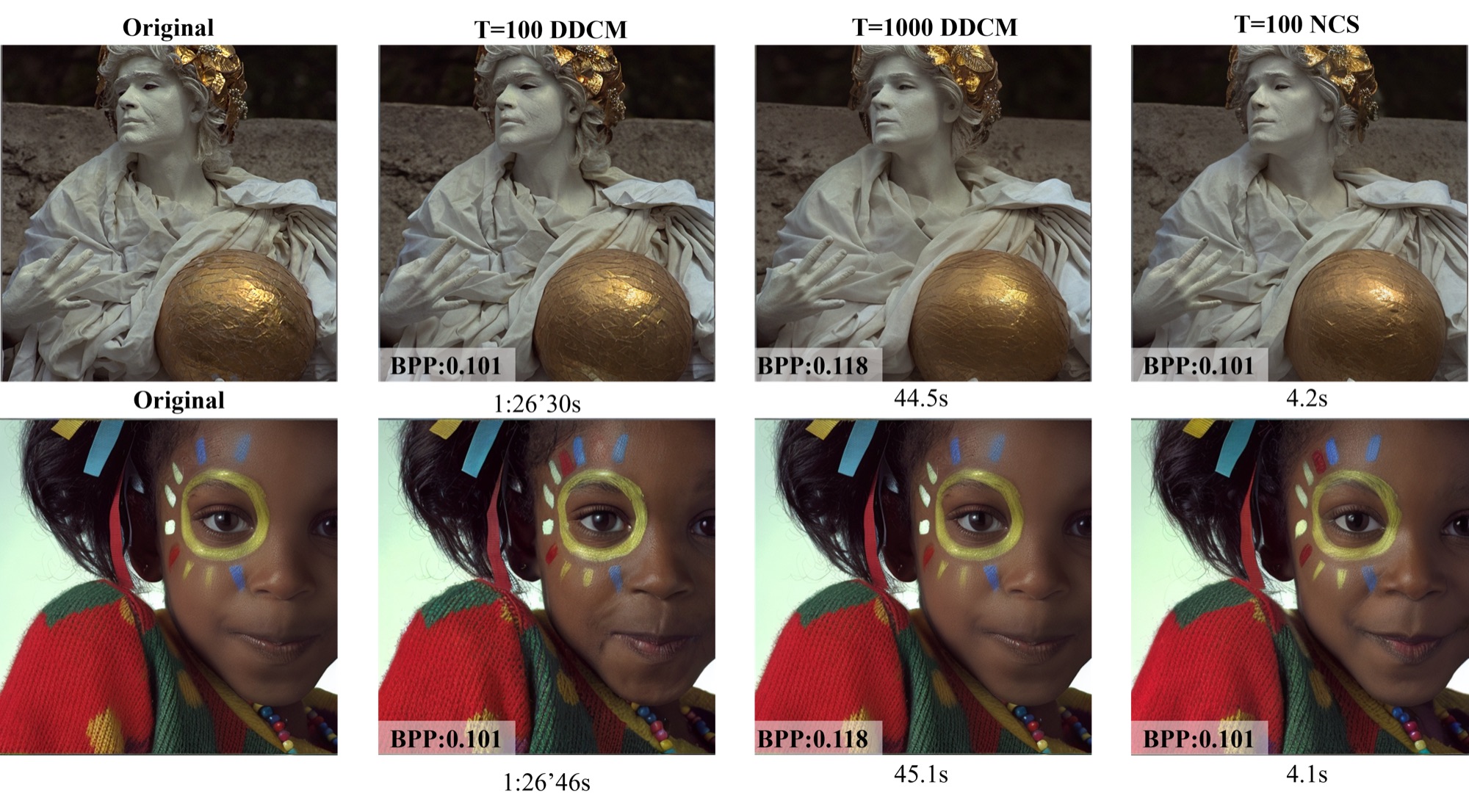}
    \caption{Comparison of the compression efficiency of NCS-MPGD and DDCM. For $T=1000$, we choose $K=32768$, $m=12$, $C=8$. For $T=100$, we choose $K=32768$, $m=2$, $C=0$. Our method achieves equivalent performance in less time.}
    \label{fig:compression_images}
\end{figure*}

\section{Noise Combination Sampling (NCS)}

\label{subsec:ncs}
Under the manifold assumption, $\vec{x}_t$ concentrates around a low-dimensional set induced by $\mathcal{M}_t:=\{\vec x \in \mathcal{N}{(\vec x, \sigma_t^2I)}, \vec x \in \mathcal{M}_0\}$.
Most training-free inverse solvers form an estimate $\tilde{\vec{x}}_{0|t}$ and then enforce measurement consistency by
injecting an external guidance term, e.g., $\nabla_{\vec{x}_t}\log p(\vec{y}\mid \vec{x}_t)$ (or its proxy), into the denoising update.
Implicitly, this strategy relies on the assumption that the guided update still yields an $\tilde{\vec{x}}_{0|t}$ that stays within the true data manifold $\mathcal{M}_0$. In practice, however, the conditional score is only approximate, and directly perturbing the denoising dynamics can easily move the trajectory off-manifold, causing instability and sensitivity to guidance strength and step size (Fig.~\ref{fig:illustration}). 

This preserves the denoising update structure while steering samples toward measurement consistency. NCS takes the opposite view: instead of explicitly modifying $\tilde{\vec{x}}_{0|t}$ or searching for an optimal guidance scale, it implicitly embeds the approximated conditional information into the noise component of the update rule, thereby avoiding direct perturbations of the sampling trajectory off the learned manifold:
\[
    \vec{x}_{t-1} \approx \vec{x}_t - f(\vec{x}_t, t) + g^2(t)\, \nabla_{\vec{x}_t} \log p_t(\vec{x}_t) + \textcolor{orange}{g(t)\, \vec{\epsilon}_t^*}, \notag
\]
where $\vec{\epsilon}_t^*$ is a \textcolor{orange}{synthesized noise} vector that approximates the effect of the conditional term. As a result, the degradation in the measurement score is substantially mitigated, leading to more stable diffusion dynamics:

To achieve this, NCS restricts $\vec{\epsilon}_t^*$ to lie in the span of a finite set of noise vectors, referred to as the \textcolor{blue}{noise codebook}, and synthesizes it as a linear combination of these basis vectors to align with the target conditional direction. The most natural alignment objective is the cosine similarity:
\begin{equation}
    \label{eq:ncs_optimization_angle}
    \max_{\vec{\epsilon}_t^* \in \mathcal{N}(0, \mathbf{I})}
    \cos\left(\nabla_{\vec{x}_t} \log p(\vec{y} \mid \vec{x}_t),\ \vec{\epsilon}_t^*\right).
\end{equation}
Instead of directly optimizing the Eq.~\eqref{eq:ncs_optimization_angle}, NCS adopts an inner-product-based surrogate objective. This surrogate admits a computationally efficient closed-form solution and yields a direction that is nearly identical in high-dimensional settings when the codebook size $K$ is relatively small.

\begin{theorem}[Noise Combination Sampling]
    \label{thm:ncs}
    For linear inverse problems, the optimal noise vector $\vec{\epsilon}_t^*$ that best aligns with the conditional score direction is given by:
    \[
        \vec{\epsilon}_t^* = \sum_{i=1}^K \gamma_i \vec{\epsilon}^i_{t},
    \]
    where $\{\vec{\epsilon}^i_{t}\}_{i=1}^K$ are Gaussian vectors drawn from a fixed noise codebook with size $K$, and $\vec{\gamma} = (\gamma_1, \ldots, \gamma_K)$ denotes the combination weights. NCS seeks the optimal combination weights $\vec{\gamma}^*$ that maximizes the inner product between the conditional score direction and the synthesized noise:
    \begin{equation}
        \label{eq:ncs_optimization}
        \vec{\gamma}^* = \operatorname*{argmax}_{\vec{\gamma} \in \mathbb{R}^K,\ \|\vec{\gamma}\|_2 = 1} \left\langle \nabla_{\vec{x}_t} \log p(\vec{y} \mid \vec{x}_t),\ \sum_{i=1}^K \gamma_i \vec{\epsilon}^i_{t} \right\rangle.
    \end{equation}
    \end{theorem}
    
% The synthesized noise vector $\vec{\epsilon}_t^*$ remains standard normal, i.e., $\vec{\epsilon}_t^* \sim \mathcal{N}(0, \mathbf{I})$, as a consequence of the unit-norm constraint $\|\vec{\gamma}\|_2 = 1$. A complete proof is provided in Appendix~\ref{app:ncs_gaussianity}.

\begin{algorithm}[t]
    \caption{Noise Combination Sampling for Linear Inverse Problems}
    \label{alg:NCS}
    \begin{algorithmic}[1]
    \REQUIRE Select the $K$ for Codebooks $\mathcal{C}_t = \{\vec{\epsilon}_t^1, \dots, \vec{\epsilon}_t^K\}$ for all $t$, represented as matrices $\mat{E}_t$; observation $\vec{y}$; approximate conditional score direction $\vec{c}$
    \ENSURE Reconstructed sample $\vec{x}_0$
    
    \STATE Sample initial latent $\vec{x}_T \sim \mathcal{N}(0, \mathbf{I})$
    \FOR{$t = T$ {\bfseries to} $1$}
        \STATE $\vec{c} \approx \nabla_{\vec{x}_t} \log p(\vec{y} \mid \vec{x}_t)$ \hfill (approximate) 
        \STATE $\vec{\gamma}^* \leftarrow \vec{c}^\top \mat{E}_t / \| \vec{c}^\top \mat{E}_t \|_2$ \hfill (Theorem~\ref{thm:ncs_optimal})
        \STATE $\vec{\epsilon}_t^* \leftarrow \sum_{i=1}^K \gamma_i^* \vec{\epsilon}_t^i$
        \STATE $\vec{x}_{t-1} \leftarrow \mu_\theta(\vec{x}_t, t) + \sigma_t \vec{\epsilon}_t^*$ \hfill (Eq.~(\ref{eq:ddpm_update}))
    \ENDFOR
    \RETURN $\vec{x}_0$
    \end{algorithmic}
    \end{algorithm}

\begin{theorem}[Optimal Noise Combination]
    \label{thm:ncs_optimal}
    Let $\vec{c} = \nabla_{\vec{x}_t} \log p(\vec{y} \mid \vec{x}_t) \in \mathbb{R}^d$ and
    $\mat{E}_t = [\vec{\epsilon}_t^1, \ldots, \vec{\epsilon}_t^K] \in \mathbb{R}^{d \times K}$.
    The optimal weight vector $\vec{\gamma}^* \in \mathbb{R}^K$ that maximizes Eq.~(\ref{eq:ncs_optimization})
    is given by:
    \begin{equation}
        \label{eq:dcc_selection_column}
        \vec{\gamma}^* = \frac{ \mat{E}_t^\top \vec{c} }{ \| \mat{E}_t^\top \vec{c} \|_2 }.
    \end{equation}
    \end{theorem}
The equivalence between cosine-direction fitting and the inner-product-based optimization adopted by NCS, together with the resulting closed-form solution, is formally justified in Appendix~\ref{app:optimal}. Under the NCS framework, existing conditional guidance methods can be unified as optimal noise combinations, including DPS~\citep{chung2023diffusion}, MPGD~\citep{he2024manifold}, and $\Pi$GDM~\citep{song2023pseudoinverse}. Detailed formulations are provided in Appendix~\ref{app:ncs_definitions}. Since matrix multiplications are highly efficient in practice, NCS introduces less than a $4\%$ increase in computational time; detailed measurements are reported in Appendix~\ref{app:computational_cost}.

% The NCS framework extends naturally to the top-$m$ case in DDCM-style problems. It improves the efficiency of noise utilization: As illustrated in Fig.~\ref{fig:inner_product_and_norm_vs_k}, combining a few noise vectors can achieve an inner product magnitude comparable to that obtained by DDCM’s selection from a large codebook noise vectors.

\begin{figure*}[!t]
    \centering
    \includegraphics[width=0.9\linewidth]{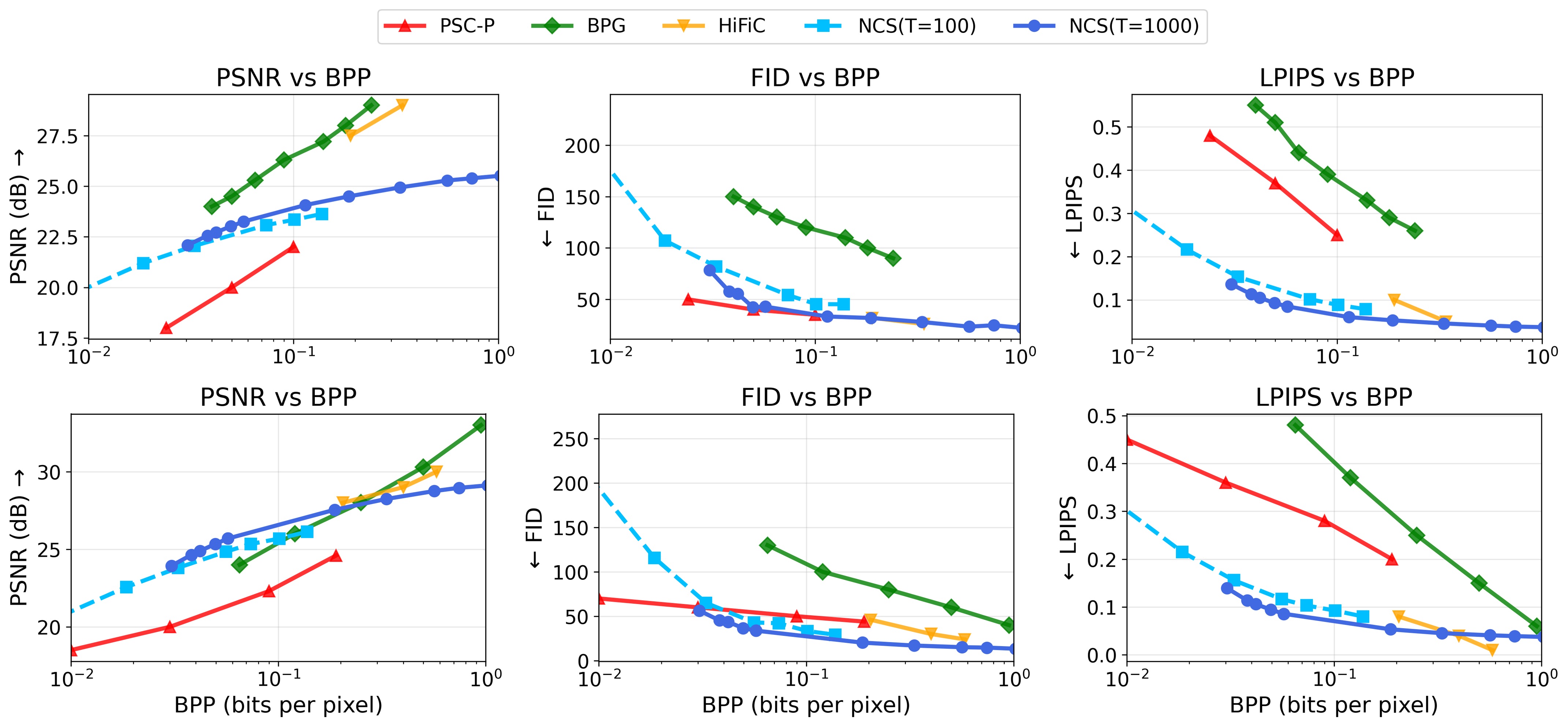}
    \caption{Comparison of compression methods. NCS achieves comparable reconstruction quality while reducing the number of compression steps from 1000 to 100, demonstrating significant efficiency with minimal quality loss.}
    \label{fig:compression_comparison}
\end{figure*} 

\begin{figure}[!t]
    \centering
    \includegraphics[width=0.95\linewidth]{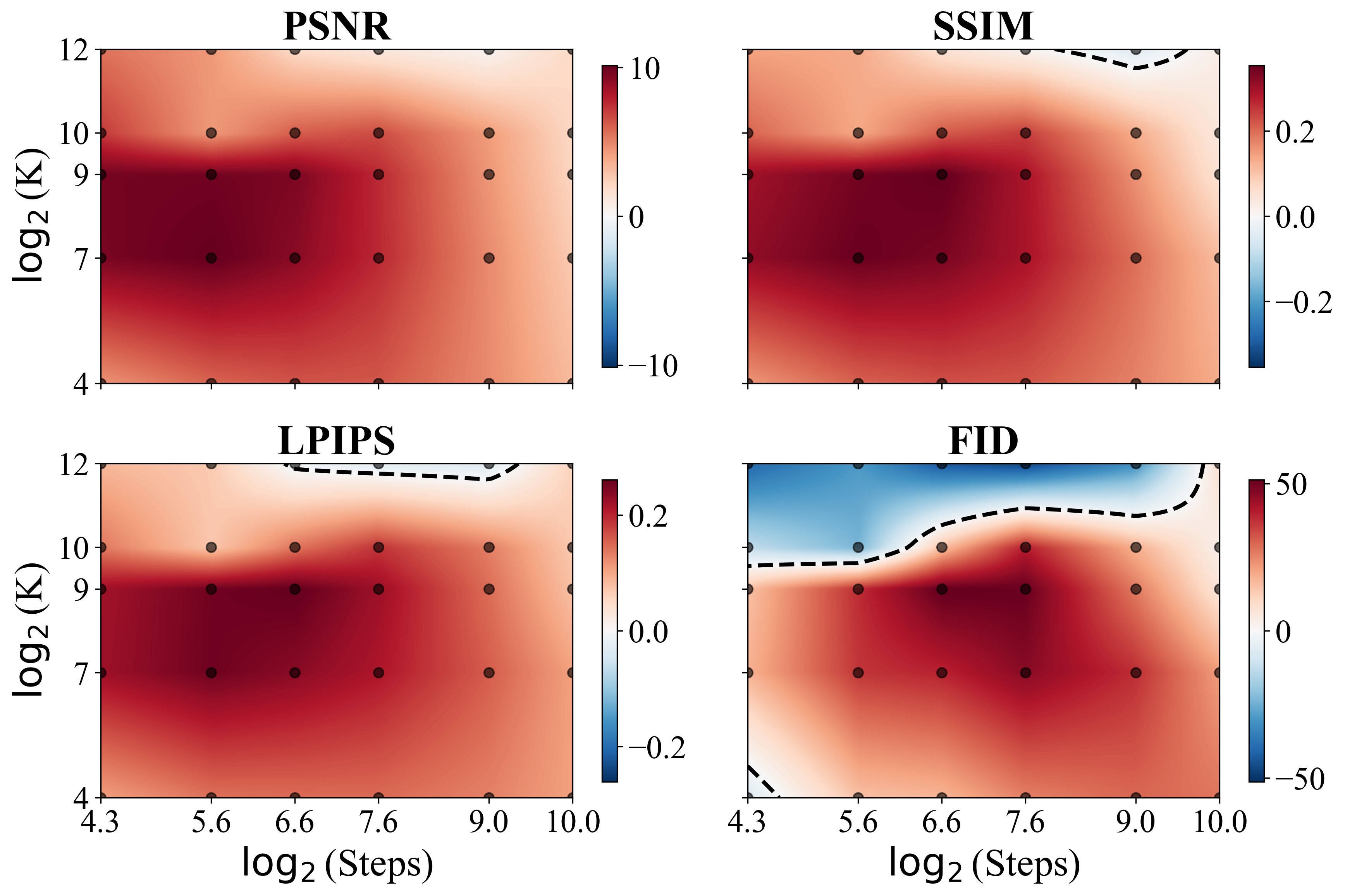}
    \caption{Influence of iteration and codebook size on inpainting task (box-region, $\sigma^2 = 0.05$). The \textcolor{red}{red} regime in the heatmap shows the improvement of NCS-DPS compared to DPS methods according to the metric.}
    \label{fig:heatmap}
\end{figure}

We analyze the sensitivity of NCS to the codebook size $K$ using an inpainting experiment in Figure~\ref{fig:heatmap}. It is evident that NCS yields consistent improvements for the DPS algorithm across almost all values of $K$ and step numbers $T$. As $K$ increases, the synthesized noise $\vec{\epsilon}^*$ becomes more aligned with the measurement score. Nevertheless, due to near-orthogonality in high-dimensional spaces, randomly sampled noise vectors remain almost orthogonal. As illustrated in Fig.~\ref{fig:inner_product_and_norm_vs_k}, the inner product between the synthesized noise and the measurement score grows linearly with $K$, while the $\ell_2$ norm remains consistent with Gaussian noise at the $\sqrt{d}$ scale.

\begin{proposition}[SNR Scaling of NCS]
    \label{prop:n_approx_sqrt_d}
    Let $\vec{\epsilon}^{*} \in \mathbb{R}^d$ denote the synthesized noise produced by NCS from $K$ i.i.d.\ Gaussian noise vectors. Define $\mathrm{SNR}(K,d)$ as the ratio between the expected energy of $\vec{\epsilon}^{*}$ along the measurement score direction and the expected energy in its orthogonal complement. Then,
    \[
        \mathrm{SNR}(K,d) \;=\; \frac{K}{d-1}.
    \]
    \end{proposition}
Choosing $K = \sqrt{d}$ yields $\mathrm{SNR}(K,d) = \Theta(1/\sqrt{d})$, which decays slowly enough to provide effective guidance while preserving the Gaussian-like structure of the noise. However, for ill-posed inverse problems where the forward operator $\mat{A}: \mathbb{R}^d \to \mathbb{R}^n$ satisfies $n \ll d$, the measurement score is effectively constrained to a lower-dimensional subspace of dimension $n$. Thus, we substitute the effective dimension $n$ for $d$ and choose $K \approx \sqrt{n}$ in practice. Task-specific values of $K$ are reported in Appendix~\ref{app:inverse_problems}. A detailed proof of Proposition~\ref{prop:n_approx_sqrt_d} is provided in Appendix~\ref{app:pseudo-noise-analysis}.

\subsection{Generative Compression Tasks}

Compression can be viewed as a special case of inverse problems, where $\mat{A} = \mat{I}$, $\vec{n} = \vec{0}$, and $\vec{y} = \vec{x}_0$. Since the stochasticity of DMs arises from the noise $\vec{\epsilon}_t$, DDCM fixes the noise by selecting, from a noise codebook, the noise vector that best represents the influence of the observation $\vec{y}$. Reconstruction is then achieved by choosing the corresponding noise vector. From this perspective, DDCM can be interpreted as a special case of NCS:

\begin{definition}[DDCM as a Special Case of NCS]
DDCM selects a single noise vector from a predefined codebook via one-hot maximization:
    \begin{equation}
        \vec{\epsilon}_t^* = \operatorname*{argmax}_{i \in \{1, \dots, K\}} 
        \left\langle \vec{x}_0 - \vec{\tilde{x}}_{0|t},\ \vec{\epsilon}_t^i \right\rangle .
    \end{equation}
    This procedure is equivalent to solving the NCS optimization problem in Eq.~\eqref{eq:ncs_optimization} under the additional constraint that $\vec{\gamma} \in \{0,1\}^K$ and $\|\vec{\gamma}\|_0 = 1$, i.e., exactly one component of $\vec{\gamma}$ is nonzero. A detailed proof of this equivalence is provided in Appendix~\ref{app:ncs_optimization}.
    \end{definition}

DDCM observes that employing multiple noise vectors can improve compression quality, and quantizes the combination weights $\gamma_i$ of each selected noise vector into $2^{C}$ bins to reduce storage overhead, where $C$ is the number of bits per coefficient. The resulting bits-per-pixel (BPP) is defined as ${(T-1)\left(m \log_2 K + C (m-1)\right)}/{(H \times W)}$, where $m$ denotes the number of selected noise vectors and $H \times W$ is the image resolution. To determine the optimal combination of noise vectors and their quantized weights, DDCM employs a greedy search strategy that iteratively explores all quantization bins (see~\citet{ohayon2025compressed}, Appendix~B.5). However, this procedure incurs a computational complexity that grows \emph{exponentially} with both $C$ and $m$, rendering it computationally prohibitive for large-scale or low-latency settings.

Based on Theorem~\ref{thm:ncs_optimal}, we propose a closed-form approximation for quantizing $\gamma_i$, which substantially enables increasing $m$ and $C$ while simultaneously reducing $T$. As shown in Figures~\ref{fig:compression_images} and~\ref{fig:compression_comparison}, reducing $T$ by a factor of $10$ while increasing both $m$ and the number of quantization bins preserves the compression ratio, incurs minimal reconstruction degradation, and achieves a $10\times$ speedup in compression time. The approximation time is negligible; additional implementation details are provided in Appendix~\ref{app:ncs_quantization}.

% ---------------- FFHQ RESULTS (Inpainting + SR4x) ----------------
\begin{table*}[!t]
    \centering
\begin{footnotesize}
    \begin{sc}
    \caption{Quantitative comparison of baseline solvers and their NCS variants on FFHQ dataset. Each cell shows PSNR \,/\,FID \,/\,LPIPS.}
    \label{tab:inverse_results_ffhq_main}
    \begin{tabular}{llccc}
    \toprule
    & & \multicolumn{3}{c}{PSNR(↑)\,/\,FID(↓)\,/\,LPIPS(↓)}\\
    \cmidrule(lr){3-5}
    Task & Method & 20 & 100 & 1000 \\
    \midrule
    \multirow{8}{*}{\shortstack{Inpainting\\(Box)}}    
        & DPS              & 12.52\,/\,133.9 /\,0.497 & 18.67\,/\,95.65 /\,0.286 & 22.71\,/\,57.11 /\,0.139 \\
        & \cellcolor{gray!20}\textbf{NCS-DPS} & \cellcolor{gray!20}\textbf{19.16\,/\,116.0 /\,0.323} & \cellcolor{gray!20}\textbf{22.31\,/\,71.38 /\,0.170} & \cellcolor{gray!20}\textbf{23.41\,/\,39.82 /\,0.088} \\
        & MPGD             & 16.84\,/\,107.3 /\,\textbf{0.220} & 17.26\,/\,107.6 /\,0.164 & 13.51\,/\,241.7 /\,0.387 \\
        & \cellcolor{gray!20}\textbf{NCS-MPGD}& \cellcolor{gray!20}\textbf{19.00}\,/\,\textbf{99.31} /\,0.277 & \cellcolor{gray!20}\textbf{20.53\,/\,62.50 /\,{0.153}} & \cellcolor{gray!20}\textbf{20.96\,/\,44.42 /\,{0.101}} \\
        & DAPS             & 22.01\,/\,53.52 /\,0.209            &  {\textbf{22.56}}\,/\,46.49 /\, 0.197 & 24.20\,/\,{\textbf{39.79}} /\,\textbf{0.168}\\
        & \cellcolor{gray!20}\textbf{NCS-DAPS}& \cellcolor{gray!20}\textcolor{blue}{\textbf{22.33}}\,/\,\textcolor{blue}{\textbf{49.71}} /\,{\textbf{0.205}}& \cellcolor{gray!20}22.47\,/\,{\textbf{45.01}} /\,\textbf{0.195}   & \cellcolor{gray!20}{\textbf{24.24}}\,/\,39.98 /\,0.168\\
        & DDNM             & 22.05\,/\,69.50\,/\,\textcolor{blue}{\textbf{0.136}} & 23.87\,/\,62.90\,/\,0.108 & 23.99\,/\,112.9\,/\,0.162 \\
        & RED-Diff         & 15.92\,/\,211.3\,/\,0.809 & \textcolor{blue}{\textbf{24.62}}\,/\,\textcolor{blue}{\textbf{35.60}}\,/\,\textcolor{blue}{\textbf{0.091}} & \textcolor{blue}{\textbf{24.45}}\,/\,\textcolor{blue}{\textbf{25.00}}\,/\,\textcolor{blue}{\textbf{0.075}} \\
    \midrule
    \multirow{8}{*}{\shortstack{Inpainting\\(Random)}}    
        & DPS              & 13.13\,/\,132.8 /\,0.472 & 19.26\,/\,96.09 /\,0.278 & 27.35\,/\,58.85 /\,0.126 \\
        & \cellcolor{gray!20}\textbf{NCS-DPS} & \cellcolor{gray!20}\textbf{21.20\,/\,87.55 /\,0.297} & \cellcolor{gray!20}\textbf{\textcolor{blue}{27.31\,/\,42.34} /\,0.137} & \cellcolor{gray!20}\textbf{\textcolor{blue}{31.57}\,/\,\textcolor{blue}{14.28} /\,\textcolor{blue}{0.042}} \\
        & MPGD             & \textbf{\textcolor{blue}{21.80\,/\,76.86}} /\,\textcolor{blue}{\textbf{0.172}} & \textbf{25.45\,/\,45.73 /\,{0.100}} & 25.05\,/\,69.68 /\,0.214 \\
        & \cellcolor{gray!20}\textbf{NCS-MPGD}& \cellcolor{gray!20}20.25\,/\,102.7 /\,0.290 & \cellcolor{gray!20}25.21\,/\,61.82 /\,0.126 & \cellcolor{gray!20}\textbf{28.71\,/\,31.91 /\,0.049} \\
        & DAPS             & 14.08\,/\,259.2 /\,0.625            &  16.31\,/\,226.4 /\,0.543               &  25.33\,/\,67.69 /\,0.238\\
        & \cellcolor{gray!20}\textbf{NCS-DAPS}& \cellcolor{gray!20}\textbf{16.83}\,/\,\textbf{202.8} /\,\textbf{0.557}& \cellcolor{gray!20}\textbf{19.50}\,/\,\textbf{141.9} /\,\textbf{0.456}&\cellcolor{gray!20}\textbf{25.73\,/\,63.56 /\,0.230}\\
        & DDNM             & 15.48\,/\,229.4\,/\,0.612 & 22.64\,/\,130.4\,/\,0.259 & 26.87\,/\,91.20\,/\,0.157 \\
        & RED-Diff         & 11.70\,/\,416.4\,/\,1.285 & 21.50\,/\,204.8\,/\,{{0.409}} & 23.62\,/\,110.6\,/\,0.198 \\
    \midrule
    \multirow{8}{*}{SR 4×}
        & DPS              & 12.87\,/\,121.9\,/\,0.480 & 16.70\,/\,98.36\,/\,0.338 & 23.53\,/\,69.54\,/\,0.171 \\
        & \cellcolor{gray!20}\textbf{NCS-DPS} & \cellcolor{gray!20}\textbf{21.07\,/\,111.5\,/\,0.290} & \cellcolor{gray!20}\textbf{{26.33}\,/\,62.51\,/\,0.133} & \cellcolor{gray!20}\textbf{26.59\,/\,\textcolor{blue}{31.57}\,/\,\textcolor{blue}{0.084}} \\
        & MPGD             & 19.35\,/\,{\textbf{90.92}}\,/\,0.246 & 22.59\,/\,63.68\,/\,0.148 & 20.46\,/\,84.88\,/\,0.490 \\
        & \cellcolor{gray!20}\textbf{NCS-MPGD}& \cellcolor{gray!20}\textbf{22.83}\,/\,92.41\,/\,\textbf{0.231} & \cellcolor{gray!20}\textbf{25.82\,/\,\textcolor{blue}{49.53\,/\,0.115}} & \cellcolor{gray!20}\textbf{25.85\,/\,35.83\,/\,{0.161}} \\
        & DAPS             & 25.52\,/\,\textcolor{blue}{\textbf{89.95}}\,/\,0.329 &  {26.31\,/\,79.90\,/\,0.329}  & {{28.22}\,/\,52.80\,/\,0.200}\\
        & DDNM             & \textcolor{blue}{\textbf{27.88}}\,/\,99.80\,/\,\textcolor{blue}{\textbf{0.187}} & 28.49\,/\,98.60\,/\,0.169 & 29.03\,/\,101.7\,/\,0.176 \\
        & RED-Diff         & 15.95\,/\,173.0\,/\,0.911 & \textcolor{blue}{\textbf{29.91}}\,/\,72.20\,/\,0.149 & \textcolor{blue}{\textbf{30.27}}\,/\,64.30\,/\,0.136 \\
        \midrule
        \multirow{8}{*}{SR 8×}
            & DPS              & 11.84\,/\,118.3\,/\,0.506 & 15.05\,/\,102.1\,/\,0.387 & 20.87\,/\,76.90\,/\,\textbf{\textcolor{blue}{0.220}} \\
            & \cellcolor{gray!20}\textbf{NCS-DPS} & \cellcolor{gray!20}\textbf{20.83\,/\,109.4\,/\,0.290} & \cellcolor{gray!20}\textbf{24.09\,/\,67.34\,/\,\textcolor{blue}{0.154}} & \cellcolor{gray!20}\textbf{21.97\,/\,\textcolor{blue}{49.35}\,/\,{0.266}} \\
            & MPGD             & 17.60\,/\,93.26\,/\,0.303 & 20.08\,/\,77.28\,/\,0.213 & 18.08\,/\,87.93\,/\,0.560 \\
            & \cellcolor{gray!20}\textbf{NCS-MPGD}& \cellcolor{gray!20}\textbf{22.10}\,/\,\textcolor{blue}{\textbf{92.13}}\,/\,\textcolor{blue}{\textbf{0.246}} & \cellcolor{gray!20}\textbf{23.29}\,/\,\textcolor{blue}{\textbf{65.00}}\,/\,\textbf{0.170} & \cellcolor{gray!20}\textbf{21.67\,/\,60.28\,/\,0.352} \\
            & DAPS             & 23.66\,/\,202.8\,/\,0.387 & 24.02\,/\,270.5\,/\,0.337 & 25.20\,/\,188.4\,/\,\textbf{0.278} \\
            & \cellcolor{gray!20}\textbf{NCS-DAPS}& \cellcolor{gray!20}\textbf{{23.74}\,/\,199.9\,/\,0.382} & \cellcolor{gray!20}\textbf{{24.14}\,/\,273.6\,/\,0.334} & \cellcolor{gray!20}\textbf{{25.23}\,/\,188.6}\,/\,0.279 \\
            & DDNM             & \textcolor{blue}{\textbf{24.52}}\,/\,112.2\,/\,0.252 & 25.09\,/\,112.6\,/\,0.239 & 25.51\,/\,121.2\,/\,0.254 \\
            & RED-Diff         & 15.83\,/\,217.0\,/\,0.974 & \textcolor{blue}{\textbf{26.38}}\,/\,104.4\,/\,0.271 & \textcolor{blue}{\textbf{26.38}}\,/\,94.00\,/\,0.258 \\
        \midrule
        \multirow{8}{*}{Gaussian Deblur}
            & DPS              & 12.12\,/\,139.4\,/\,0.495 & 20.02\,/\,88.92\,/\,0.253 & 24.81\,/\,60.57\,/\,0.130 \\
            & \cellcolor{gray!20}\textbf{NCS-DPS} & \cellcolor{gray!20}\textbf{21.59\,/\,115.1\,/\,0.311} & \cellcolor{gray!20}\textbf{26.68\,/\,52.14\,/\,0.115} & \cellcolor{gray!20}\textbf{27.02\,/\,\textcolor{blue}{49.90}\,/\,\textcolor{blue}{0.083}} \\
            & MPGD             & 21.36\,/\,78.46\,/\,0.206 & 23.85\,/\,64.22\,/\,0.133 & 24.80\,/\,74.67\,/\,0.130 \\
            & \cellcolor{gray!20}\textbf{NCS-MPGD}& \cellcolor{gray!20}\textbf{25.42\,/\,\textcolor{blue}{\textbf{74.20}}\,/\,\textcolor{blue}{0.184}} & \cellcolor{gray!20}\textbf{{26.58}}\,/\,\textcolor{blue}{\textbf{48.04}}\,/\,\textcolor{blue}{\textbf{0.110}} & \cellcolor{gray!20}\textbf{{26.36}\,/\,86.73\,/\,0.166} \\
            & DAPS             & {25.80}\,/\,85.06\,/\,0.319 & {26.63\,/\,71.26\,/\,0.267} & \textcolor{blue}{\textbf{28.38}}\,/\,{{50.10}}\,/\,{0.187} \\
            & DDNM             & \textcolor{blue}{\textbf{27.86}}\,/\,92.50\,/\,0.241 & \textcolor{blue}{\textbf{28.25}}\,/\,105.5\,/\,0.249 & 28.21\,/\,105.5\,/\,0.250 \\
            & RED-Diff         & 15.87\,/\,251.2\,/\,1.025 & 25.71\,/\,116.3\,/\,0.305 & 26.90\,/\,91.30\,/\,0.232 \\
        \midrule
        \multirow{8}{*}{Motion Deblur}
            & DPS              & 12.21\,/\,140.3\,/\,0.493 & 20.19\,/\,88.61\,/\,0.251 & 25.87\,/\,56.65\,/\,0.118 \\
            & \cellcolor{gray!20}\textbf{NCS-DPS} & \cellcolor{gray!20}\textbf{22.05\,/\,114.9\,/\,0.303} & \cellcolor{gray!20}\textbf{\textcolor{blue}{28.28}\,/\,47.25\,/\,{0.097}} & \cellcolor{gray!20}\textbf{29.50\,/\,\textcolor{blue}{28.41}\,/\,0.044} \\
            & MPGD             & 21.86\,/\,76.79\,/\,0.191 & 24.88\,/\,53.05\,/\,0.107 & 25.54\,/\,54.78\,/\,0.142 \\
            & \cellcolor{gray!20}\textbf{NCS-MPGD}& \cellcolor{gray!20}\textbf{{26.22}}\,/\,\textcolor{blue}{\textbf{64.51}}\,/\,\textcolor{blue}{\textbf{0.153}} & \cellcolor{gray!20}\textbf{{27.76}}\,/\,\textcolor{blue}{\textbf{39.31}}\,/\,\textcolor{blue}{\textbf{0.078}} & \cellcolor{gray!20}\textbf{28.17}\,/\,\textbf{41.91}\,/\,\textbf{\textcolor{blue}{0.083}} \\
            & DAPS             & \textbf{\textcolor{blue}{26.25}}\,/\,66.91\,/\,{0.289} & {27.77}\,/\,57.02\,/\,{0.232} & {{30.52}}\,/\,{34.34}\,/\,0.139 \\
            & RED-Diff         & 16.01\,/\,200.0\,/\,0.946 & 28.17\,/\,88.30\,/\,0.205 & \textcolor{blue}{\textbf{32.06}}\,/\,50.10\,/\,0.092 \\
        \bottomrule
    \end{tabular}
\end{sc}
\end{footnotesize}
\end{table*}

\section{Experiments}

For experiments on FFHQ~\citep{karras2019style} and ImageNet~\citep{deng2009imagenet}, reported in Tables~\ref{tab:inverse_results_ffhq_main} and~\ref{tab:inverse_results_imagenet_main}, we adopt the DMs from~\citet{dhariwal2021diffusion}. For compression experiments, we use Stable Diffusion~2.0 models as introduced in~\citet{rombach2022highresolutionimagesynthesislatent}. These experiments were conducted on an NVIDIA RTX~4090 GPU. Experiments on scientific inverse problems were performed using the models provided by~\citet{zheng2025inversebench} and were conducted on an NVIDIA RTX~5090 GPU.

\subsection{Inverse Problem Solving}

We evaluate our methods on image inpainting (box/random region), super-resolution ($\times4$ /$\times8$), Gaussian deblurring, and motion deblurring. We select DPS, MPGD, DDNM~\citep{wang2022zero}, Red-Diff~\citep{mardani2024a}, and DAPS~\citep{zhang2025improving} as baseline solvers. All problem settings follow~\citet{chung2023diffusion}, and all tasks use $\sigma = 0.05$. Detailed configurations are provided in Appendix~\ref{app:inverse_problems}. We apply the NCS framework to DPS, MPGD, and DAPS to compare with the original methods. We use \textbf{black bold font} to indicate the better result between each baseline method and its NCS counterpart, and \textcolor{blue}{\textbf{blue bold font}} to highlight the best-performing method among all approaches under the same setting. For tasks where NCS-DAPS yields negligible improvement over DAPS (e.g., SR~$\times$4), we omit NCS-DAPS results for brevity. Our NCS framework improves existing DPS and MPGD methods and achieves the best performance on several tasks.

% ---------------- IMAGENET RESULTS ----------------
\begin{table*}[!t]
    \centering
    \caption{Quantitative comparison of baseline solvers and their NCS variants on ImageNet. Each cell shows PSNR / FID / LPIPS.}
    \label{tab:inverse_results_imagenet_main}
    \begin{footnotesize}
        \begin{sc}
    \begin{tabular}{llccc}
    \toprule
    & & \multicolumn{3}{c}{PSNR(↑) / FID(↓) / LPIPS(↓)}\\
    \cmidrule(lr){3-5}
    Task & Method & 20 & 100 & 1000 \\
    \midrule
    
    \multirow{8}{*}{\shortstack{Inpainting\\(Random)}}    
        & DPS              & 12.88\,/\,271.9\,/\,0.671 & 16.86\,/\,231.9\,/\,0.545 & 23.96\,/\,101.2\,/\,0.278 \\
        & \cellcolor{gray!20}\textbf{NCS-DPS} & \cellcolor{gray!20}\textbf{17.66\,/\,262.1\,/\,0.596} & \cellcolor{gray!20}\textbf{23.70\,/\,110.1\,/\,{0.300}} & \cellcolor{gray!20}\textbf{28.69\,/\,37.77\,/\,{0.097}} \\
        & MPGD             & 17.15\,/\,\textbf{121.8}\,/\,\textbf{0.285} & 19.32\,/\,170.9\,/\,0.435 & 16.38\,/\,263.9\,/\,0.770 \\
        & \cellcolor{gray!20}\textbf{NCS-MPGD}& \cellcolor{gray!20}\textbf{17.47}\,/\,238.2\,/\,0.495 & \cellcolor{gray!20}\textbf{22.05\,/\,102.0\,/\,0.222} & \cellcolor{gray!20}\textbf{24.02\,/\,58.80\,/\,0.178} \\
        & DAPS             & 25.59\,/\,60.41\,/\,0.252 &  {\textbf{26.31}}\,/\,44.64\,/\,{\textbf{0.329}}  & {\textbf{28.79}}\,/\,\textcolor{blue}{\textbf{23.83}}\,/\,\textbf{0.146}\\
        & \cellcolor{gray!20}\textbf{NCS-DAPS}& \cellcolor{gray!20}{\textbf{25.97}}\,/\,{\textbf{55.28}}\,/\,{\textbf{0.240}}&    \cellcolor{gray!20}25.75\,/\,{\textbf{42.57}}\,/\,0.318    & \cellcolor{gray!20}28.64\,/\,25.44\,/\,0.149\\
        & DDNM             & \textcolor{blue}{\textbf{29.25}}\,/\,\textcolor{blue}{\textbf{46.20}}\,/\,\textcolor{blue}{\textbf{0.103}} & \textcolor{blue}{\textbf{31.41}}\,/\,\textcolor{blue}{\textbf{42.40}}\,/\,\textcolor{blue}{\textbf{0.080}} & \textcolor{blue}{\textbf{32.51}}\,/\,{40.80}\,/\,\textcolor{blue}{\textbf{0.071}} \\
        & RED-Diff         & 15.17\,/\,266.7\,/\,0.769 & 20.84\,/\,103.2\,/\,0.212 & 20.63\,/\,86.90\,/\,0.177 \\
    \midrule
    \multirow{7}{*}{SR 4×}
        & DPS              & 12.12\,/\,\textbf{255.6}\,/\,0.688 & 14.92\,/\,248.3\,/\,0.587 & 20.40\,/\,144.6\,/\,0.365 \\
        & \cellcolor{gray!20}\textbf{NCS-DPS} & \cellcolor{gray!20}\textbf{17.60}\,/\,260.9\,/\,\textbf{0.595} & \cellcolor{gray!20}\textbf{22.62\,/\,112.4\,/\,0.328} & \cellcolor{gray!20}\textbf{23.78\,/\,\textcolor{blue}{44.12}\,/\,\textcolor{blue}{0.195}} \\
        & MPGD             & 16.41\,/\,191.0\,/\,0.420 & 18.60\,/\,96.65\,/\,0.327 & 10.43\,/\,292.1\,/\,1.109 \\
        & \cellcolor{gray!20}\textbf{NCS-MPGD}& \cellcolor{gray!20}\textbf{19.83\,/\,165.5\,/\,0.417} & \cellcolor{gray!20}\textbf{22.35\,/\,{76.57}\,/\,\textcolor{blue}{0.223}} & \cellcolor{gray!20}\textbf{16.94\,/\,92.43\,/\,0.704} \\
        & DAPS             & {23.25}\,/\,135.9\,/\,{0.368} &  {{23.02}}\,/\,{257.4}\,/\,{{0.358}}  & {{25.21}}\,/\,{105.0}\,/\,{0.301}\\
        & DDNM             & \textcolor{blue}{\textbf{24.87}}\,/\,\textcolor{blue}{\textbf{112.5}}\,/\,\textcolor{blue}{\textbf{0.364}} & 25.41\,/\,118.4\,/\,0.340 & {25.71}\,/\,123.6\,/\,0.322 \\
        & RED-Diff         & 15.46\,/\,232.4\,/\,0.900 & \textcolor{blue}{\textbf{26.01}}\,/\,\textcolor{blue}{\textbf{75.60}}\,/\,0.306 & \textcolor{blue}{\textbf{26.11}}\,/\,{79.10}\,/\,0.282 \\
        \midrule
        \multirow{7}{*}{Gaussian Deblur}
            & DPS              & 12.02\,/\,265.5\,/\,0.675 & 17.32\,/\,213.8\,/\,0.493 & 21.01\,/\,105.2\,/\,0.298 \\
            & \cellcolor{gray!20}\textbf{NCS-DPS} & \cellcolor{gray!20}\textbf{13.18\,/\,247.3\,/\,0.656} & \cellcolor{gray!20}\textbf{22.29\,/\,87.41\,/\,0.298} & \cellcolor{gray!20}\textbf{23.87}\,/\,\textcolor{blue}{\textbf{72.03}}\,/\,\textcolor{blue}{\textbf{0.215}} \\
            & MPGD             & 16.53\,/\,196.2\,/\,0.413 & 14.72\,/\,261.9\,/\,\textbf{\textcolor{blue}{0.253}} & 10.69\,/\,317.4\,/\,1.079 \\
            & \cellcolor{gray!20}\textbf{NCS-MPGD}& \cellcolor{gray!20}\textbf{22.06}\,/\,\textbf{\textcolor{blue}{114.6}}\,/\,\textbf{0.376} & \cellcolor{gray!20}\textbf{22.73}\,/\,\textbf{\textcolor{blue}{93.45}}\,/\,{0.258} & \cellcolor{gray!20}\textbf{16.97}\,/\,\textbf{206.9}\,/\,\textbf{0.720} \\
            & DAPS             & {{23.42}}\,/\,135.2\,/\,\textcolor{blue}{\textbf{0.370}} & \textcolor{blue}{\textbf{26.62}}\,/\,{120.6}\,/\,{0.267} & \textcolor{blue}{\textbf{24.89}}\,/\,{102.6}\,/\,{0.275} \\
            & DDNM             & \textcolor{blue}{\textbf{24.15}}\,/\,{129.0}\,/\,0.476 & {24.35}\,/\,137.6\,/\,0.483 & {24.41}\,/\,138.3\,/\,0.470 \\
            & RED-Diff         & 15.26\,/\,313.2\,/\,1.067 & 22.59\,/\,195.7\,/\,0.566 & 23.33\,/\,171.7\,/\,0.448 \\
        \midrule
        \multirow{7}{*}{Motion Deblur}
            & DPS              & 12.05\,/\,265.5\,/\,0.676 & 17.39\,/\,218.4\,/\,0.502 & 22.75\,/\,87.47\,/\,0.268 \\
            & \cellcolor{gray!20}\textbf{NCS-DPS} & \cellcolor{gray!20}\textbf{13.39\,/\,250.8\,/\,0.648} & \cellcolor{gray!20}\textbf{24.13\,/\,77.85\,/\,0.255} & \cellcolor{gray!20}\textbf{27.25\,/\,\textcolor{blue}{34.19}\,/\,\textcolor{blue}{0.103}} \\
            & MPGD             & 16.99\,/\,165.3\,/\,0.375 & 14.81\,/\,242.0\,/\,0.253 & 13.99\,/\,227.4\,/\,0.886 \\
            & \cellcolor{gray!20}\textbf{NCS-MPGD}& \cellcolor{gray!20}\textbf{22.83\,/\,90.49}\,/\,\textcolor{blue}{\textbf{0.283}} & \cellcolor{gray!20}\textbf{23.90}\,/\,\textbf{\textcolor{blue}{62.78\,/\,0.169}} & \cellcolor{gray!20}\textbf{17.46\,/\,157.6\,/\,0.648} \\
            & DAPS             & \textbf{\textcolor{blue}{24.54}}\,/\,\textcolor{blue}{\textbf{80.48}}\,/\,{0.297} & \textcolor{blue}{\textbf{25.21}}\,/\,{68.35}\,/\,{0.272} & {{27.93}}\,/\,{43.42}\,/\,{0.181} \\
            & RED-Diff         & 15.48\,/\,289.6\,/\,0.970 & 24.96\,/\,135.5\,/\,0.365 & \textcolor{blue}{\textbf{28.85}}\,/\,61.40\,/\,0.185 \\
    \bottomrule
    \end{tabular}
    \end{sc}
    \end{footnotesize}
    \end{table*}
    % ----- END NEW TABLE -----
    
        % 表格5：带标准差的综合对比表格
        % =============================================================================
        \begin{table*}[!t]
        \centering
        \caption{Comprehensive comparison with standard deviations across Scientific Inverse Problems benchmarks.}
        \label{tab:all-results-std}
        \begin{footnotesize}
        \resizebox{\textwidth}{!}{%
        \begin{tabular}{l|ccc|cc|ccc}
        \toprule
        \multirow{2}{*}{Method} & \multicolumn{3}{c|}{Blackhole} & \multicolumn{2}{c|}{Inverse Scattering} & \multicolumn{3}{c}{MRI Knee} \\
        \cmidrule(lr){2-4} \cmidrule(lr){5-6} \cmidrule(lr){7-9}
         & PSNR $\uparrow$ & $\chi^2_{\text{cp}}$ $\downarrow$ & $\chi^2_{\text{logca}}$ $\downarrow$ & PSNR $\uparrow$ & SSIM $\uparrow$ & PSNR $\uparrow$ & SSIM $\uparrow$ & Misfit $\downarrow$ \\
        \midrule
        DPS      & 26.77{\scriptsize $\pm$4.02} & 12.09{\scriptsize $\pm$44.6} & 4.23{\scriptsize $\pm$9.49} & 29.96{\scriptsize $\pm$4.44} & .798{\scriptsize $\pm$.155} & 27.01{\scriptsize $\pm$1.45} & .592{\scriptsize $\pm$.109} & 34.49{\scriptsize $\pm$25.0} \\
        \cellcolor{gray!20}\textbf{NCS-DPS}  & \cellcolor{gray!20}\textcolor{blue}{\textbf{27.70}}{\scriptsize $\pm$3.72} & \cellcolor{gray!20}\textbf{6.18}{\scriptsize $\pm$34.8} & \cellcolor{gray!20}\textbf{2.66}{\scriptsize $\pm$11.0} & \cellcolor{gray!20}\textcolor{blue}{\textbf{30.42}}{\scriptsize $\pm$5.10} & \cellcolor{gray!20}\textbf{.899}{\scriptsize $\pm$.027} & \cellcolor{gray!20}\textbf{29.41}{\scriptsize $\pm$1.23} & \cellcolor{gray!20}\textbf{.686}{\scriptsize $\pm$.088} & \cellcolor{gray!20}\textbf{32.67}{\scriptsize $\pm$24.6} \\
        DAPS     & 25.40{\scriptsize $\pm$3.48} & \textcolor{blue}{\textbf{1.45}}{\scriptsize $\pm$0.59} & \textcolor{blue}{\textbf{1.15}}{\scriptsize $\pm$0.20} & 28.50{\scriptsize $\pm$4.57} & .889{\scriptsize $\pm$.029} & 20.89{\scriptsize $\pm$6.13} & .343{\scriptsize $\pm$.223} & \textcolor{blue}{\textbf{29.77}}{\scriptsize $\pm$23.9} \\
        REDDiff  & 23.81{\scriptsize $\pm$3.95} & 1.93{\scriptsize $\pm$1.71} & 2.04{\scriptsize $\pm$2.86} & 29.90{\scriptsize $\pm$4.93} & \textcolor{blue}{\textbf{.953}}{\scriptsize $\pm$.040} & \textcolor{blue}{\textbf{30.75}}{\scriptsize $\pm$1.58} & \textcolor{blue}{\textbf{.745}}{\scriptsize $\pm$.108} & 30.55{\scriptsize $\pm$24.5} \\
        \bottomrule
        \end{tabular}%
        }
        \end{footnotesize}
        \end{table*}

For scientific inverse problems, we evaluate three tasks, including black hole imaging, linear inverse scattering, and compressed sensing MRI, using the same parameters as those reported in Appendix~B of~\citet{zheng2025inversebench}. Under the default setting $K = \sqrt{n}$, the NCS-DPS method consistently improves upon DPS and outperforms state-of-the-art methods with hyperparameter search on several tasks.

For nonlinear tasks, the improvements brought by NCS are primarily observed in the low-step regime. We report the corresponding experimental results in Appendix~\ref{app:nonlinear}.

\subsection{Compression Experiments} 

We conducted experiments to evaluate the compression performance of our NCS-based approach. We compared our method against several state-of-the-art compression algorithms including PSC-P, BPG, and HiFiC \citep{Albalawi2015BPGHardware, elata2024zero, muckley2023improving, mentzer2020high} on the Kodak24 dataset \citep{franzen1999kodak} and the ImageNet dataset, using the same experimental setup as in DDCM \citet{ohayon2025compressed}. The results demonstrate that our method maintains high fidelity while achieving significant storage savings compared to baseline approaches.

% ----- BEGIN NEW TABLE -----
% Reworked into two separate tables: FFHQ and ImageNet

% \subsection{Influence of the number of noise vectors}

% \subsection{Influence of the number of steps}

\section{Conclusion}
In this work, we propose NCS, a framework for approximating the measurement score in DMs via an optimal linear combination of noise vectors. We derive a closed-form solution to the alignment optimization and show that NCS unifies and generalizes existing inverse problem solvers, achieving strong performance with fewer diffusion steps and improved stability. We further demonstrate that NCS substantially accelerates DDCM while maintaining comparable reconstruction quality. Several open questions remain, such as understanding the limited improvement of NCS-DAPS, which we plan to investigate in future work.

\bibliographystyle{plainnat}
\bibliography{references}

%%%%%%%%%%%%%%%%%%%%%%%%%%%%%%%%%%%%%%%%%%%%%%%%%%%%%%%%%%%%%%%%%%%%%%%%%%%%%%%
%%%%%%%%%%%%%%%%%%%%%%%%%%%%%%%%%%%%%%%%%%%%%%%%%%%%%%%%%%%%%%%%%%%%%%%%%%%%%%%
% APPENDIX
%%%%%%%%%%%%%%%%%%%%%%%%%%%%%%%%%%%%%%%%%%%%%%%%%%%%%%%%%%%%%%%%%%%%%%%%%%%%%%%
%%%%%%%%%%%%%%%%%%%%%%%%%%%%%%%%%%%%%%%%%%%%%%%%%%%%%%%%%%%%%%%%%%%%%%%%%%%%%%%
\newpage
\appendix

% ---------- Black Hole Imaging ----------
\section{Additional Visual Results for Inverse Problems}
\begin{figure}[H]
    \centering
    \includegraphics[width=0.95\textwidth]{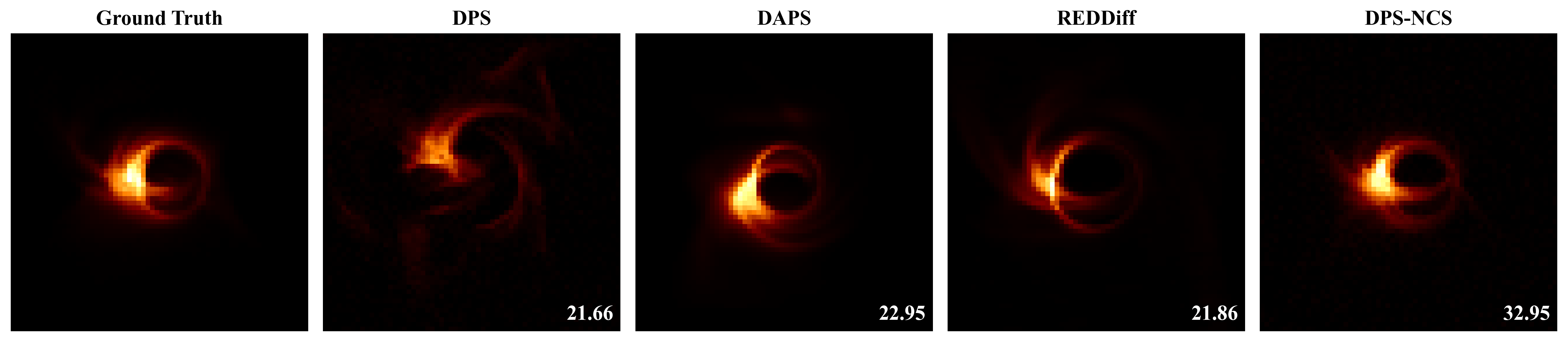}\\[2pt]
    \includegraphics[width=0.95\textwidth]{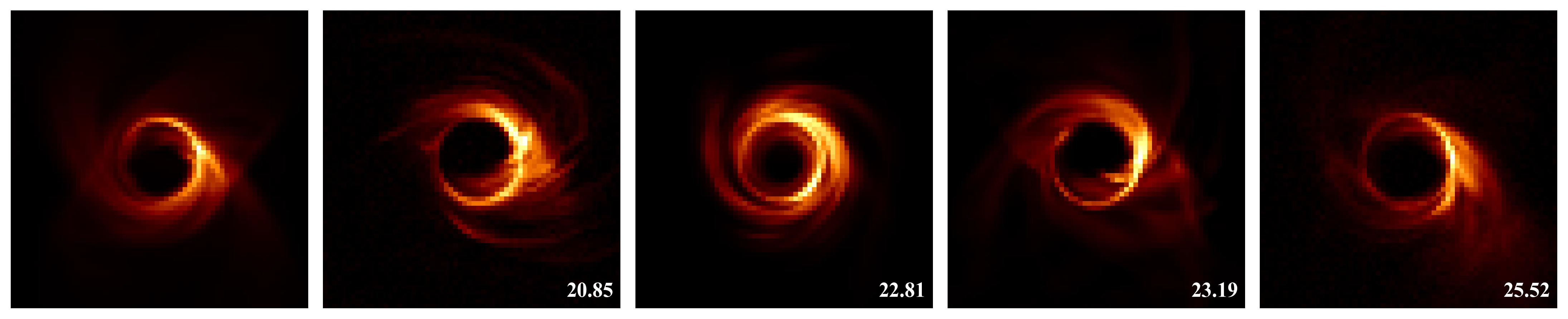}\\[2pt]
    \includegraphics[width=0.95\textwidth]{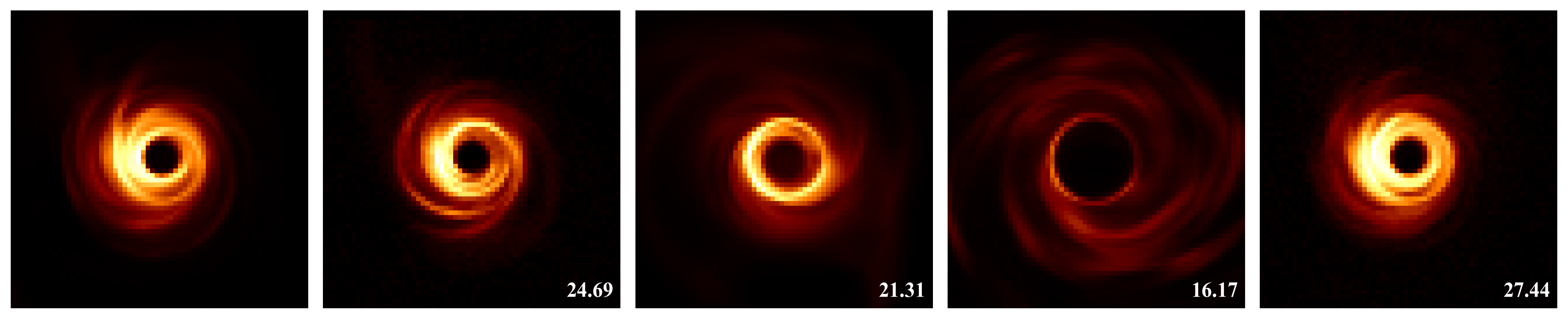}\\[2pt]
    \includegraphics[width=0.95\textwidth]{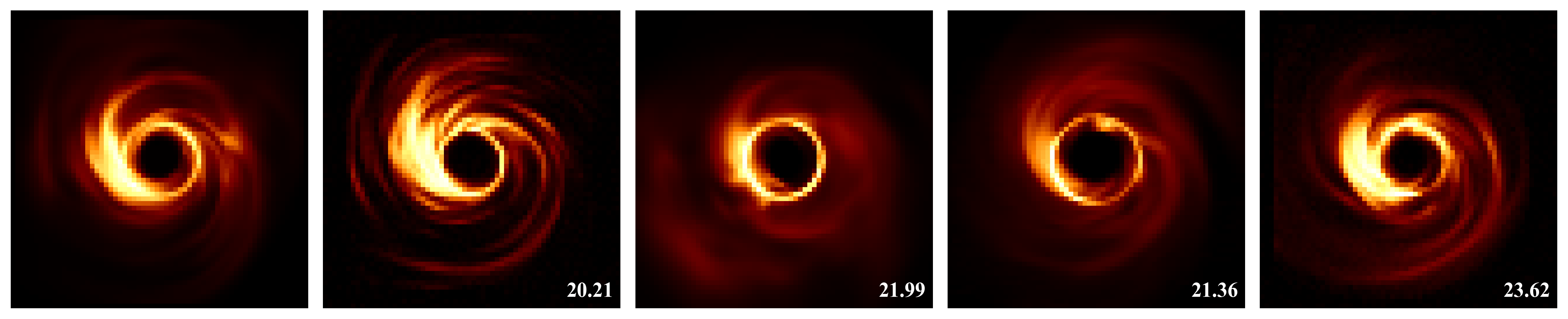}\\[2pt]
    \includegraphics[width=0.95\textwidth]{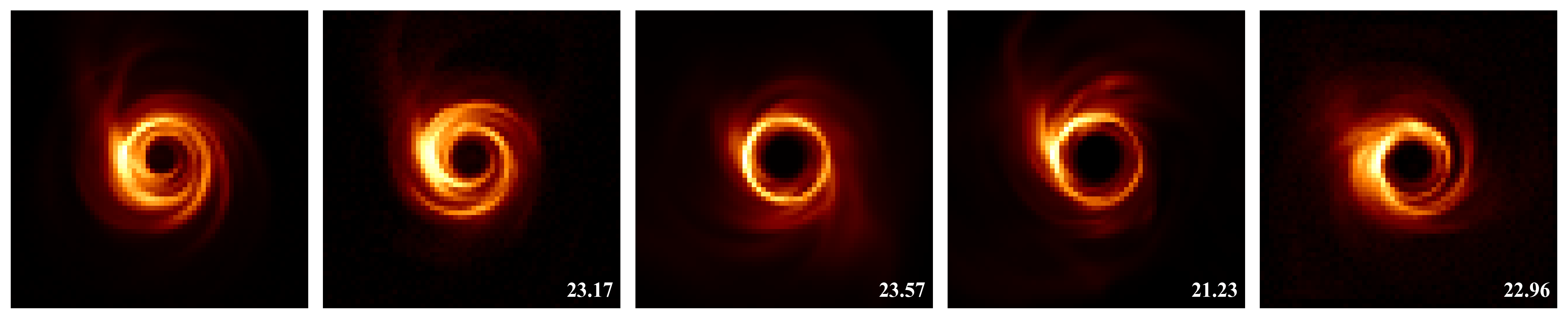}
    \caption{Additional visual results for \textbf{Black Hole Imaging}. Each row shows a different sample comparing ground truth and reconstructions from DPS, DAPS, REDDiff, and NCS-DPS (Ours).}
    \label{fig:app_blackhole}
\end{figure}

\clearpage
% ---------- Inverse Scattering ----------
\begin{figure}[H]
    \centering
    \includegraphics[width=0.95\textwidth]{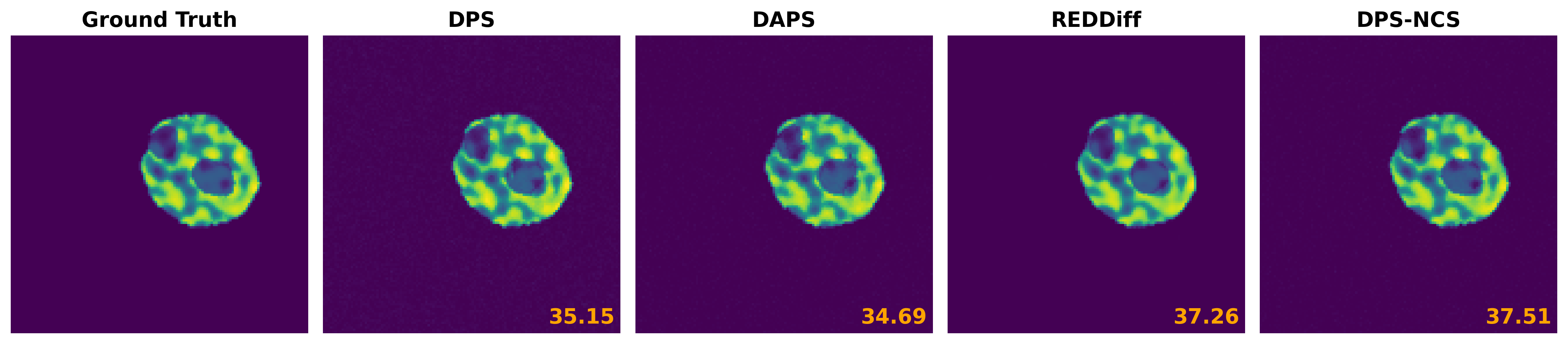}\\[2pt]
    \includegraphics[width=0.95\textwidth]{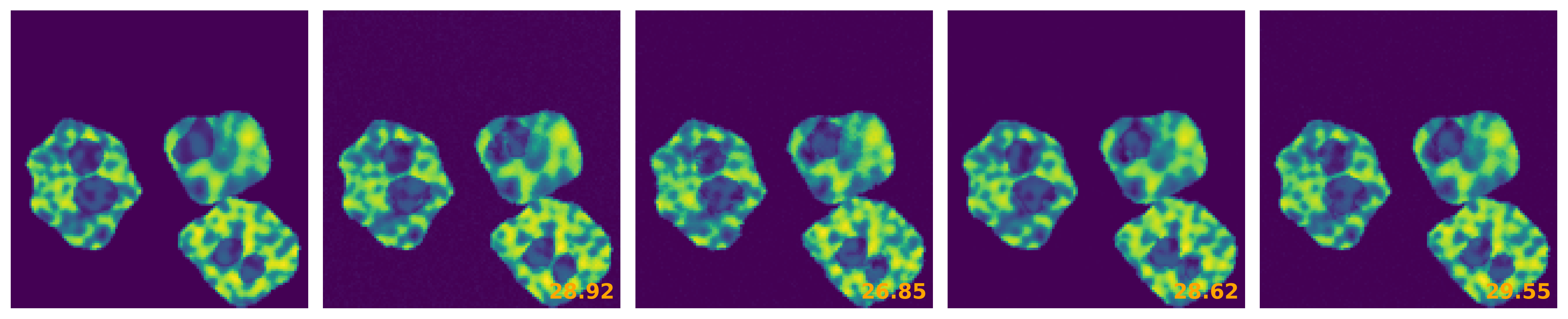}\\[2pt]
    \includegraphics[width=0.95\textwidth]{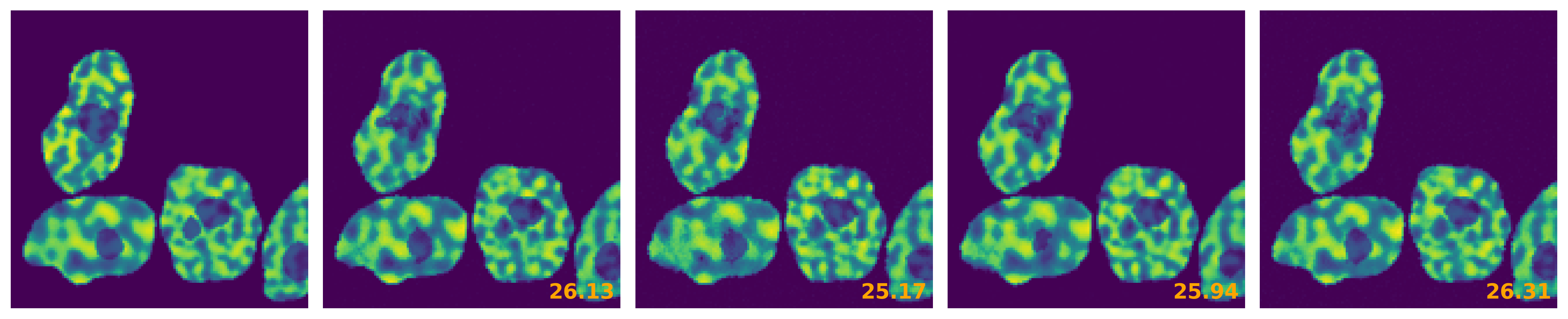}\\[2pt]
    \includegraphics[width=0.95\textwidth]{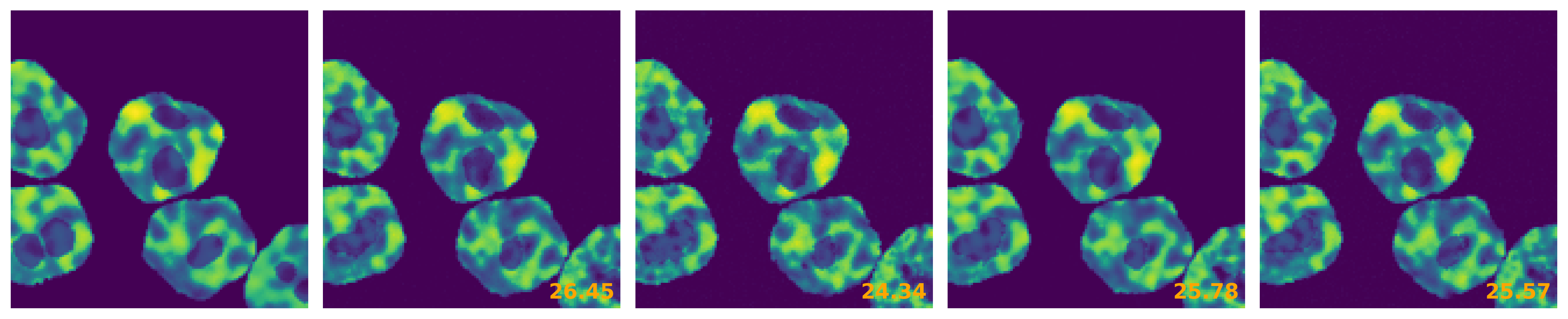}\\[2pt]
    \includegraphics[width=0.95\textwidth]{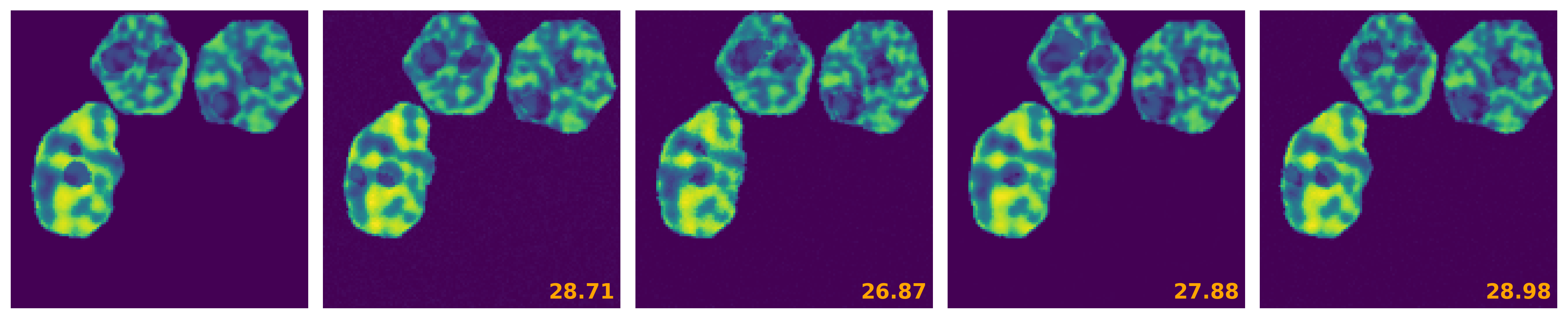}
    \caption{Additional visual results for \textbf{Inverse Scattering}. Each row shows a different sample comparing ground truth and reconstructions from DPS, DAPS, REDDiff, and NCS-DPS (Ours).}
    \label{fig:app_inv_scatter}
\end{figure}

% ---------- Compressed Sensing MRI ----------
\begin{figure}[H]
    \centering
    \includegraphics[width=0.95\textwidth]{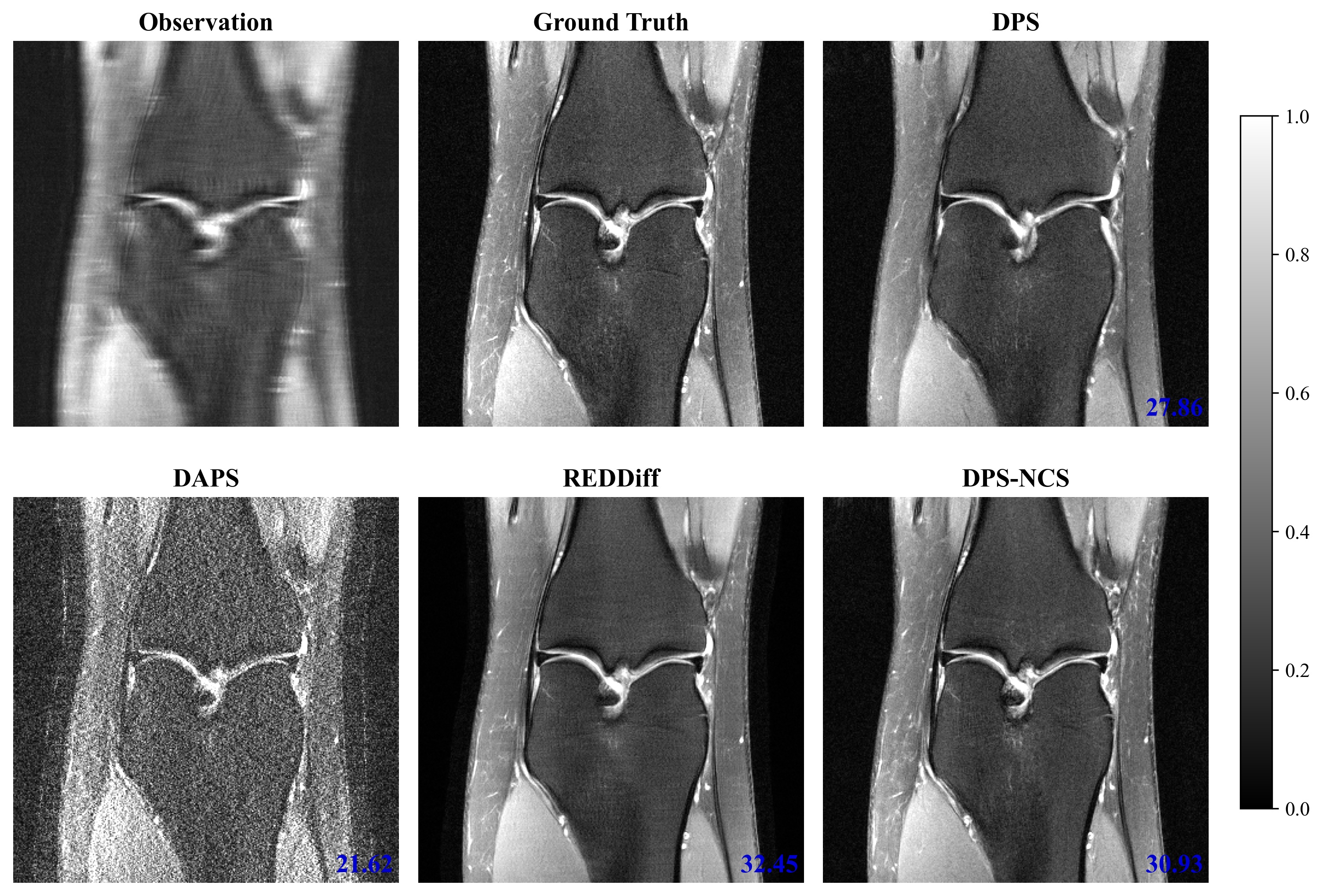}\\[4pt]
    \includegraphics[width=0.95\textwidth]{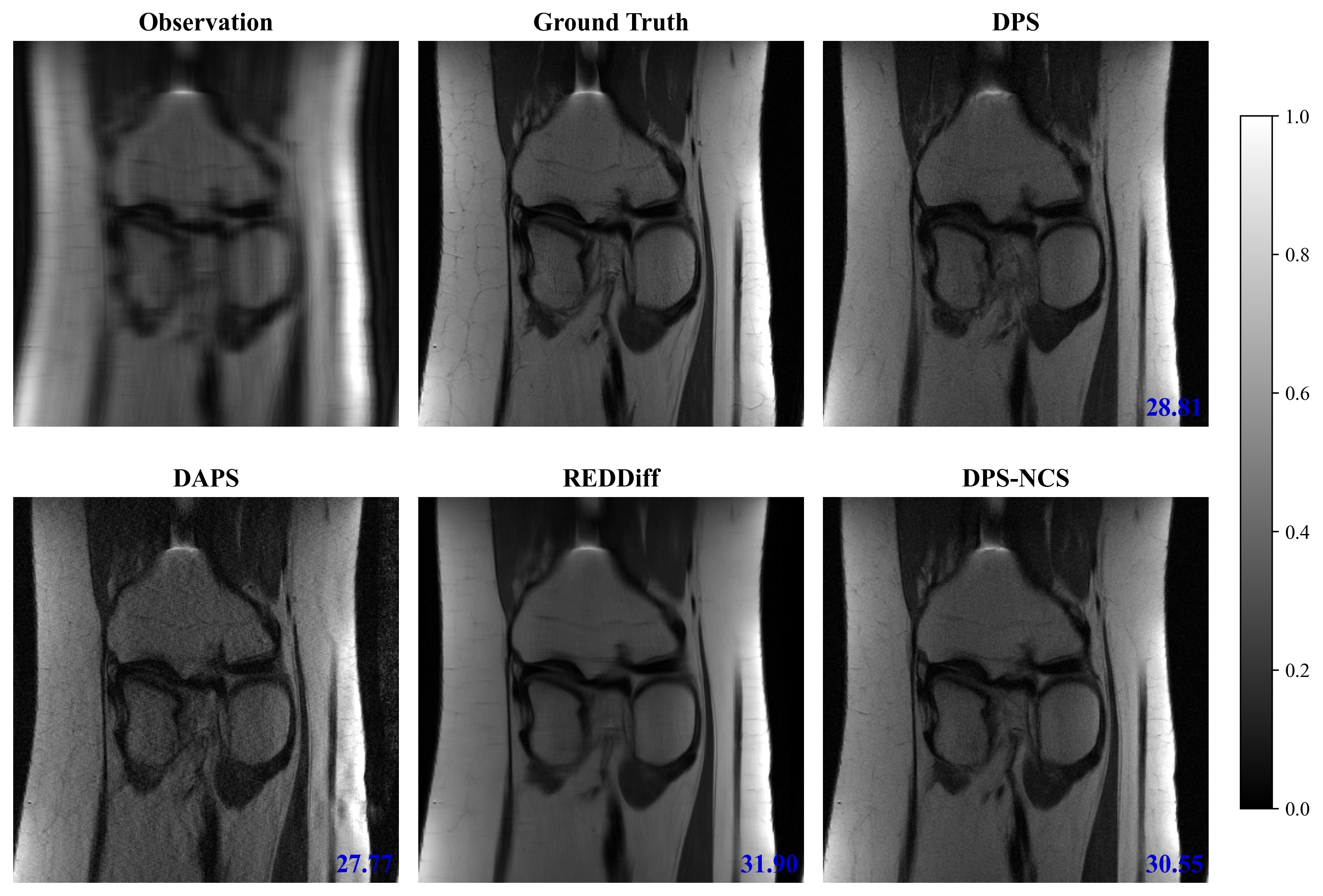}
    \caption{Additional visual results for \textbf{Compressed Sensing MRI}. Each row shows a different sample comparing ground truth and reconstructions from DPS, DAPS, REDDiff, and NCS-DPS (Ours).}
    \label{fig:app_mri}
\end{figure}

\clearpage
% ---------------- VISUAL COMPARISON (1000 steps) ----------------
\begin{figure}[p]
    \centering
    \begin{tabular}{cccccc}
    \multicolumn{6}{c}{\textbf{Gaussian Deblur - 1000 Steps}} \\[2pt]
    \scriptsize Input & \scriptsize Label & \scriptsize DAPS & \scriptsize DDNM & \scriptsize REDDiff & \scriptsize NCS-DPS \\
    \includegraphics[width=0.13\textwidth]{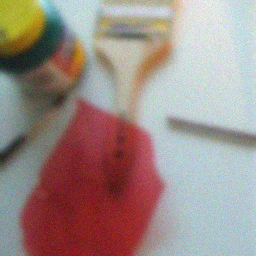} &
    \includegraphics[width=0.13\textwidth]{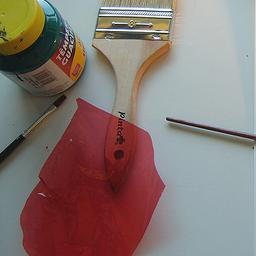} &
    \includegraphics[width=0.13\textwidth]{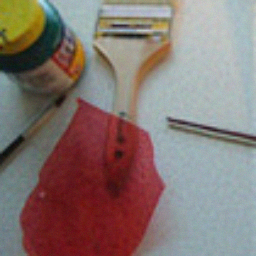} &
    \includegraphics[width=0.13\textwidth]{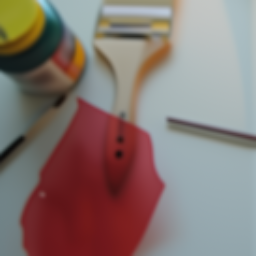} &
    \includegraphics[width=0.13\textwidth]{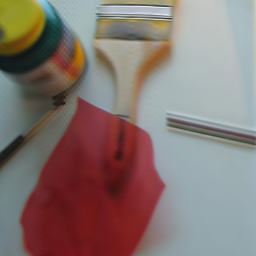} &
    \includegraphics[width=0.13\textwidth]{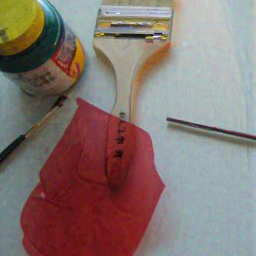} \\
    \includegraphics[width=0.13\textwidth]{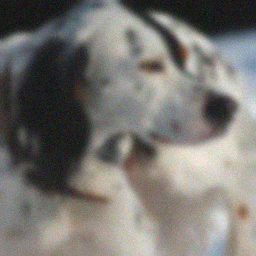} &
    \includegraphics[width=0.13\textwidth]{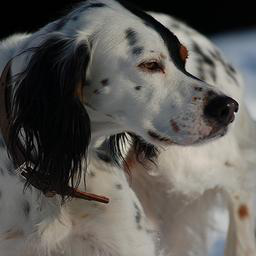} &
    \includegraphics[width=0.13\textwidth]{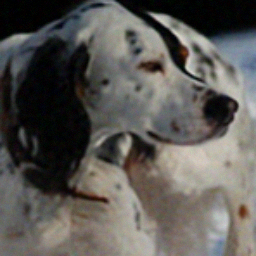} &
    \includegraphics[width=0.13\textwidth]{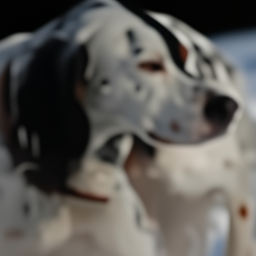} &
    \includegraphics[width=0.13\textwidth]{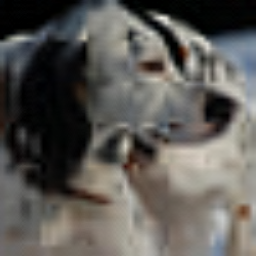} &
    \includegraphics[width=0.13\textwidth]{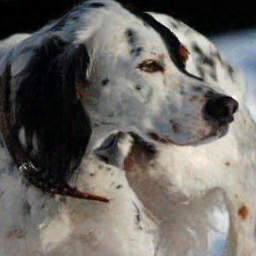} \\
    \includegraphics[width=0.13\textwidth]{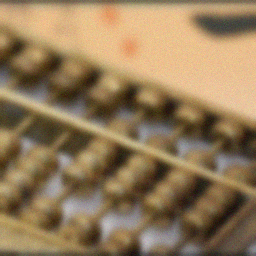} &
    \includegraphics[width=0.13\textwidth]{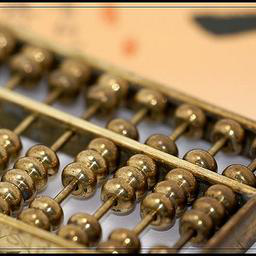} &
    \includegraphics[width=0.13\textwidth]{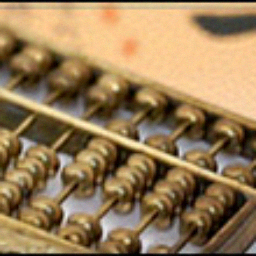} &
    \includegraphics[width=0.13\textwidth]{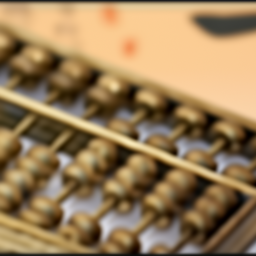} &
    \includegraphics[width=0.13\textwidth]{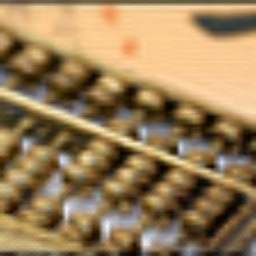} &
    \includegraphics[width=0.13\textwidth]{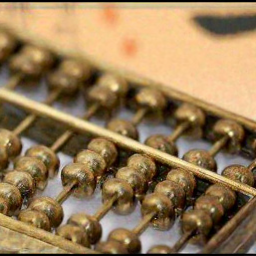} \\[6pt]
    \multicolumn{6}{c}{\textbf{Super Resolution 4$\times$ - 1000 Steps}} \\[2pt]
    \scriptsize Input & \scriptsize Label & \scriptsize DAPS & \scriptsize DDNM & \scriptsize REDDiff & \scriptsize NCS-DPS \\
    \includegraphics[width=0.13\textwidth]{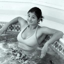} &
    \includegraphics[width=0.13\textwidth]{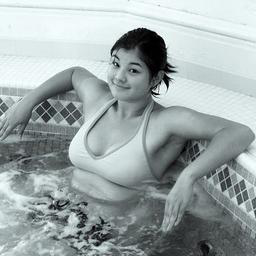} &
    \includegraphics[width=0.13\textwidth]{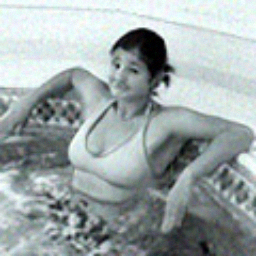} &
    \includegraphics[width=0.13\textwidth]{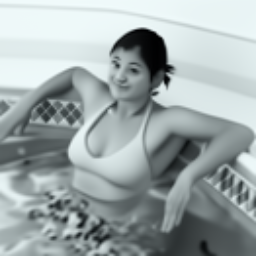} &
    \includegraphics[width=0.13\textwidth]{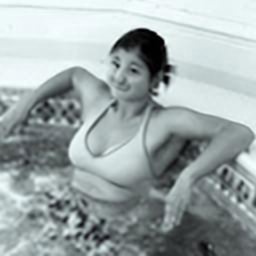} &
    \includegraphics[width=0.13\textwidth]{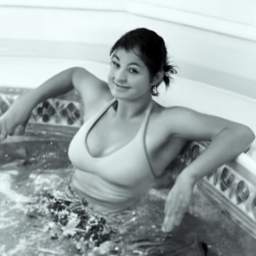} \\
    \includegraphics[width=0.13\textwidth]{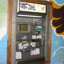} &
    \includegraphics[width=0.13\textwidth]{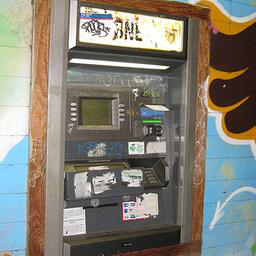} &
    \includegraphics[width=0.13\textwidth]{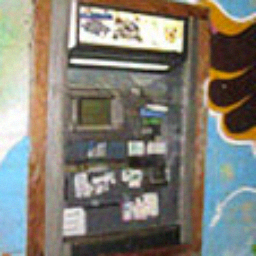} &
    \includegraphics[width=0.13\textwidth]{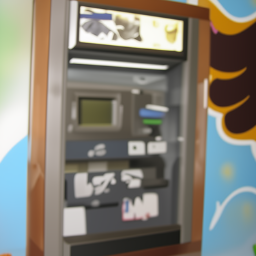} &
    \includegraphics[width=0.13\textwidth]{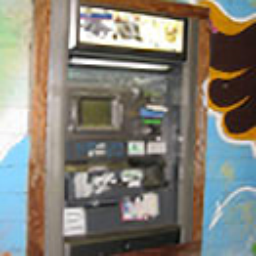} &
    \includegraphics[width=0.13\textwidth]{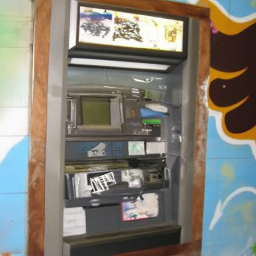} \\
    \includegraphics[width=0.13\textwidth]{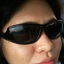} &
    \includegraphics[width=0.13\textwidth]{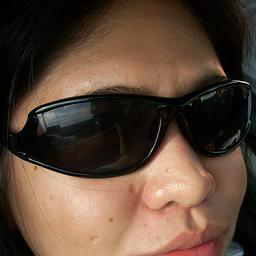} &
    \includegraphics[width=0.13\textwidth]{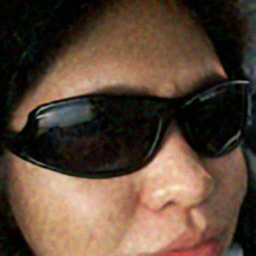} &
    \includegraphics[width=0.13\textwidth]{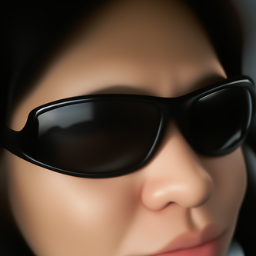} &
    \includegraphics[width=0.13\textwidth]{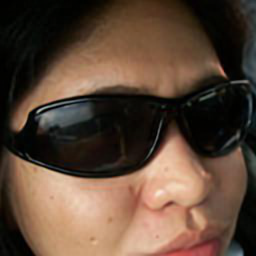} &
    \includegraphics[width=0.13\textwidth]{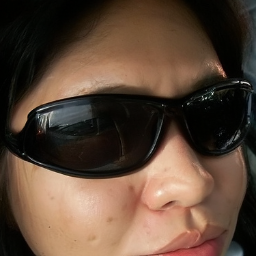} \\[6pt]
    \multicolumn{6}{c}{\textbf{Super Resolution 8$\times$ - 1000 Steps}} \\[2pt]
    \scriptsize Input & \scriptsize Label & \scriptsize DAPS & \scriptsize DDNM & \scriptsize REDDiff & \scriptsize NCS-DPS \\
    \includegraphics[width=0.13\textwidth]{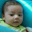} &
    \includegraphics[width=0.13\textwidth]{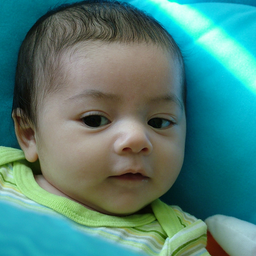} &
    \includegraphics[width=0.13\textwidth]{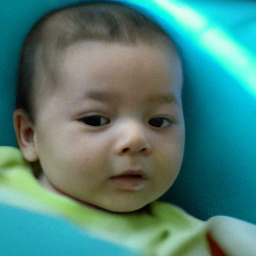} &
    \includegraphics[width=0.13\textwidth]{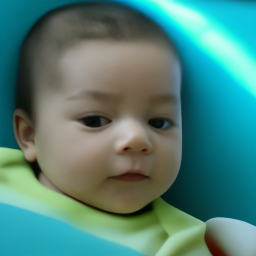} &
    \includegraphics[width=0.13\textwidth]{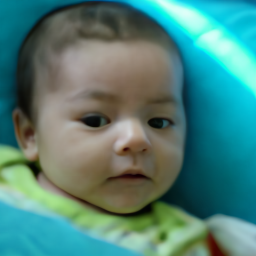} &
    \includegraphics[width=0.13\textwidth]{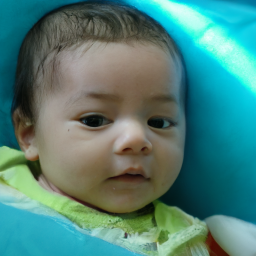} \\
    \includegraphics[width=0.13\textwidth]{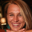} &
    \includegraphics[width=0.13\textwidth]{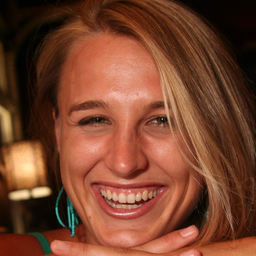} &
    \includegraphics[width=0.13\textwidth]{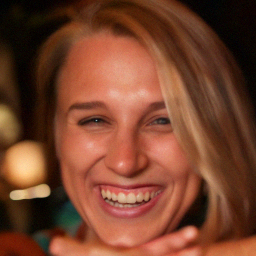} &
    \includegraphics[width=0.13\textwidth]{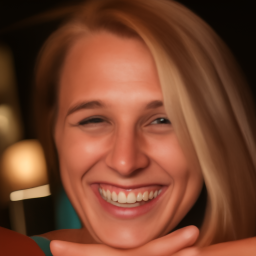} &
    \includegraphics[width=0.13\textwidth]{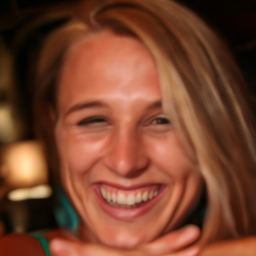} &
    \includegraphics[width=0.13\textwidth]{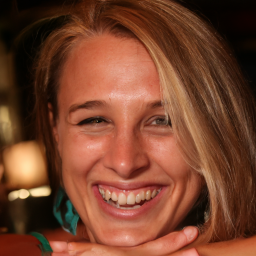} \\
    \includegraphics[width=0.13\textwidth]{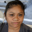} &
    \includegraphics[width=0.13\textwidth]{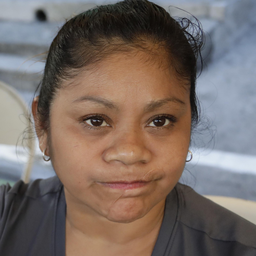} &
    \includegraphics[width=0.13\textwidth]{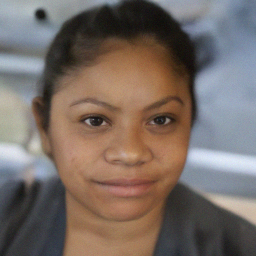} &
    \includegraphics[width=0.13\textwidth]{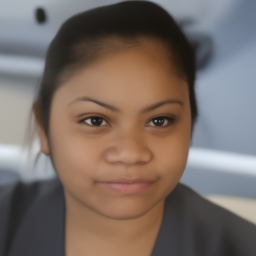} &
    \includegraphics[width=0.13\textwidth]{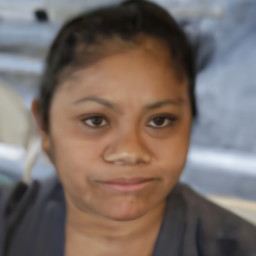} &
    \includegraphics[width=0.13\textwidth]{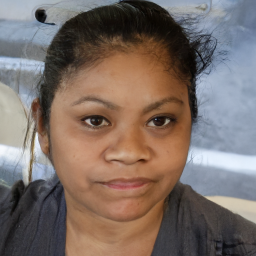} \\
    \end{tabular}
    \caption{Visual comparison on Gaussian deblur and super-resolution tasks with 1000 sampling steps. Each section shows three samples comparing input, ground truth, and reconstructions from DAPS, DDNM, REDDiff, and NCS-DPS (Ours).}
    \label{fig:app_visual_1000}
\end{figure}

\clearpage
% % ---------------- GAUSSIAN BLUR RESULTS ----------------
\begin{figure}[p]
    \centering
    \begin{tabular}{cccccc}
    \multicolumn{6}{c}{\textbf{Gaussian Blur - Timestep 20}} \\
    Input & Label & DPS & NCS-DPS & MPGD & NCS-MPGD \\
    \includegraphics[width=0.13\textwidth]{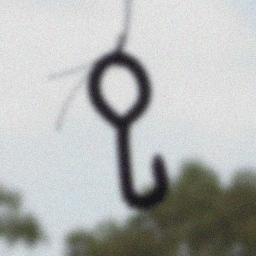} &
    \includegraphics[width=0.13\textwidth]{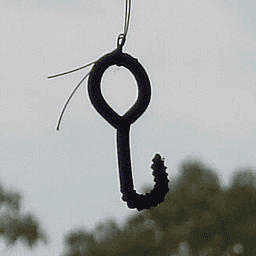} &
    \includegraphics[width=0.13\textwidth]{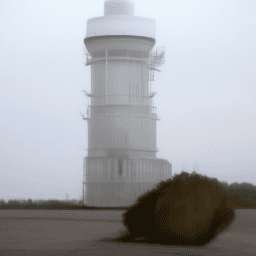} &
    \includegraphics[width=0.13\textwidth]{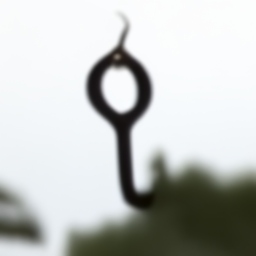} &
    \includegraphics[width=0.13\textwidth]{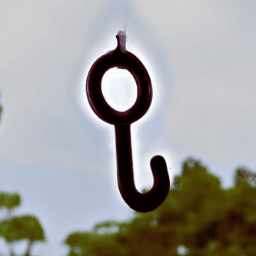} &
    \includegraphics[width=0.13\textwidth]{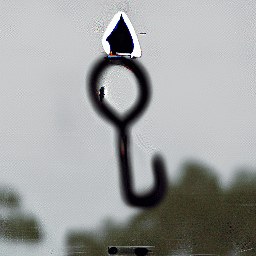} \\
    
    \includegraphics[width=0.13\textwidth]{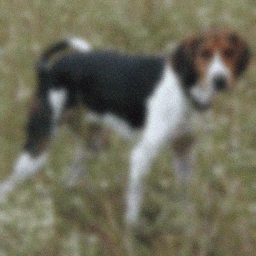} &
    \includegraphics[width=0.13\textwidth]{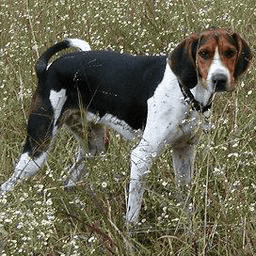} &
    \includegraphics[width=0.13\textwidth]{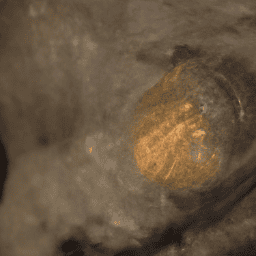} &
    \includegraphics[width=0.13\textwidth]{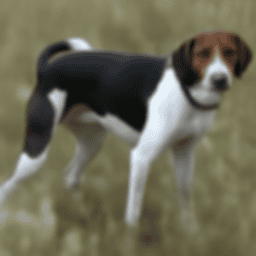} &
    \includegraphics[width=0.13\textwidth]{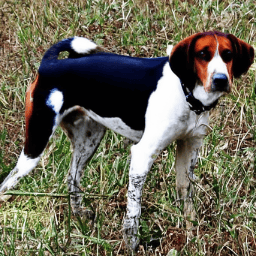} &
    \includegraphics[width=0.13\textwidth]{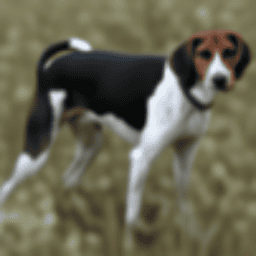} \\
    
    \includegraphics[width=0.13\textwidth]{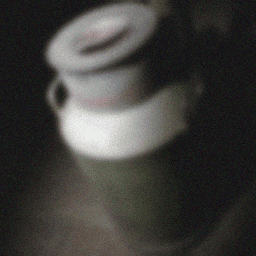} &
    \includegraphics[width=0.13\textwidth]{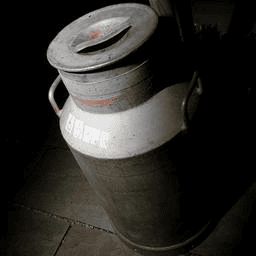} &
    \includegraphics[width=0.13\textwidth]{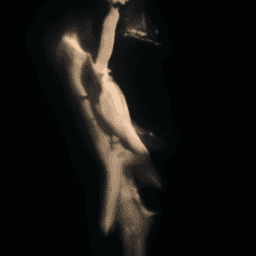} &
    \includegraphics[width=0.13\textwidth]{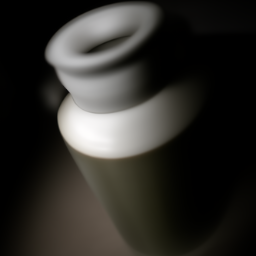} &
    \includegraphics[width=0.13\textwidth]{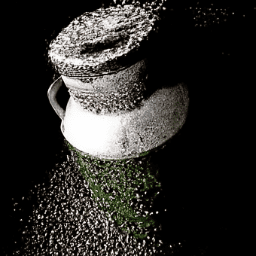} &
    \includegraphics[width=0.13\textwidth]{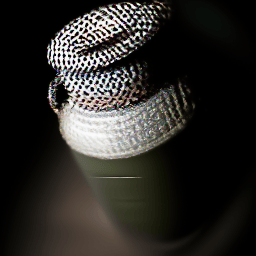} \\
    
    \multicolumn{6}{c}{\textbf{Gaussian Blur - Timestep 100}} \\
    \includegraphics[width=0.13\textwidth]{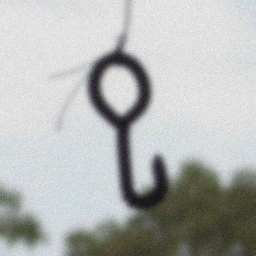} &
    \includegraphics[width=0.13\textwidth]{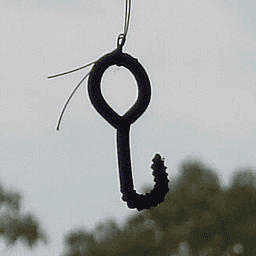} &
    \includegraphics[width=0.13\textwidth]{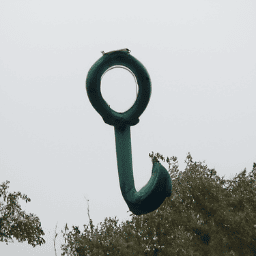} &
    \includegraphics[width=0.13\textwidth]{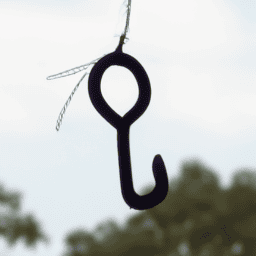} &
    \includegraphics[width=0.13\textwidth]{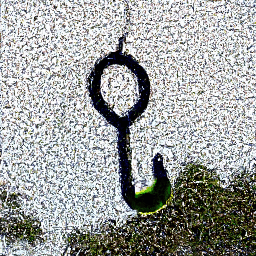} &
    \includegraphics[width=0.13\textwidth]{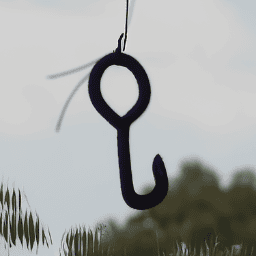} \\
    
    \includegraphics[width=0.13\textwidth]{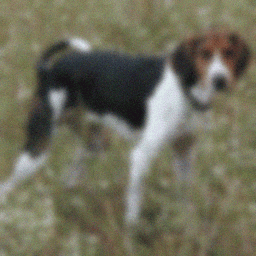} &
    \includegraphics[width=0.13\textwidth]{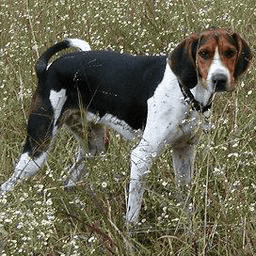} &
    \includegraphics[width=0.13\textwidth]{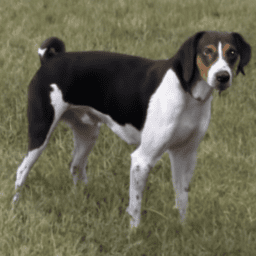} &
    \includegraphics[width=0.13\textwidth]{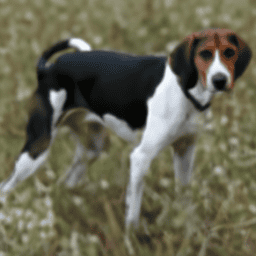} &
    \includegraphics[width=0.13\textwidth]{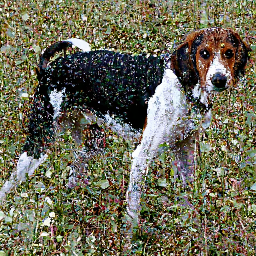} &
    \includegraphics[width=0.13\textwidth]{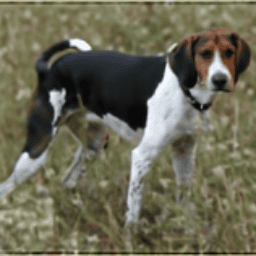} \\
    
    \includegraphics[width=0.13\textwidth]{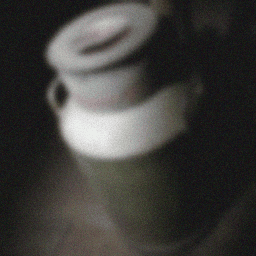} &
    \includegraphics[width=0.13\textwidth]{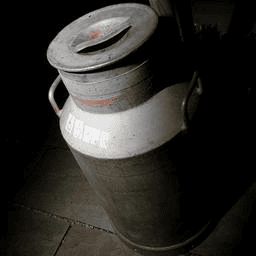} &
    \includegraphics[width=0.13\textwidth]{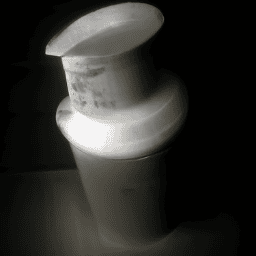} &
    \includegraphics[width=0.13\textwidth]{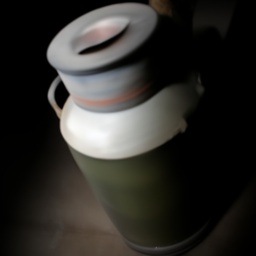} &
    \includegraphics[width=0.13\textwidth]{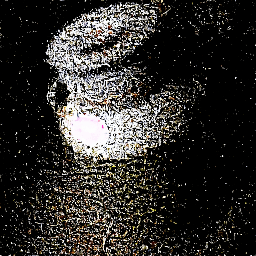} &
    \includegraphics[width=0.13\textwidth]{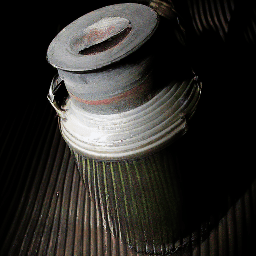} \\
    
    \multicolumn{6}{c}{\textbf{Gaussian Blur - Timestep 1000}} \\
    \includegraphics[width=0.13\textwidth]{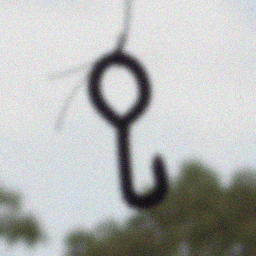} &
    \includegraphics[width=0.13\textwidth]{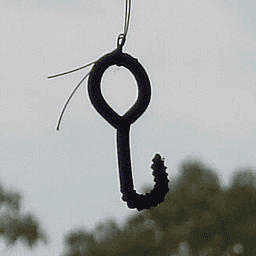} &
    \includegraphics[width=0.13\textwidth]{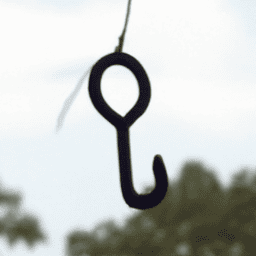} &
    \includegraphics[width=0.13\textwidth]{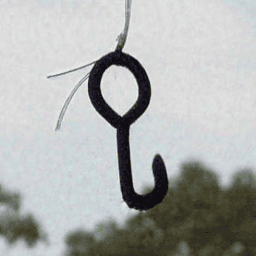} &
    \includegraphics[width=0.13\textwidth]{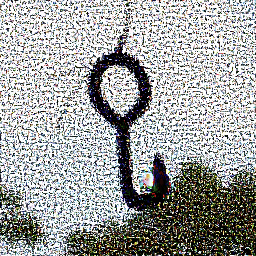} &
    \includegraphics[width=0.13\textwidth]{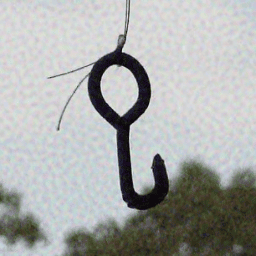} \\
    
    \includegraphics[width=0.13\textwidth]{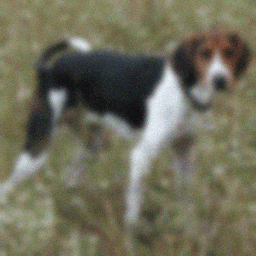} &
    \includegraphics[width=0.13\textwidth]{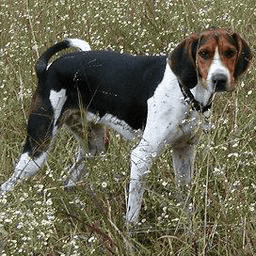} &
    \includegraphics[width=0.13\textwidth]{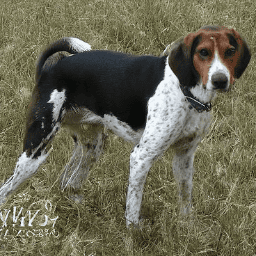} &
    \includegraphics[width=0.13\textwidth]{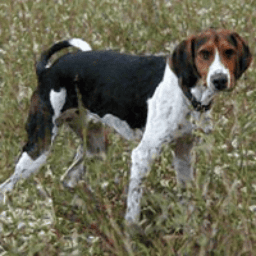} &
    \includegraphics[width=0.13\textwidth]{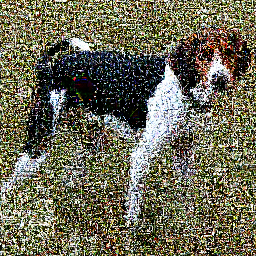} &
    \includegraphics[width=0.13\textwidth]{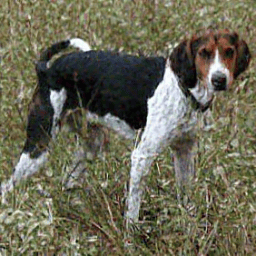} \\
    
    \includegraphics[width=0.13\textwidth]{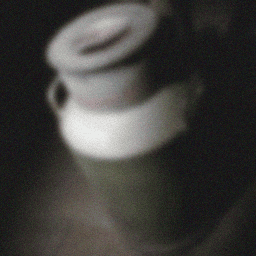} &
    \includegraphics[width=0.13\textwidth]{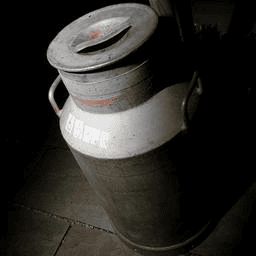} &
    \includegraphics[width=0.13\textwidth]{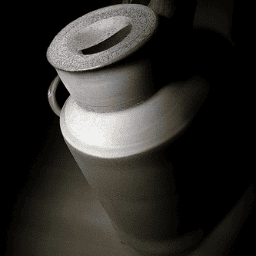} &
    \includegraphics[width=0.13\textwidth]{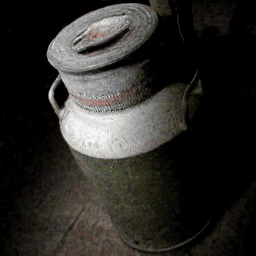} &
    \includegraphics[width=0.13\textwidth]{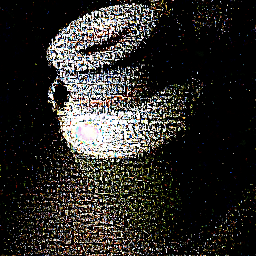} &
    \includegraphics[width=0.13\textwidth]{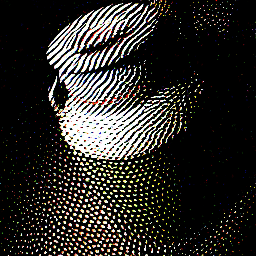} \\
    \end{tabular}
    \caption{Visual comparison on Gaussian Blur task. Each row shows results for one image across different methods. The three sections correspond to timesteps 20, 100, and 1000 respectively.}
    \label{fig:gaussian_blur_results}
\end{figure}
\clearpage
% ---------------- INPAINTING RANDOM RESULTS ----------------
\begin{figure}[p]
    \centering
    \begin{tabular}{cccccc}
    \multicolumn{6}{c}{\textbf{Inpainting Random - Timestep 20}} \\
    Input & Label & DPS & NCS-DPS & MPGD & NCS-MPGD \\
    \includegraphics[width=0.13\textwidth]{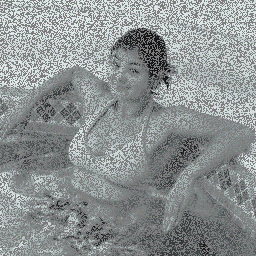} &
    \includegraphics[width=0.13\textwidth]{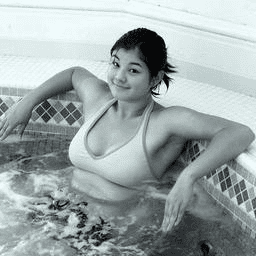} &
    \includegraphics[width=0.13\textwidth]{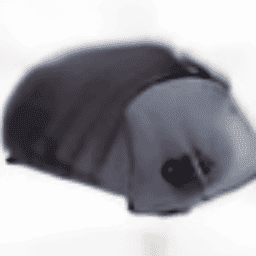} &
    \includegraphics[width=0.13\textwidth]{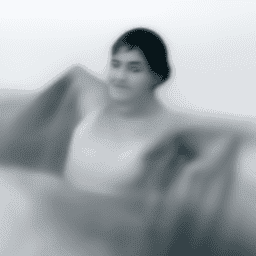} &
    \includegraphics[width=0.13\textwidth]{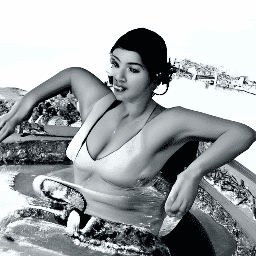} &
    \includegraphics[width=0.13\textwidth]{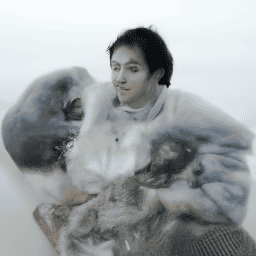} \\
    
    \includegraphics[width=0.13\textwidth]{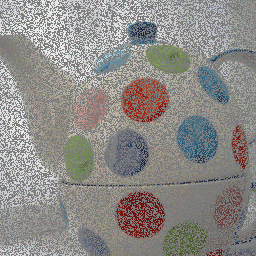} &
    \includegraphics[width=0.13\textwidth]{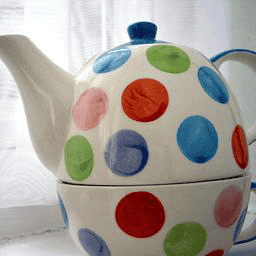} &
    \includegraphics[width=0.13\textwidth]{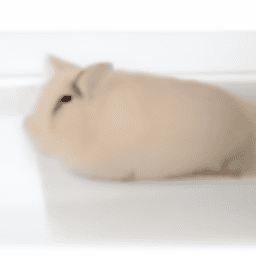} &
    \includegraphics[width=0.13\textwidth]{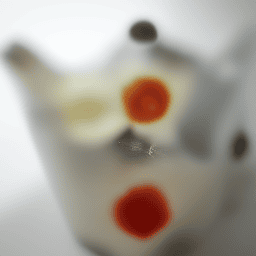} &
    \includegraphics[width=0.13\textwidth]{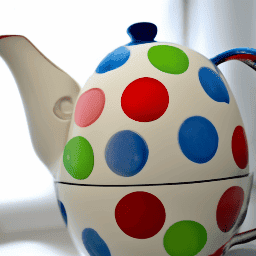} &
    \includegraphics[width=0.13\textwidth]{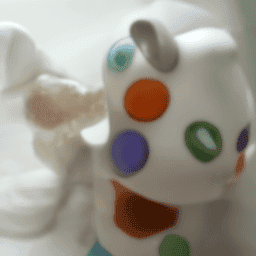} \\
    
    \includegraphics[width=0.13\textwidth]{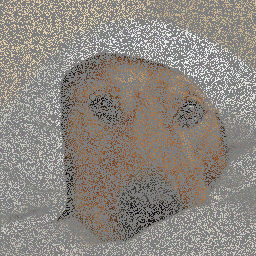} &
    \includegraphics[width=0.13\textwidth]{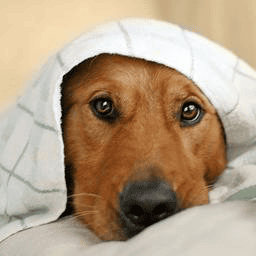} &
    \includegraphics[width=0.13\textwidth]{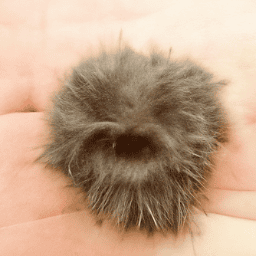} &
    \includegraphics[width=0.13\textwidth]{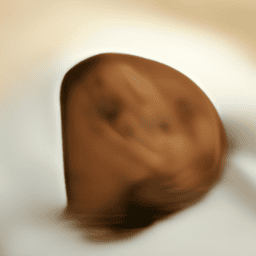} &
    \includegraphics[width=0.13\textwidth]{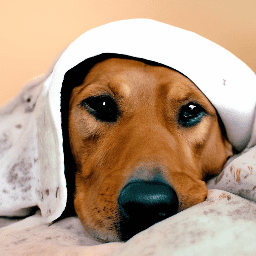} &
    \includegraphics[width=0.13\textwidth]{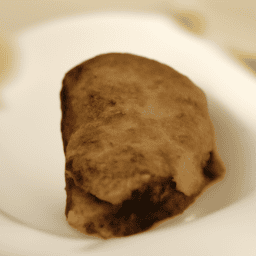} \\
    
    \multicolumn{6}{c}{\textbf{Inpainting Random - Timestep 100}} \\
    \includegraphics[width=0.13\textwidth]{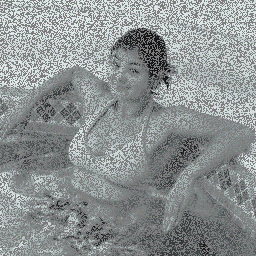} &
    \includegraphics[width=0.13\textwidth]{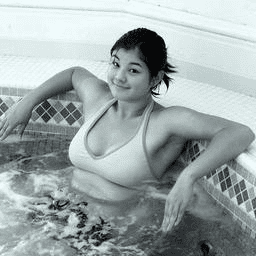} &
    \includegraphics[width=0.13\textwidth]{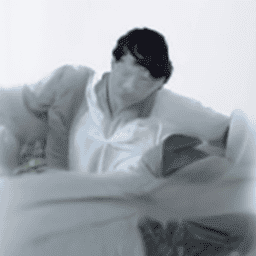} &
    \includegraphics[width=0.13\textwidth]{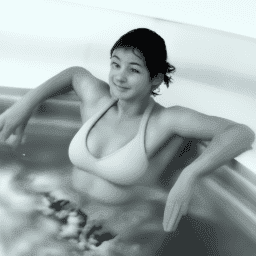} &
    \includegraphics[width=0.13\textwidth]{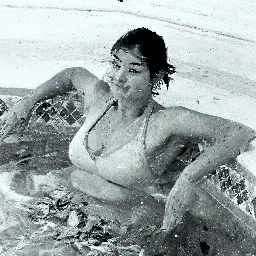} &
    \includegraphics[width=0.13\textwidth]{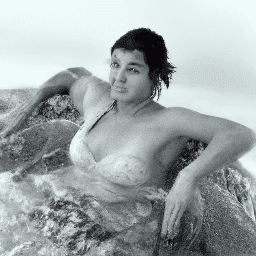} \\
    
    \includegraphics[width=0.13\textwidth]{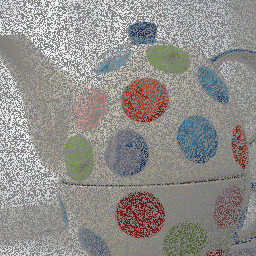} &
    \includegraphics[width=0.13\textwidth]{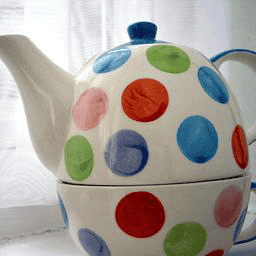} &
    \includegraphics[width=0.13\textwidth]{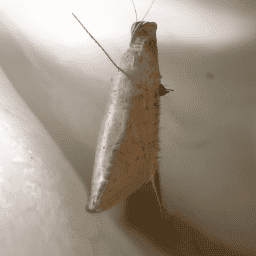} &
    \includegraphics[width=0.13\textwidth]{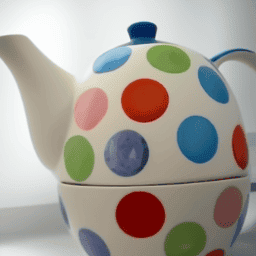} &
    \includegraphics[width=0.13\textwidth]{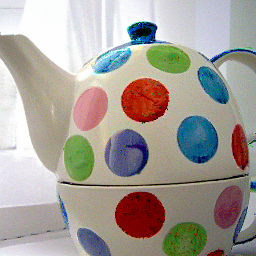} &
    \includegraphics[width=0.13\textwidth]{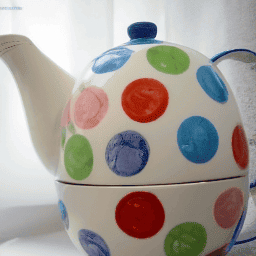} \\
    
    \includegraphics[width=0.13\textwidth]{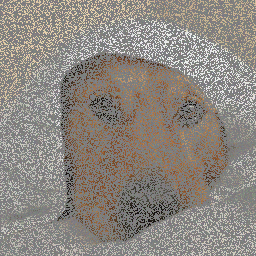} &
    \includegraphics[width=0.13\textwidth]{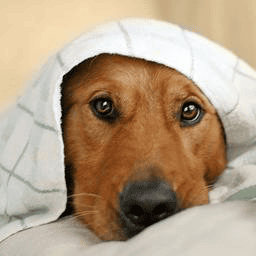} &
    \includegraphics[width=0.13\textwidth]{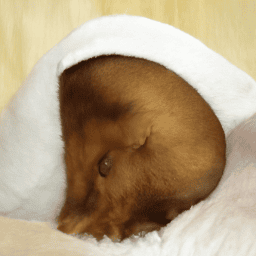} &
    \includegraphics[width=0.13\textwidth]{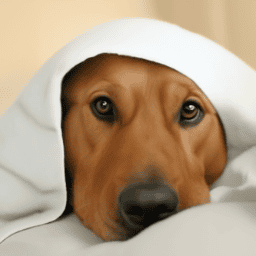} &
    \includegraphics[width=0.13\textwidth]{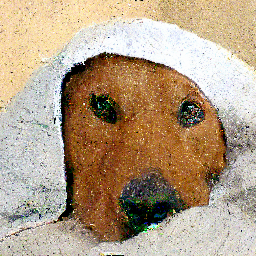} &
    \includegraphics[width=0.13\textwidth]{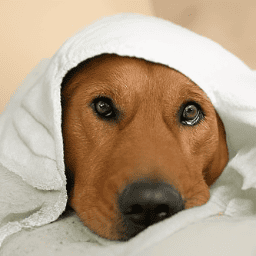} \\
    
    \multicolumn{6}{c}{\textbf{Inpainting Random - Timestep 1000}} \\
    \includegraphics[width=0.13\textwidth]{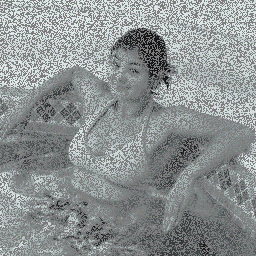} &
    \includegraphics[width=0.13\textwidth]{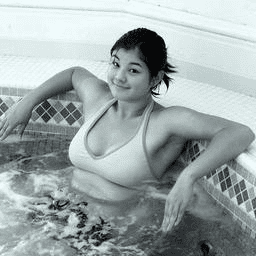} &
    \includegraphics[width=0.13\textwidth]{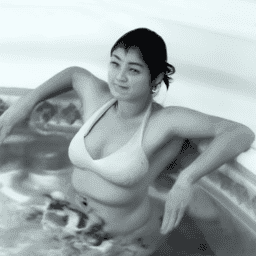} &
    \includegraphics[width=0.13\textwidth]{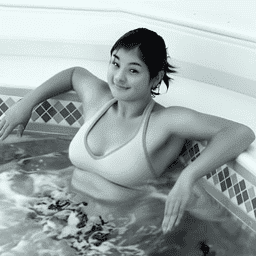} &
    \includegraphics[width=0.13\textwidth]{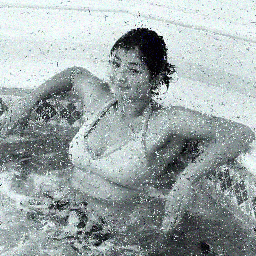} &
    \includegraphics[width=0.13\textwidth]{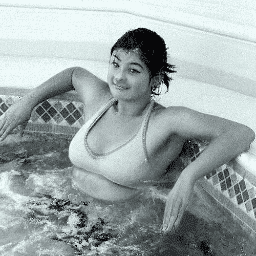} \\
    
    \includegraphics[width=0.13\textwidth]{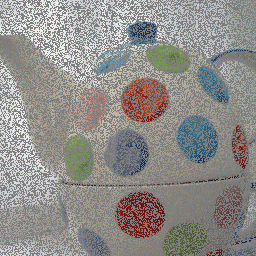} &
    \includegraphics[width=0.13\textwidth]{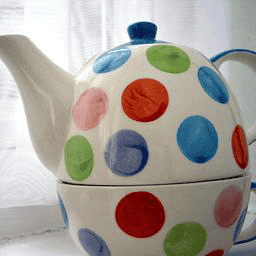} &
    \includegraphics[width=0.13\textwidth]{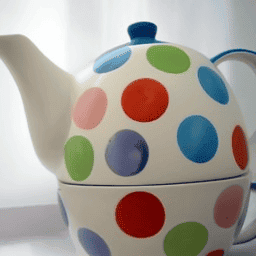} &
    \includegraphics[width=0.13\textwidth]{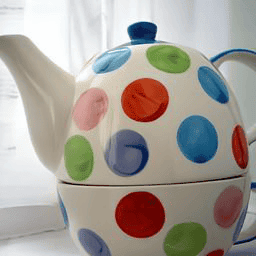} &
    \includegraphics[width=0.13\textwidth]{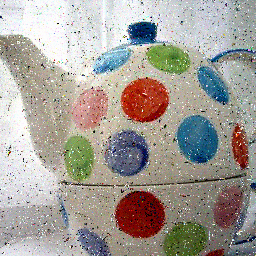} &
    \includegraphics[width=0.13\textwidth]{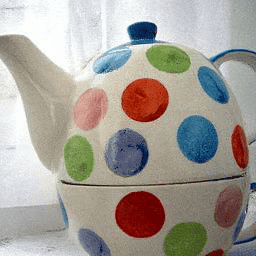} \\
    
    \includegraphics[width=0.13\textwidth]{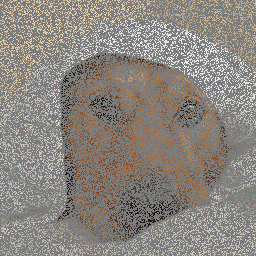} &
    \includegraphics[width=0.13\textwidth]{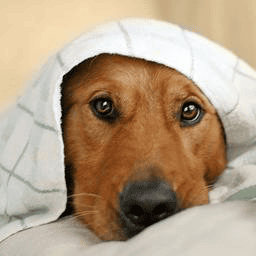} &
    \includegraphics[width=0.13\textwidth]{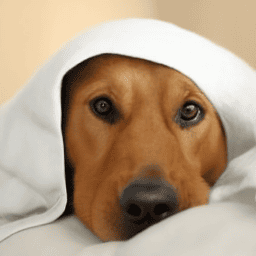} &
    \includegraphics[width=0.13\textwidth]{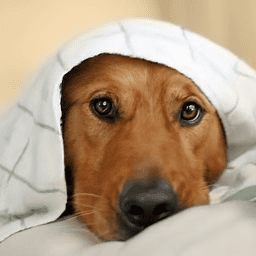} &
    \includegraphics[width=0.13\textwidth]{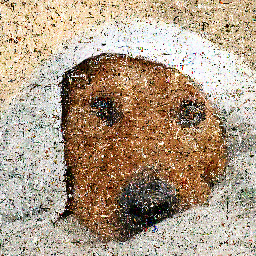} &
    \includegraphics[width=0.13\textwidth]{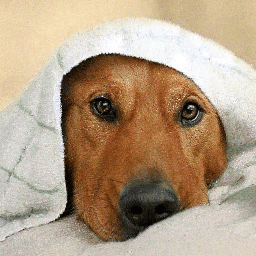} \\
    \end{tabular}
    \caption{Visual comparison on Inpainting Random task. Each row shows results for one image across different methods. The three sections correspond to timesteps 20, 100, and 1000 respectively.}
    \label{fig:inpainting_random_results}
\end{figure}
\clearpage
% ---------------- INPAINTING BOX RESULTS ----------------
\begin{figure}[p]
    \centering
    \begin{tabular}{cccccc}
    \multicolumn{6}{c}{\textbf{Inpainting Box - Timestep 20}} \\
    Input & Label & DPS & NCS-DPS & MPGD & NCS-MPGD \\
    \includegraphics[width=0.13\textwidth]{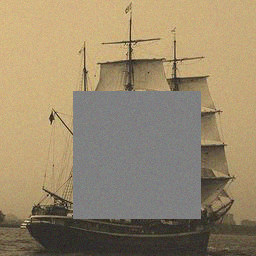} &
    \includegraphics[width=0.13\textwidth]{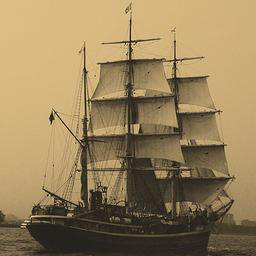} &
    \includegraphics[width=0.13\textwidth]{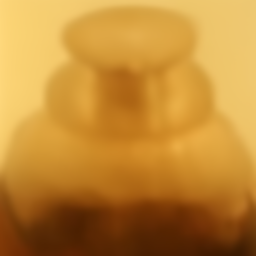} &
    \includegraphics[width=0.13\textwidth]{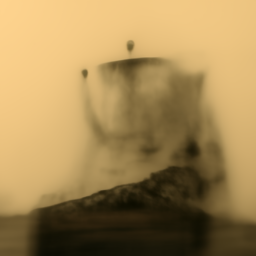} &
    \includegraphics[width=0.13\textwidth]{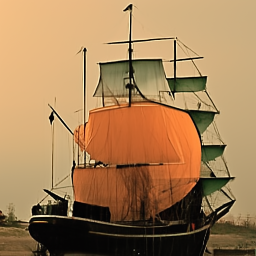} &
    \includegraphics[width=0.13\textwidth]{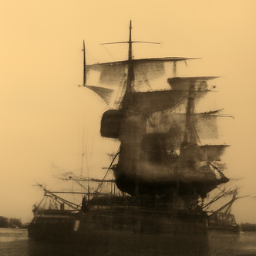} \\
    
    \includegraphics[width=0.13\textwidth]{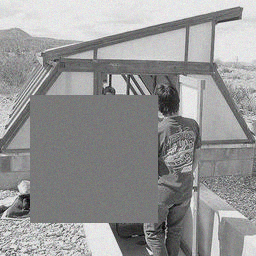} &
    \includegraphics[width=0.13\textwidth]{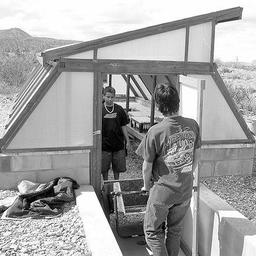} &
    \includegraphics[width=0.13\textwidth]{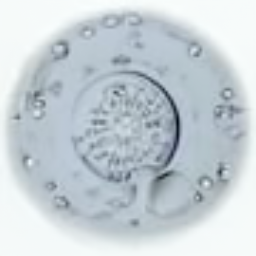} &
    \includegraphics[width=0.13\textwidth]{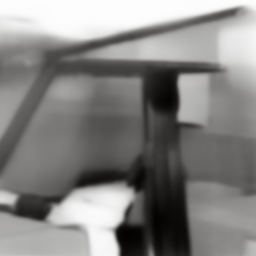} &
    \includegraphics[width=0.13\textwidth]{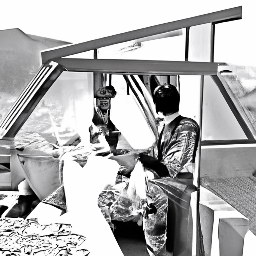} &
    \includegraphics[width=0.13\textwidth]{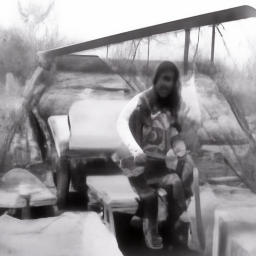} \\
    
    \includegraphics[width=0.13\textwidth]{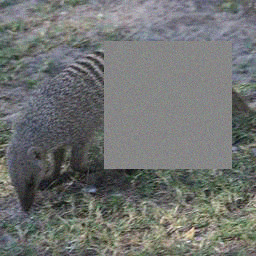} &
    \includegraphics[width=0.13\textwidth]{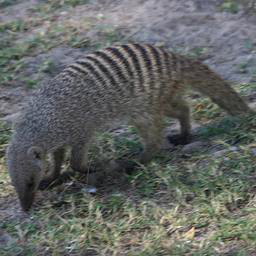} &
    \includegraphics[width=0.13\textwidth]{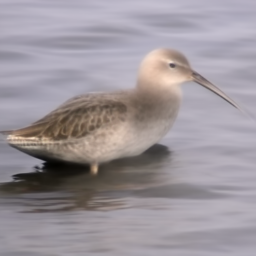} &
    \includegraphics[width=0.13\textwidth]{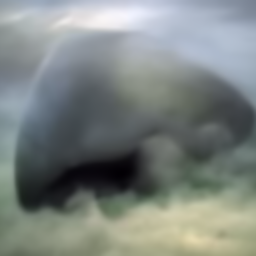} &
    \includegraphics[width=0.13\textwidth]{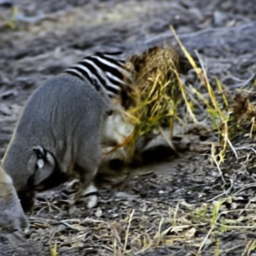} &
    \includegraphics[width=0.13\textwidth]{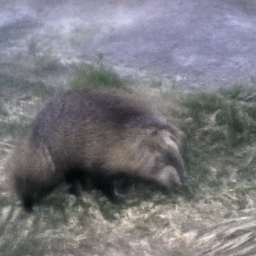} \\
    
    \multicolumn{6}{c}{\textbf{Inpainting Box - Timestep 100}} \\
    \includegraphics[width=0.13\textwidth]{fig/results_selected/inpainting_box/input/00012.png} &
    \includegraphics[width=0.13\textwidth]{fig/results_selected/inpainting_box/label/00012.png} &
    \includegraphics[width=0.13\textwidth]{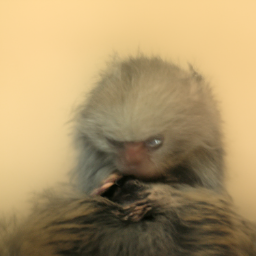} &
    \includegraphics[width=0.13\textwidth]{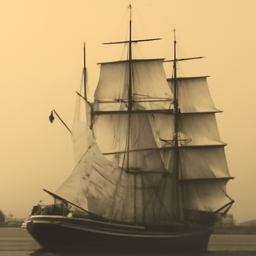} &
    \includegraphics[width=0.13\textwidth]{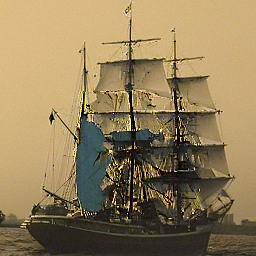} &
    \includegraphics[width=0.13\textwidth]{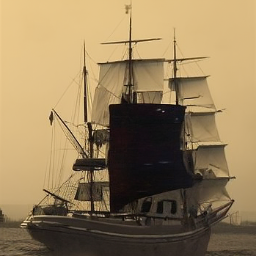} \\
    
    \includegraphics[width=0.13\textwidth]{fig/results_selected/inpainting_box/input/00017.png} &
    \includegraphics[width=0.13\textwidth]{fig/results_selected/inpainting_box/label/00017.png} &
    \includegraphics[width=0.13\textwidth]{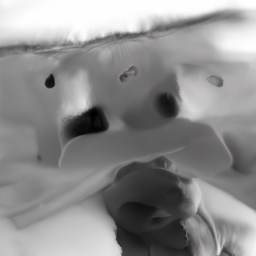} &
    \includegraphics[width=0.13\textwidth]{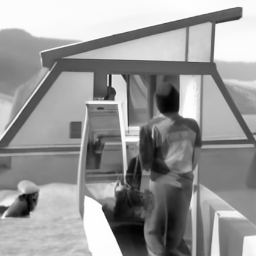} &
    \includegraphics[width=0.13\textwidth]{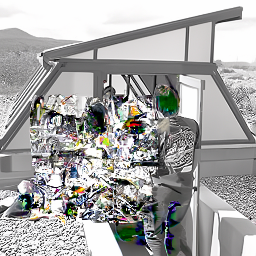} &
    \includegraphics[width=0.13\textwidth]{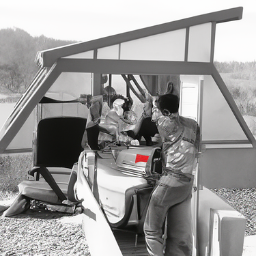} \\
    
    \includegraphics[width=0.13\textwidth]{fig/results_selected/inpainting_box/input/00049.png} &
    \includegraphics[width=0.13\textwidth]{fig/results_selected/inpainting_box/label/00049.png} &
    \includegraphics[width=0.13\textwidth]{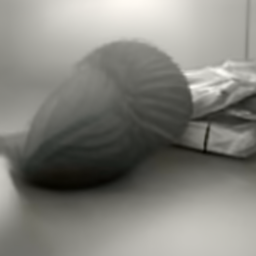} &
    \includegraphics[width=0.13\textwidth]{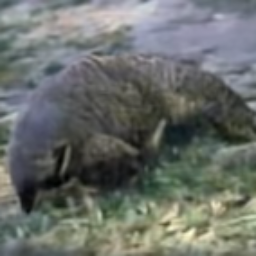} &
    \includegraphics[width=0.13\textwidth]{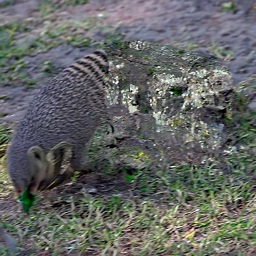} &
    \includegraphics[width=0.13\textwidth]{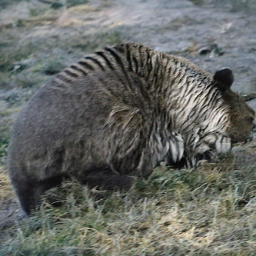} \\
    
    \multicolumn{6}{c}{\textbf{Inpainting Box - Timestep 1000}} \\
    \includegraphics[width=0.13\textwidth]{fig/results_selected/inpainting_box/input/00012.png} &
    \includegraphics[width=0.13\textwidth]{fig/results_selected/inpainting_box/label/00012.png} &
    \includegraphics[width=0.13\textwidth]{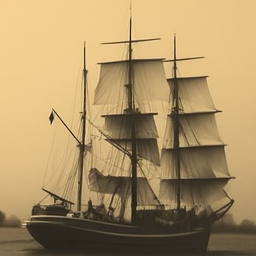} &
    \includegraphics[width=0.13\textwidth]{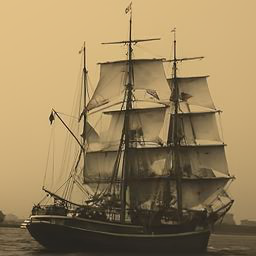} &
    \includegraphics[width=0.13\textwidth]{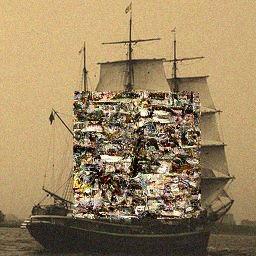} &
    \includegraphics[width=0.13\textwidth]{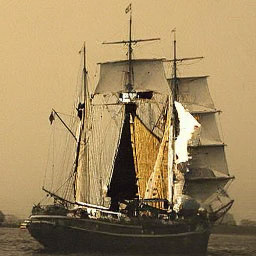} \\
    
    \includegraphics[width=0.13\textwidth]{fig/results_selected/inpainting_box/input/00017.png} &
    \includegraphics[width=0.13\textwidth]{fig/results_selected/inpainting_box/label/00017.png} &
    \includegraphics[width=0.13\textwidth]{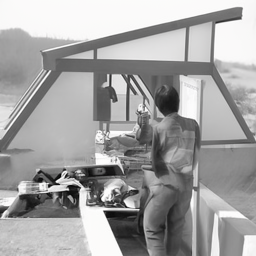} &
    \includegraphics[width=0.13\textwidth]{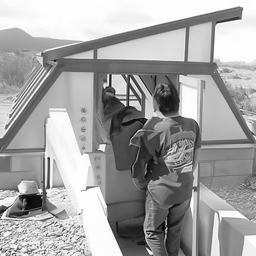} &
    \includegraphics[width=0.13\textwidth]{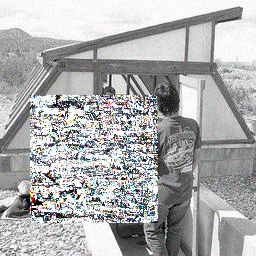} &
    \includegraphics[width=0.13\textwidth]{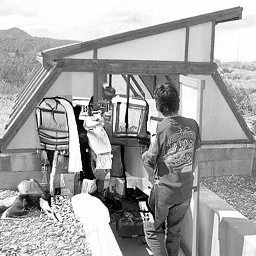} \\
    
    \includegraphics[width=0.13\textwidth]{fig/results_selected/inpainting_box/input/00049.png} &
    \includegraphics[width=0.13\textwidth]{fig/results_selected/inpainting_box/label/00049.png} &
    \includegraphics[width=0.13\textwidth]{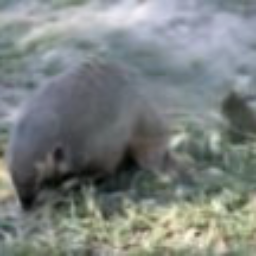} &
    \includegraphics[width=0.13\textwidth]{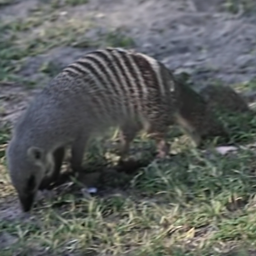} &
    \includegraphics[width=0.13\textwidth]{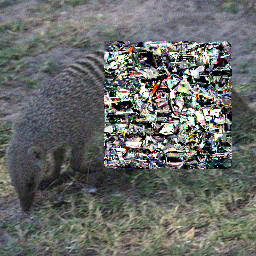} &
    \includegraphics[width=0.13\textwidth]{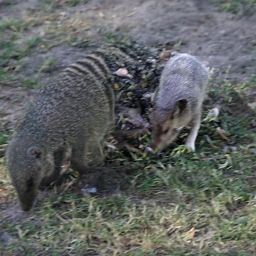} \\
    \end{tabular}
    \caption{Visual comparison on Inpainting Box task. Each row shows results for one image across different methods. The three sections correspond to timesteps 20, 100, and 1000 respectively.}
    \label{fig:inpainting_box_results}
\end{figure}
\clearpage
% ---------------- MOTION BLUR RESULTS ----------------
\begin{figure}[p]
    \centering
    \begin{tabular}{cccccc}
    \multicolumn{6}{c}{\textbf{Motion Blur - Timestep 20}} \\
    Input & Label & DPS & NCS-DPS & MPGD & NCS-MPGD \\
    \includegraphics[width=0.13\textwidth]{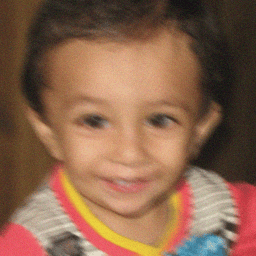} &
    \includegraphics[width=0.13\textwidth]{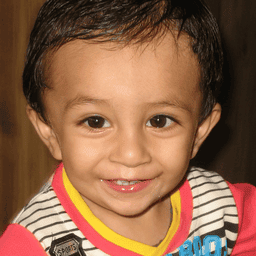} &
    \includegraphics[width=0.13\textwidth]{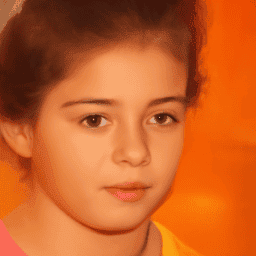} &
    \includegraphics[width=0.13\textwidth]{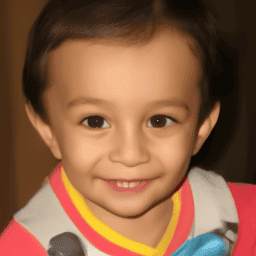} &
    \includegraphics[width=0.13\textwidth]{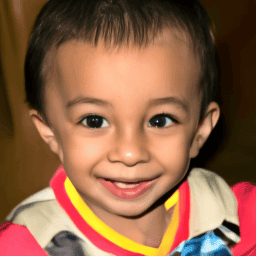} &
    \includegraphics[width=0.13\textwidth]{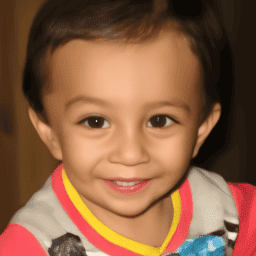} \\
    
    \includegraphics[width=0.13\textwidth]{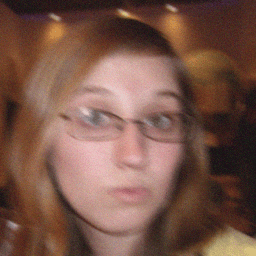} &
    \includegraphics[width=0.13\textwidth]{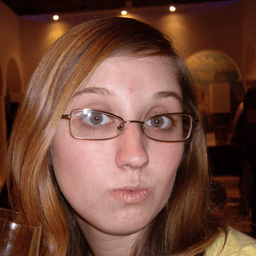} &
    \includegraphics[width=0.13\textwidth]{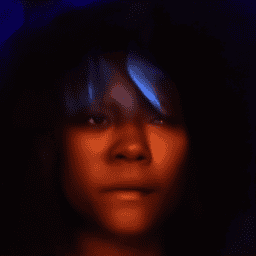} &
    \includegraphics[width=0.13\textwidth]{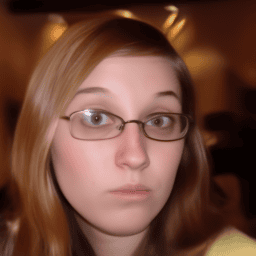} &
    \includegraphics[width=0.13\textwidth]{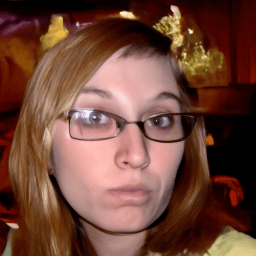} &
    \includegraphics[width=0.13\textwidth]{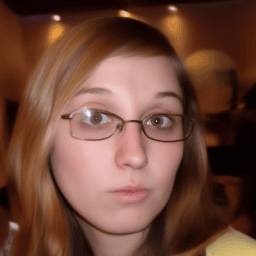} \\
    
    \includegraphics[width=0.13\textwidth]{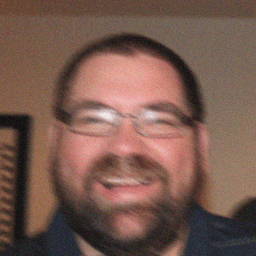} &
    \includegraphics[width=0.13\textwidth]{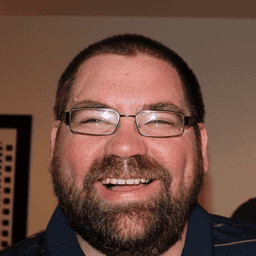} &
    \includegraphics[width=0.13\textwidth]{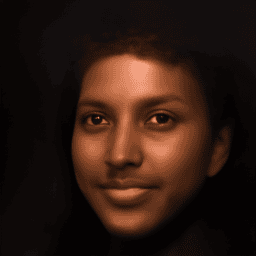} &
    \includegraphics[width=0.13\textwidth]{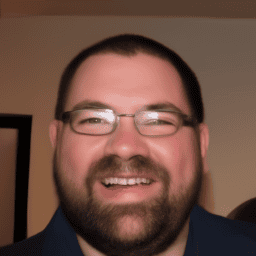} &
    \includegraphics[width=0.13\textwidth]{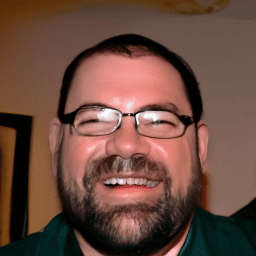} &
    \includegraphics[width=0.13\textwidth]{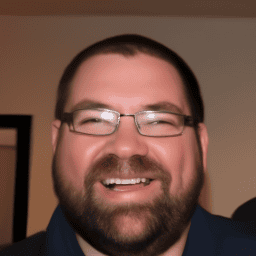} \\
    
    \multicolumn{6}{c}{\textbf{Motion Blur - Timestep 100}} \\
    \includegraphics[width=0.13\textwidth]{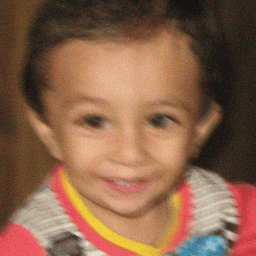} &
    \includegraphics[width=0.13\textwidth]{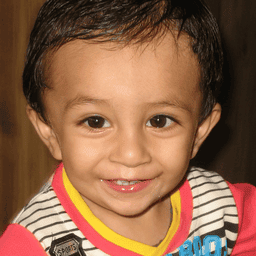} &
    \includegraphics[width=0.13\textwidth]{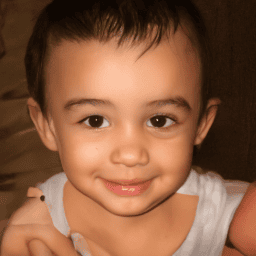} &
    \includegraphics[width=0.13\textwidth]{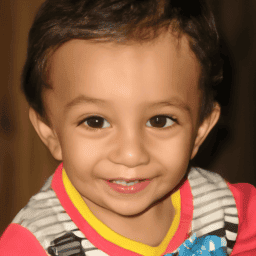} &
    \includegraphics[width=0.13\textwidth]{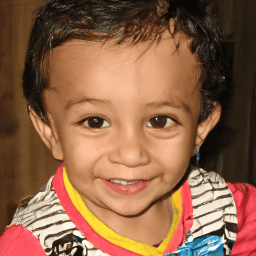} &
    \includegraphics[width=0.13\textwidth]{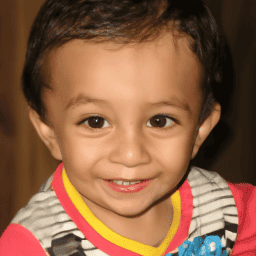} \\
    
    \includegraphics[width=0.13\textwidth]{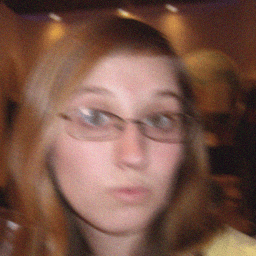} &
    \includegraphics[width=0.13\textwidth]{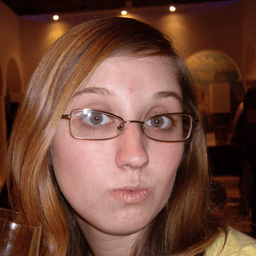} &
    \includegraphics[width=0.13\textwidth]{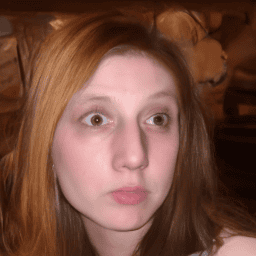} &
    \includegraphics[width=0.13\textwidth]{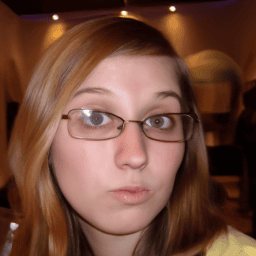} &
    \includegraphics[width=0.13\textwidth]{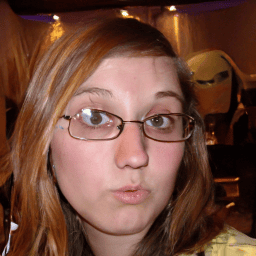} &
    \includegraphics[width=0.13\textwidth]{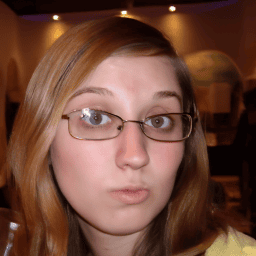} \\
    
    \includegraphics[width=0.13\textwidth]{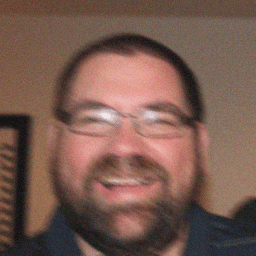} &
    \includegraphics[width=0.13\textwidth]{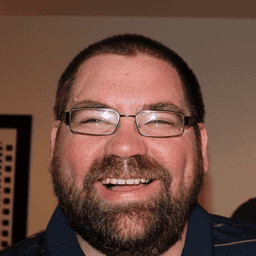} &
    \includegraphics[width=0.13\textwidth]{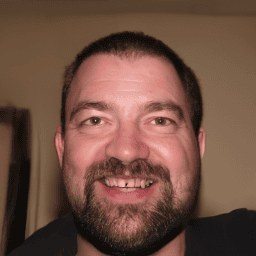} &
    \includegraphics[width=0.13\textwidth]{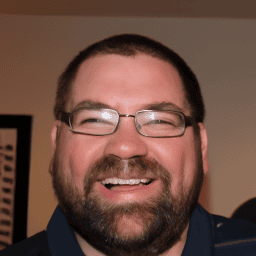} &
    \includegraphics[width=0.13\textwidth]{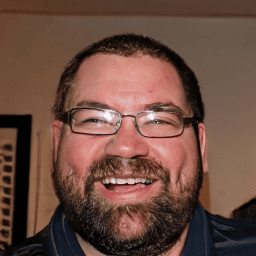} &
    \includegraphics[width=0.13\textwidth]{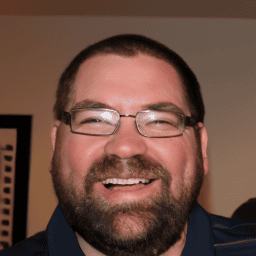} \\
    
    \multicolumn{6}{c}{\textbf{Motion Blur - Timestep 1000}} \\
    \includegraphics[width=0.13\textwidth]{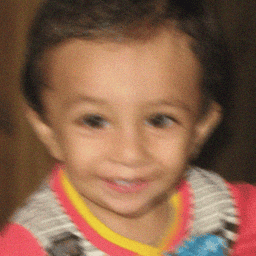} &
    \includegraphics[width=0.13\textwidth]{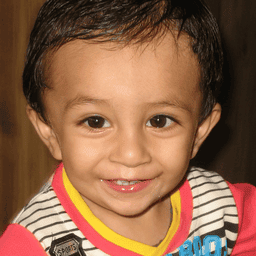} &
    \includegraphics[width=0.13\textwidth]{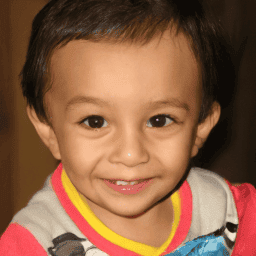} &
    \includegraphics[width=0.13\textwidth]{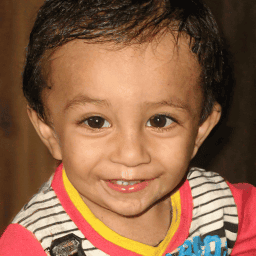} &
    \includegraphics[width=0.13\textwidth]{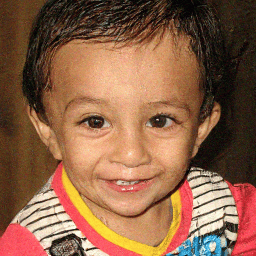} &
    \includegraphics[width=0.13\textwidth]{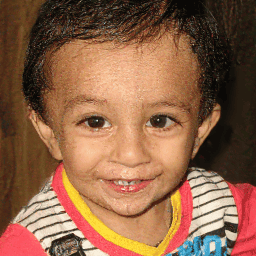} \\
    
    \includegraphics[width=0.13\textwidth]{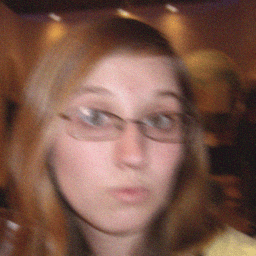} &
    \includegraphics[width=0.13\textwidth]{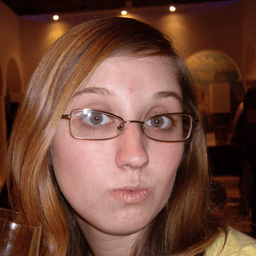} &
    \includegraphics[width=0.13\textwidth]{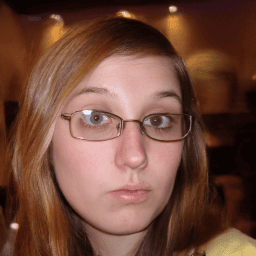} &
    \includegraphics[width=0.13\textwidth]{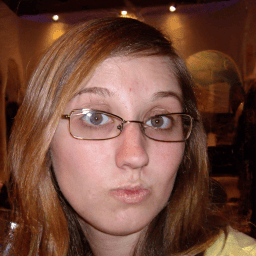} &
    \includegraphics[width=0.13\textwidth]{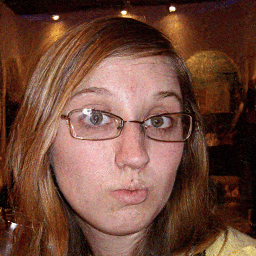} &
    \includegraphics[width=0.13\textwidth]{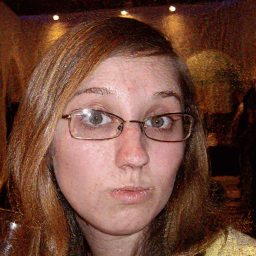} \\
    
    \includegraphics[width=0.13\textwidth]{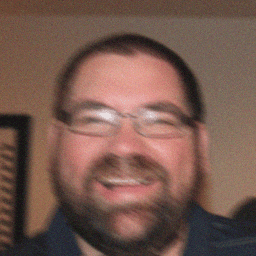} &
    \includegraphics[width=0.13\textwidth]{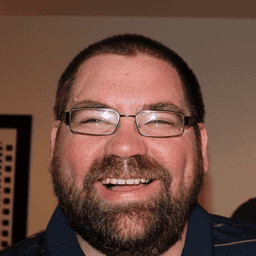} &
    \includegraphics[width=0.13\textwidth]{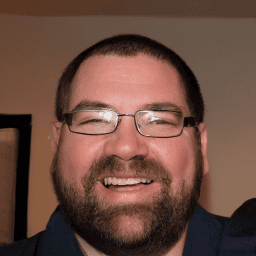} &
    \includegraphics[width=0.13\textwidth]{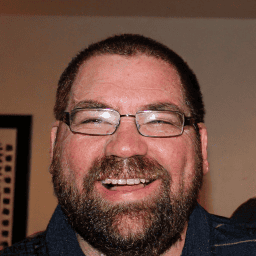} &
    \includegraphics[width=0.13\textwidth]{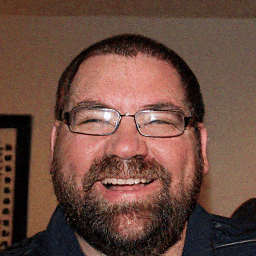} &
    \includegraphics[width=0.13\textwidth]{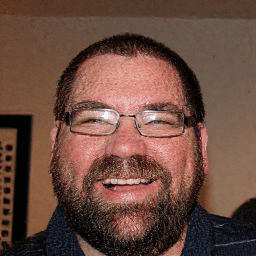} \\
    \end{tabular}
    \caption{Visual comparison on Motion Blur task. Each row shows results for one image across different methods. The three sections correspond to timesteps 20, 100, and 1000 respectively.}
    \label{fig:motion_blur_results}
\end{figure}
\clearpage
% ---------------- SUPER RESOLUTION X4 RESULTS ----------------
\begin{figure}[p]
    \centering
    \begin{tabular}{cccccc}
    \multicolumn{6}{c}{\textbf{Super Resolution x4 - Timestep 20}} \\
    Input & Label & DPS & NCS-DPS & MPGD & NCS-MPGD \\
    \includegraphics[width=0.13\textwidth]{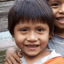} &
    \includegraphics[width=0.13\textwidth]{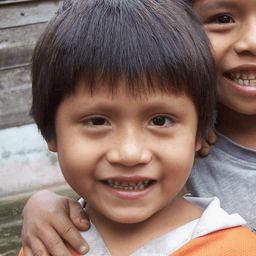} &
    \includegraphics[width=0.13\textwidth]{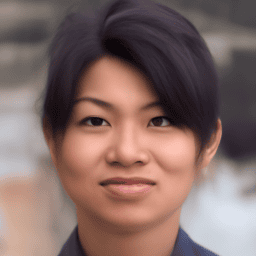} &
    \includegraphics[width=0.13\textwidth]{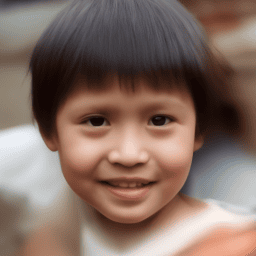} &
    \includegraphics[width=0.13\textwidth]{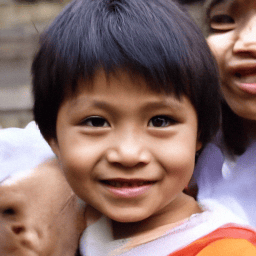} &
    \includegraphics[width=0.13\textwidth]{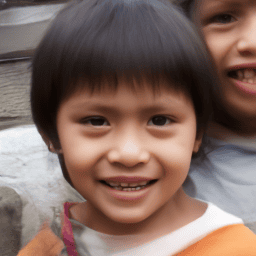} \\
    
    \includegraphics[width=0.13\textwidth]{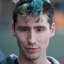} &
    \includegraphics[width=0.13\textwidth]{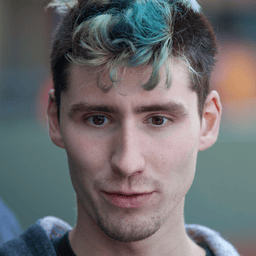} &
    \includegraphics[width=0.13\textwidth]{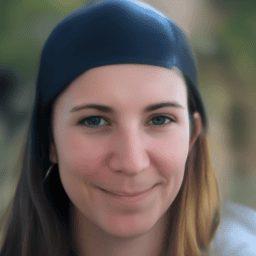} &
    \includegraphics[width=0.13\textwidth]{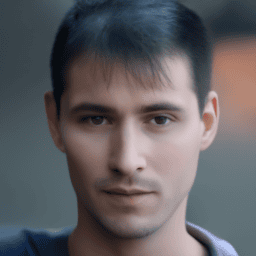} &
    \includegraphics[width=0.13\textwidth]{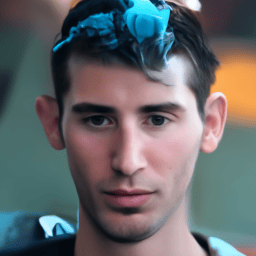} &
    \includegraphics[width=0.13\textwidth]{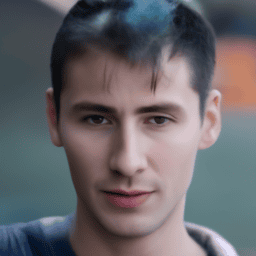} \\
    
    \includegraphics[width=0.13\textwidth]{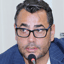} &
    \includegraphics[width=0.13\textwidth]{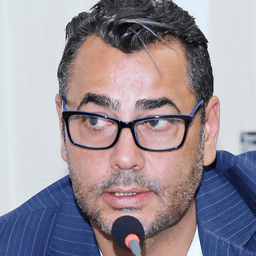} &
    \includegraphics[width=0.13\textwidth]{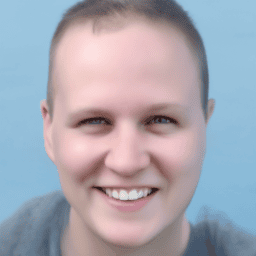} &
    \includegraphics[width=0.13\textwidth]{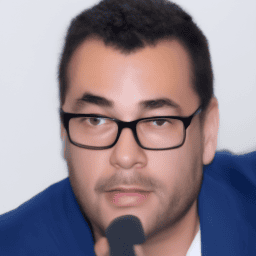} &
    \includegraphics[width=0.13\textwidth]{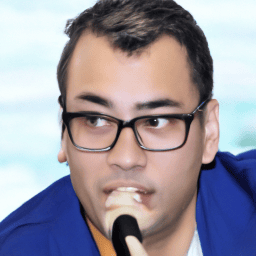} &
    \includegraphics[width=0.13\textwidth]{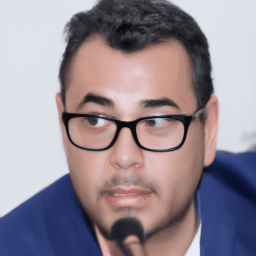} \\
    
    \multicolumn{6}{c}{\textbf{Super Resolution x4 - Timestep 100}} \\
    \includegraphics[width=0.13\textwidth]{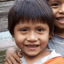} &
    \includegraphics[width=0.13\textwidth]{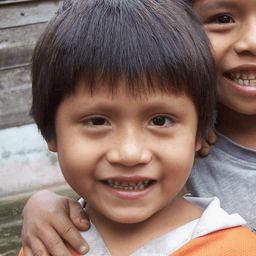} &
    \includegraphics[width=0.13\textwidth]{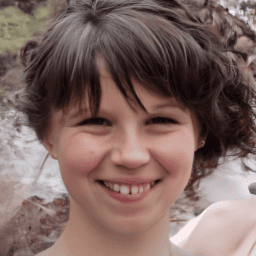} &
    \includegraphics[width=0.13\textwidth]{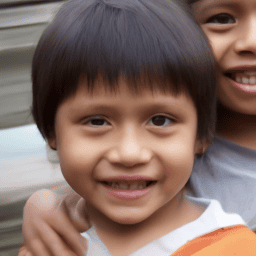} &
    \includegraphics[width=0.13\textwidth]{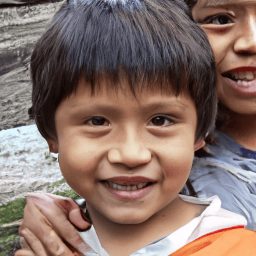} &
    \includegraphics[width=0.13\textwidth]{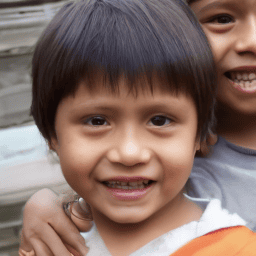} \\
    
    \includegraphics[width=0.13\textwidth]{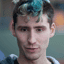} &
    \includegraphics[width=0.13\textwidth]{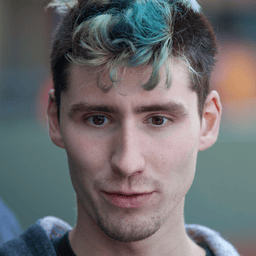} &
    \includegraphics[width=0.13\textwidth]{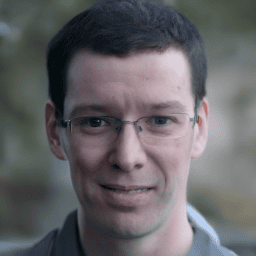} &
    \includegraphics[width=0.13\textwidth]{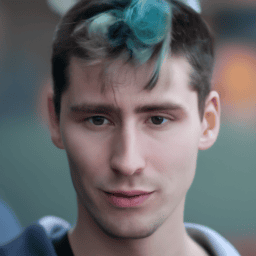} &
    \includegraphics[width=0.13\textwidth]{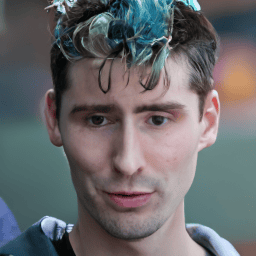} &
    \includegraphics[width=0.13\textwidth]{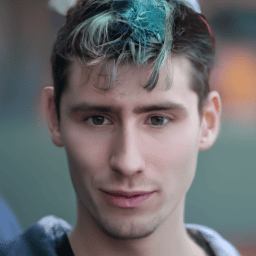} \\
    
    \includegraphics[width=0.13\textwidth]{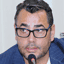} &
    \includegraphics[width=0.13\textwidth]{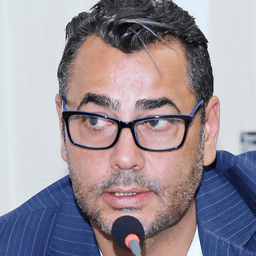} &
    \includegraphics[width=0.13\textwidth]{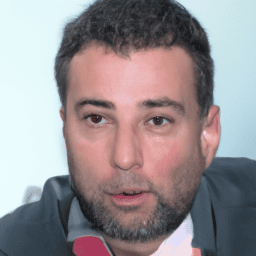} &
    \includegraphics[width=0.13\textwidth]{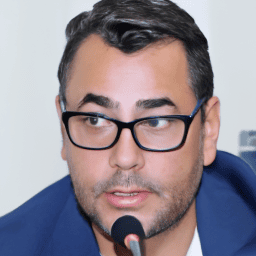} &
    \includegraphics[width=0.13\textwidth]{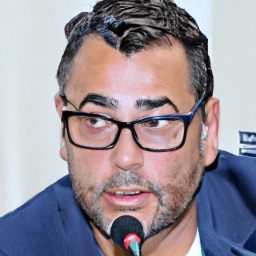} &
    \includegraphics[width=0.13\textwidth]{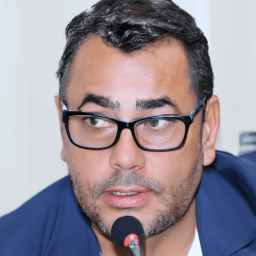} \\
    
    \multicolumn{6}{c}{\textbf{Super Resolution x4 - Timestep 1000}} \\
    \includegraphics[width=0.13\textwidth]{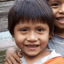} &
    \includegraphics[width=0.13\textwidth]{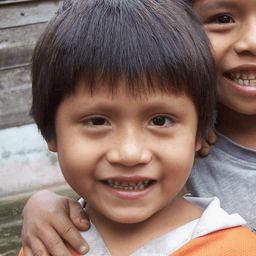} &
    \includegraphics[width=0.13\textwidth]{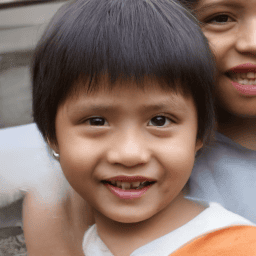} &
    \includegraphics[width=0.13\textwidth]{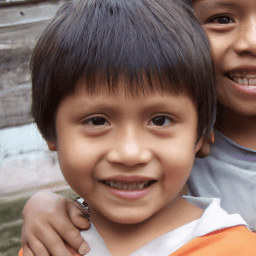} &
    \includegraphics[width=0.13\textwidth]{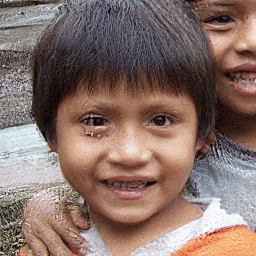} &
    \includegraphics[width=0.13\textwidth]{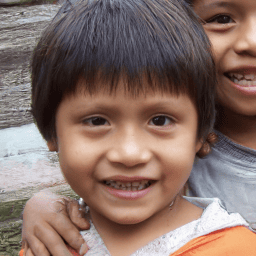} \\
    
    \includegraphics[width=0.13\textwidth]{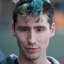} &
    \includegraphics[width=0.13\textwidth]{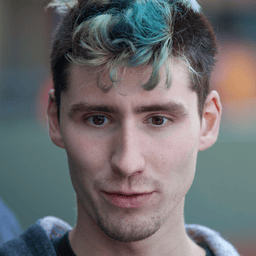} &
    \includegraphics[width=0.13\textwidth]{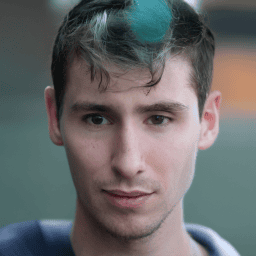} &
    \includegraphics[width=0.13\textwidth]{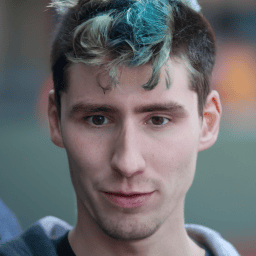} &
    \includegraphics[width=0.13\textwidth]{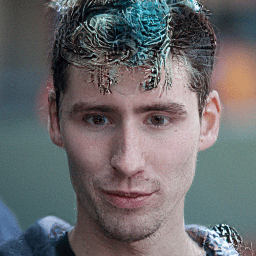} &
    \includegraphics[width=0.13\textwidth]{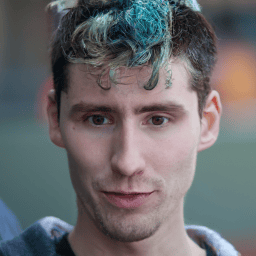} \\
    
    \includegraphics[width=0.13\textwidth]{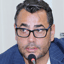} &
    \includegraphics[width=0.13\textwidth]{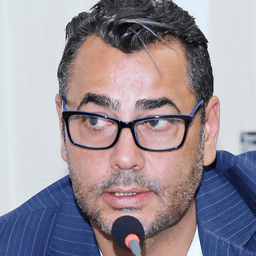} &
    \includegraphics[width=0.13\textwidth]{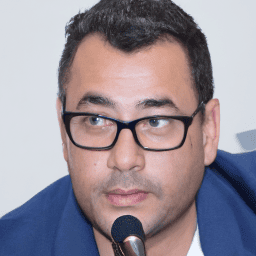} &
    \includegraphics[width=0.13\textwidth]{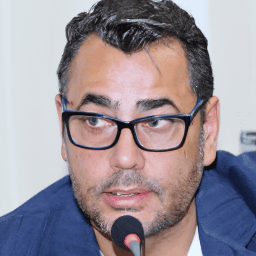} &
    \includegraphics[width=0.13\textwidth]{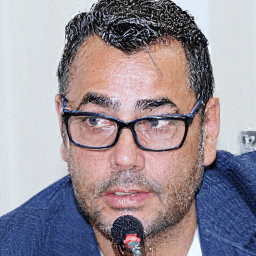} &
    \includegraphics[width=0.13\textwidth]{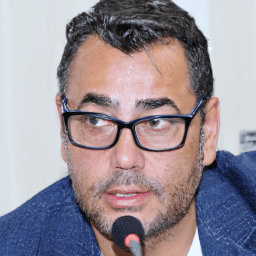} \\
    \end{tabular}
    \caption{Visual comparison on Super Resolution x4 task. Each row shows results for one image across different methods. The three sections correspond to timesteps 20, 100, and 1000 respectively.}
    \label{fig:super_resolution_x4_results}
\end{figure}

\clearpage

% ===================== Compression Results =====================

\section{Definitions of NCS-DPS, NCS-MPGD, and NCS-$\Pi$GDM}
\label{app:ncs_definitions}
\begin{definition}[NCS-DPS]
    \citet{chung2023diffusion} approximates the conditional score using the gradient of a likelihood loss defined between the observation $\vec{y}$ and the estimated signal $\vec{\tilde{x}}_{0|t}$. Specifically,
    \begin{align}
        \nabla_{\vec{x}_t} \log p(\vec{y} \mid \vec{x}_t) 
        &= \nabla_{\vec{x}_t} \log \mathcal{N}(\vec{y}; \mat{A}\vec{\tilde{x}}_{0|t}, \sigma_t^2\mathbf{I}) \nonumber \\
        &= \frac{1}{\sigma_t^2} \left(\frac{\partial \vec{\tilde{x}}_{0|t}}{\partial \vec{x}_t}\right)^\top \mat{A}^\top \left(\vec{y} - \mat{A}\vec{\tilde{x}}_{0|t}\right).
    \end{align}
    Its NCS counterpart is obtained by aligning the synthesized noise vector with this gradient direction:
    \begin{equation}
        \vec{\gamma}^* = \operatorname*{argmax}_{\vec{\gamma} \in \mathbb{R}^K,\ \|\vec{\gamma}\|_2 = 1} 
        \left\langle 
            \left(\frac{\partial \vec{\tilde{x}}_{0|t}}{\partial \vec{x}_t}\right)^\top \mat{A}^\top \left(\vec{y} - \mat{A} \vec{\tilde{x}}_{0|t}\right),\ 
            \sum_{i=1}^K \gamma_i \vec{\epsilon}^i_{t} 
        \right\rangle.
    \end{equation}
\end{definition}

\begin{definition}[NCS-MPGD]
    \citet{he2024manifold} proposes performing updates directly on the estimated signal $\vec{\tilde{x}}_{0|t}$ rather than the latent variable $\vec{x}_t$. By isolating the additional term introduced in its update rule, we obtain the following approximation of the measurement score:
    \begin{align}
        \nabla_{\vec{x}_t} \log p(\vec{y} \mid \vec{x}_t) 
        &= -\lambda_t \sqrt{\bar{\alpha}_t} \nabla_{\vec{\tilde{x}}_{0|t}} \|\vec{y} - \mat{A} \vec{\tilde{x}}_{0|t}\|_2^2 \nonumber \\
        &= 2\lambda_t \sqrt{\bar{\alpha}_t} \mat{A}^\top (\vec{y} - \mat{A} \vec{\tilde{x}}_{0|t}),
    \end{align}
    where $\lambda_t$ is a time-dependent step size. The corresponding NCS formulation is given by:
    \begin{equation}
        \label{eq:ncs_mpgd}
        \vec{\gamma}^* = \operatorname*{argmax}_{\vec{\gamma} \in \mathbb{R}^K,\ \|\vec{\gamma}\|_2 = 1} 
        \left\langle  
            \mat{A}^\top (\vec{y} - \mat{A} \vec{\tilde{x}}_{0|t}),\ 
            \sum_{i=1}^K \gamma_i \vec{\epsilon}^i_{t} 
        \right\rangle.
    \end{equation}
    \end{definition}

Notably, if $\vec{\tilde{x}}_{0|t}$ is replaced with $\vec{x}_t$, the formulation reduces to Score-Based Annealed Langevin Dynamics (ALD)~\citep{jalal2021robust}. For a comprehensive analysis of the connections between these methods, refer to the survey in~\citep{daras2024survey}.

\begin{definition}[NCS-$\Pi$GDM]
    \citet{song2023pseudoinverse} introduces Pseudoinverse-guided DMs ($\Pi$GDM), a problem-agnostic approach that directly estimates conditional scores from the measurement model without additional training. $\Pi$GDM can address inverse problems with noisy, non-linear, or even non-differentiable measurements. The method enforces data consistency by applying the range--null space rectification:
    \begin{equation}
        \hat{\vec{x}}_{0|t}
        =
        \mat{A}^{\dagger}\vec{y}
        +(\mat{I}-\mat{A}^{\dagger}\mat{A})\vec{\tilde{x}}_{0|t}
        =
        \vec{\tilde{x}}_{0|t}+\mat{A}^{\dagger}\!\left(\vec{y}-\mat{A}\vec{\tilde{x}}_{0|t}\right),
    \end{equation}
    where $\mat{A}^{\dagger}$ denotes the pseudo-inverse of $\mat{A}$.
    The induced range-space correction term is therefore
    $\Delta_t := \mat{A}^{\dagger}(\vec{y}-\mat{A}\vec{\tilde{x}}_{0|t})$.
    The corresponding NCS formulation is given by:
    \begin{equation}
        \label{eq:ncs_pigdm}
        \vec{\gamma}^* = \operatorname*{argmax}_{\vec{\gamma} \in \mathbb{R}^K,\ \|\vec{\gamma}\|_2 = 1} 
        \left\langle 
            \mat{A}^{\dagger}\!\left(\vec{y} - \mat{A} \vec{\tilde{x}}_{0|t}\right),
            \sum_{i=1}^K \gamma_i \vec{\epsilon}^i_{t} 
        \right\rangle .
    \end{equation}
\end{definition}

\begin{definition}[NCS-DAPS]
    \citet{zhang2025improving} proposes Decoupled Annealing Posterior Sampling (DAPS), which decouples the reverse diffusion and data consistency steps. At each annealing step with noise level $\sigma_t$, DAPS performs: (1) reverse diffusion from $\vec{x}_t$ to obtain an estimate $\vec{\tilde{x}}_{0|t}$, (2) MCMC sampling to enforce data consistency and obtain $\vec{x}_{0|t}^{y}$, and (3) forward diffusion by adding noise:
    \begin{equation}
        \vec{x}_{t+1} = \vec{x}_{0|t}^{y} + \sigma_{t+1} \vec{\epsilon}.
    \end{equation}
    The corresponding NCS formulation replaces the forward diffusion noise with a synthesized noise that aligns with the overall update direction from $\vec{x}_t$ to $\vec{x}_{0|t}^{y}$:
    \begin{equation}
        \label{eq:ncs_daps}
        \vec{\gamma}^* = \operatorname*{argmax}_{\vec{\gamma} \in \mathbb{R}^K,\ \|\vec{\gamma}\|_2 = 1} 
        \left\langle 
            \vec{x}_{0|t}^{y} - \vec{x}_t,
            \sum_{i=1}^K \gamma_i \vec{\epsilon}^i_{t} 
        \right\rangle.
    \end{equation}
\end{definition}

\section{Proof of Gaussianity and Optimal Weights}
\label{app:ncs_gaussianity}

We first show that, under a unit-norm constraint on the combination weights, the synthesized noise remains standard normal, and then derive the closed-form optimizer for~\eqref{eq:ncs_optimization}.

\begin{lemma}[Gaussianity of unit-norm combinations]\label{lem:ncs_gaussianity}
Let $\{\vec{\epsilon}_t^i\}_{i=1}^K$ be mutually independent with $\vec{\epsilon}_t^i \sim \mathcal{N}(\vec{0}, \mathbf{I})$. For any $\vec{\gamma} = (\gamma_1,\ldots,\gamma_K) \in \mathbb{R}^K$ with $\|\vec{\gamma}\|_2 = 1$ that is deterministic (or independent of $\{\vec{\epsilon}_t^i\}_{i=1}^K$), the linear combination
\[
\vec{\epsilon}_t^* \;=\; \sum_{i=1}^K \gamma_i \vec{\epsilon}_t^i
\]
satisfies $\vec{\epsilon}_t^* \sim \mathcal{N}(\vec{0}, \mathbf{I})$.
\end{lemma}

\begin{proof}
By linearity of expectation,
$\mathbb{E}[\vec{\epsilon}_t^*] = \sum_{i=1}^K \gamma_i \,\mathbb{E}[\vec{\epsilon}_t^i] = \vec{0}$.
For the covariance, independence and isotropy give
\[
\operatorname{Cov}(\vec{\epsilon}_t^*)
= \sum_{i=1}^K \sum_{j=1}^K \gamma_i \gamma_j \,\mathbb{E}\!\big[\vec{\epsilon}_t^i (\vec{\epsilon}_t^j)^\top\big]
= \sum_{i=1}^K \gamma_i^2 \,\mathbf{I}
= \|\vec{\gamma}\|_2^2 \,\mathbf{I}
= \mathbf{I}.
\]
Since $(\vec{\epsilon}_t^1,\ldots,\vec{\epsilon}_t^K)$ is jointly Gaussian and $\vec{\epsilon}_t^*$ is a linear transformation of it, $\vec{\epsilon}_t^*$ is Gaussian with mean $\vec{0}$ and covariance $\mathbf{I}$; hence $\vec{\epsilon}_t^* \sim \mathcal{N}(\vec{0}, \mathbf{I})$.
\end{proof}

\paragraph{Conclusion.}
Combining Lemma~\ref{lem:ncs_gaussianity} with Theorem~\ref{thm:ncs_optimal}, we obtain that the synthesized noise is standard normal whenever the synthesis codebook is independent of the weight computation, i.e., $\vec{\epsilon}_t^* \sim \mathcal{N}(\vec{0}, \mathbf{I})$.

In specific optimization scenarios, as the number of noise sources increases linearly, the inner product also increases nearly linearly according to log(K). This approach achieves significantly higher efficiency than selecting a single noise source within K (in DDCM). Furthermore, the magnitude of the optimal noise remains constant over a considerable range (approximately equal to the average magnitude of the noise sources). Specific variations can be observed in Fig.~\ref{fig:inner_product_and_norm_vs_k}.

\section{NCS Definition}
\label{app:optimal}
\subsection{Proof of Theorem~\ref{thm:ncs_optimal}}

\begin{proof}
Let $\vec{c} = \nabla_{\vec{x}_t} \log p(\vec{y} \mid \vec{x}_t) \in \mathbb{R}^d$ be the approximate measurement score, and let
$\mat{E}_t = [\vec{\epsilon}_t^1, \ldots, \vec{\epsilon}_t^K] \in \mathbb{R}^{d \times K}$ stack $K$ standard Gaussian noise vectors as columns.
For any $\vec{\gamma} \in \mathbb{R}^K$ with $\|\vec{\gamma}\|_2 = 1$, we have
\[
\langle \vec{c}, \mat{E}_t \vec{\gamma} \rangle
= \vec{c}^\top \mat{E}_t \vec{\gamma}
= (\mat{E}_t^\top \vec{c})^\top \vec{\gamma}.
\]
Define $\vec{s} := \mat{E}_t^\top \vec{c} \in \mathbb{R}^K$. The optimization problem in Theorem~\ref{thm:ncs} becomes
\[
\vec{\gamma}^*
= \arg\max_{\|\vec{\gamma}\|_2 = 1} \langle \vec{s}, \vec{\gamma} \rangle.
\]
By the Cauchy--Schwarz inequality,
\[
\langle \vec{s}, \vec{\gamma} \rangle
\le \|\vec{s}\|_2 \, \|\vec{\gamma}\|_2
= \|\vec{s}\|_2,
\]
with equality if and only if $\vec{\gamma}$ is aligned with $\vec{s}$, i.e., $\vec{\gamma} = \lambda \vec{s}$ for some scalar $\lambda$.
Enforcing $\|\vec{\gamma}\|_2 = 1$ yields $|\lambda| = 1/\|\vec{s}\|_2$, and hence
\[
\vec{\gamma}^*
= \frac{\vec{s}}{\|\vec{s}\|_2}
= \frac{\mat{E}_t^\top \vec{c}}{\|\mat{E}_t^\top \vec{c}\|_2}.
\]
\end{proof}

\subsection{Cosine-based Objective and Relation to Theorem~\ref{thm:ncs_optimal}}

Theorem~\ref{thm:ncs} and Theorem~\ref{thm:ncs_optimal} formulate Noise Combination Sampling (NCS) by maximizing the inner product
\[
\langle \vec{c}, \mat{E}_t \vec{\gamma} \rangle
\quad\text{subject to}\quad \|\vec{\gamma}\|_2 = 1,
\]
which leads to the closed-form solution
\(\vec{\gamma}^* \propto \mat{E}_t^\top \vec{c}\) via Cauchy--Schwarz (Theorem~\ref{thm:ncs_optimal}).

A natural alternative is to maximize the cosine similarity between the synthesized noise and the measurement score direction:
\begin{equation}
    \label{eq:ncs_cosine_objective}
    \vec{\gamma}^{\star}
    \;=\;
    \operatorname*{argmax}_{\|\vec{\gamma}\|_2 = 1}
    \cos\!\bigl(\vec{c}, \mat{E}_t \vec{\gamma}\bigr)
    \;=\;
    \operatorname*{argmax}_{\|\vec{\gamma}\|_2 = 1}
    \frac{\vec{c}^\top \mat{E}_t \vec{\gamma}}{\|\vec{c}\|_2\,\|\mat{E}_t \vec{\gamma}\|_2}.
\end{equation}
Since the sign of $\vec{\gamma}$ can always be flipped, maximizing the cosine is equivalent to maximizing its square:
\[
\max_{\|\vec{\gamma}\|_2 = 1}
\cos^2\!\bigl(\vec{c}, \mat{E}_t \vec{\gamma}\bigr)
=
\max_{\|\vec{\gamma}\|_2 = 1}
\frac{(\vec{c}^\top \mat{E}_t \vec{\gamma})^2}{\|\mat{E}_t \vec{\gamma}\|_2^2}.
\]

Let $\vec{s} := \mat{E}_t^\top \vec{c} \in \mathbb{R}^K$ and $\mat{B} := \mat{E}_t^\top \mat{E}_t \in \mathbb{R}^{K\times K}$. Then
\[
(\vec{c}^\top \mat{E}_t \vec{\gamma})^2
= (\vec{s}^\top \vec{\gamma})^2
= \vec{\gamma}^\top (\vec{s}\vec{s}^\top) \vec{\gamma},
\quad
\|\mat{E}_t \vec{\gamma}\|_2^2
= \vec{\gamma}^\top \mat{B}\,\vec{\gamma},
\]
and the cosine objective becomes a generalized Rayleigh quotient:
\begin{equation}
    \label{eq:generalized_rayleigh}
    \max_{\|\vec{\gamma}\|_2 = 1}
    \frac{\vec{\gamma}^\top (\vec{s}\vec{s}^\top) \vec{\gamma}}
         {\vec{\gamma}^\top \mat{B}\,\vec{\gamma}}.
\end{equation}

\begin{proposition}[Cosine-optimal noise combination]
\label{prop:ncs_cos}
Let $\vec{c} \in \mathbb{R}^d$ and $\mat{E}_t \in \mathbb{R}^{d\times K}$ be as in Theorem~\ref{thm:ncs_optimal}, and assume that $\mat{B} = \mat{E}_t^\top \mat{E}_t$ is invertible. Then the maximizer of the cosine objective~\eqref{eq:ncs_cosine_objective} is given (up to normalization) by
\begin{equation}
    \label{eq:ncs_cos_solution}
    \vec{\gamma}^{\star}
    \;\propto\;
    \mat{B}^{-1} \vec{s}
    \;=\;
    (\mat{E}_t^\top \mat{E}_t)^{-1} \mat{E}_t^\top \vec{c},
\end{equation}
and the corresponding synthesized noise $\mat{E}_t \vec{\gamma}^{\star}$ is the orthogonal projection of $\vec{c}$ onto $\text{span}\{\vec{\epsilon}_t^1,\dots,\vec{\epsilon}_t^K\}$.
\end{proposition}

\begin{proof}
The maximization problem~\eqref{eq:generalized_rayleigh} is a generalized Rayleigh quotient with
$\mat{A} = \vec{s}\vec{s}^\top$ and $\mat{B} = \mat{E}_t^\top \mat{E}_t$. The maximizer is the leading generalized eigenvector of $(\mat{A},\mat{B})$, i.e., any nonzero solution of $\mat{A}\vec{\gamma} = \lambda \mat{B}\vec{\gamma}$. Since $\mat{A}$ has rank one,
\[
\vec{s}\vec{s}^\top \vec{\gamma} = \lambda \mat{B}\vec{\gamma}
\quad\Rightarrow\quad
\vec{s} \, (\vec{s}^\top \vec{\gamma}) = \lambda \mat{B}\vec{\gamma},
\]
so $\mat{B}\vec{\gamma}$ must be colinear with $\vec{s}$, that is $\mat{B}\vec{\gamma} \propto \vec{s}$, which yields~\eqref{eq:ncs_cos_solution} after normalization. The projection interpretation follows from
\(
\mat{E}_t \mat{B}^{-1} \mat{E}_t^\top
\)
being the orthogonal projector onto $\text{span}\{\vec{\epsilon}_t^i\}_{i=1}^K$.
\end{proof}

\paragraph{Relation to Theorem~\ref{thm:ncs_optimal}.}
The inner-product objective in Theorem~\ref{thm:ncs_optimal} leads to
\[
\vec{\gamma}^*_{\text{inner}}
\;\propto\;
\vec{s}
=
\mat{E}_t^\top \vec{c},
\]
while the cosine-optimal solution from Proposition~\ref{prop:ncs_cos} is
\[
\vec{\gamma}^{\star}_{\cos}
\;\propto\;
\mat{B}^{-1} \vec{s}
=
(\mat{E}_t^\top \mat{E}_t)^{-1} \mat{E}_t^\top \vec{c}.
\]
The two coincide exactly when $\mat{B}$ is a scalar multiple of the identity:
\[
\mat{E}_t^\top \mat{E}_t = \alpha \mat{I}_K
\quad\Rightarrow\quad
\vec{\gamma}^{\star}_{\cos} \propto \vec{\gamma}^*_{\text{inner}}.
\]
In our setting, the columns of $\mat{E}_t$ are i.i.d.\ standard Gaussian vectors in $\mathbb{R}^d$, and in the high-dimensional regime $d \gg K$ the Gram matrix concentrates around a scalar multiple of the identity:
\[
\mat{E}_t^\top \mat{E}_t \approx d\,\mat{I}_K,
\]
so that
\(
(\mat{E}_t^\top \mat{E}_t)^{-1} \approx \tfrac{1}{d}\mat{I}_K
\)
and
\(
\vec{\gamma}^{\star}_{\cos} \approx \vec{\gamma}^*_{\text{inner}}.
\)
Thus, Theorem~\ref{thm:ncs_optimal} can be viewed as a computationally convenient approximation to the cosine-optimal combination: maximizing the inner product with the score direction yields the same direction as maximizing the cosine similarity whenever the codebook noise vectors are approximately orthogonal, which is the typical case in high dimensions.

\section{Proof of Equivalence between NCS Optimization and NCS-MPGD}
\label{app:ncs_optimization}

\begin{theorem}[Equivalence of NCS Formulations]
    For the NCS optimization problem~\eqref{eq:ncs_optimization}, we propose the following formulation, in which we replace $\vec{x}_0$ in \citet{ohayon2025compressed} with the more accurate reconstruction $\vec{\tilde{x}}_{0|t}$.
    \begin{equation}
        \vec{\gamma}^* = \operatorname*{argmax}_{\vec{\gamma} \in \mathbb{R}^K, \|\vec{\gamma}\|_2 = 1} \left\langle \vec{y} - \mat{A}\tilde{\vec{x}}_{0 \mid t}, \mat{A}\sum_{i=1}^K \gamma_i \vec{\epsilon}^i_{t} \right\rangle,
    \end{equation}
    It is equivalent to the NCS-MPGD formulation:
    \begin{equation}
        \vec{\gamma}^* = \operatorname*{argmax}_{\vec{\gamma} \in \mathbb{R}^K,\, \|\vec{\gamma}\|_2 = 1} \left\langle - \mat{A}^\top (\vec{y} - \mat{A} \vec{\tilde{x}}_{0|t}),\ \sum_{i=1}^K \gamma_i \vec{\epsilon}^i_{t} \right\rangle.
    \end{equation}
    That is, both optimization problems have the same optimal solution $\vec{\gamma}^*$.
\end{theorem}

\begin{proof}
    We prove this equivalence by showing that the two objective functions are identical for any feasible $\vec{\gamma}$.
    
    Let $\vec{c} = \vec{y} - \mat{A}\tilde{\vec{x}}_{0 \mid t}$ and $\vec{\varepsilon} = \sum_{i=1}^K \gamma_i \vec{\epsilon}^i_{t}$. The standard NCS objective can be written as:
    \begin{align}
        \left\langle \vec{c}, \mat{A}\vec{\varepsilon} \right\rangle &= \vec{c}^\top (\mat{A}\vec{\varepsilon}) \\ \nonumber
        &= (\vec{c}^\top \mat{A}) \vec{\varepsilon} \\\nonumber
        &= \vec{\varepsilon}^\top (\vec{c}^\top \mat{A})^\top \\\nonumber
        &= \vec{\varepsilon}^\top \mat{A}^\top \vec{c} \\\nonumber
        &= \left\langle \mat{A}^\top \vec{c}, \vec{\varepsilon} \right\rangle.\nonumber
    \end{align}
    
    The NCS-MPGD objective is:
    \begin{align}
        \left\langle - \mat{A}^\top \vec{c}, \vec{\varepsilon} \right\rangle &= \left\langle - \mat{A}^\top \vec{c}, \sum_{i=1}^K \gamma_i \vec{\epsilon}^i_{t} \right\rangle \\
        &= \sum_{i=1}^K \gamma_i \left\langle - \mat{A}^\top \vec{c}, \vec{\epsilon}^i_{t} \right\rangle.
    \end{align}
    
    Since the two inner products are identical for every $\vec{\varepsilon}$ (and thus for every $\vec{\gamma}$), and the feasible set $\{\vec{\gamma} \in \mathbb{R}^K \mid \|\vec{\gamma}\|_2 = 1\}$ is the same for both problems, their maxima and argmax sets are identical.
    
    Furthermore, since each objective is a linear functional of $\vec{\gamma}$, the maximum over the unit sphere occurs when $\vec{\gamma}$ is aligned with the functional's direction. For the standard NCS formulation, this gives:
    \begin{equation}
        \vec{\gamma}^* = \frac{(\mat{A}^\top \vec{c})^\top \mat{E}_t}{\|(\mat{A}^\top \vec{c})^\top \mat{E}_t\|_2},
    \end{equation}
    where $\mat{E}_t = [\vec{\epsilon}_t^1, \ldots, \vec{\epsilon}_t^K]$ is the matrix of noise vectors.
    
    For the NCS-MPGD formulation, the optimal solution is:
    \begin{equation}
        \vec{\gamma}^* = \frac{(-\mat{A}^\top \vec{c})^\top \mat{E}_t}{\|(-\mat{A}^\top \vec{c})^\top \mat{E}_t\|_2} = \frac{-(\mat{A}^\top \vec{c})^\top \mat{E}_t}{\|(\mat{A}^\top \vec{c})^\top \mat{E}_t\|_2}.
    \end{equation}
    
    The negative sign in the NCS-MPGD formulation is due to the maximization of the negative inner product, which is equivalent to minimizing the positive inner product. However, since we are maximizing the absolute value of the alignment, both formulations yield the same optimal direction (up to a sign, which is normalized out by the unit norm constraint).
    
    Therefore, the two optimization problems are not merely equivalent, they are the same problem written in two different notations, and they achieve their maximum when the optimal $\vec{\gamma}^*$ is the same (up to normalization).
\end{proof}

\section{Quantization by NCS closed-form solution}
\label{app:ncs_quantization}

To utilize extra noise vectors to approximate the measurement score, \citet{ohayon2025compressed} proposes to use greedy search to find the noise vector that maximize the inner product between the measurement score and the noise vector, and to search the optimal quantization parameters on the selected noise vector and the next one. This process is computationally expensive, requiring hours to search for the optimal quantization parameters. If $C=10$, $m=10$, and $K=1024$, this requires approximately 10 hours. For notational simplicity, we omit the time subscript $t$ throughout this section.

\paragraph{Problem.}
Let $\{\vec{\epsilon}_i\}_{i=1}^N$ be given noise vectors and let $\vec{c}$ be the target.
Define $b_i := \langle \vec{\epsilon}_i, \vec{c}\rangle$ and (optionally) align signs so $b_i\ge 0$ by replacing $\vec{\epsilon}_i\leftarrow \operatorname{sgn}(b_i)\vec{\epsilon}_i$.
We form a mixture
\[
\vec{m}(\vec{\gamma}) \;=\; \sum_{i=1}^m \gamma_i\,\vec{\epsilon}_{(i)}, \qquad
\vec{\gamma}=(\gamma_1,\ldots,\gamma_m)^\top, \quad \sum_{i=1}^m \gamma_i^2 = 1,
\]
where $(i)$ indexes an ordered subset of $m$ atoms (e.g., the Top-$m$ by $|b_i|$).
The objective is to maximize alignment, i.e.
\[
\max_{\vec{\gamma}}\; \langle \vec{m}(\vec{\gamma}), \vec{c} \rangle
\;=\; \max_{\vec{\gamma}}\; \sum_{i=1}^m \gamma_i\, b_{(i)}
\quad \text{s.t.}\quad \sum_{i=1}^m \gamma_i^2=1.
\]

\paragraph{Closed-form (continuous) solution.}
Without quantization, the optimal coefficients on a fixed support are
\[
\vec{\gamma}^{\,*}
\;=\; \frac{\vec{b}_{S}}{\|\vec{b}_{S}\|_2},
\qquad \vec{b}_{S}=\big(b_{(1)},\ldots,b_{(m)}\big)^\top,
\]
by Cauchy–Schwarz. If the support $S$ is free, it is optimal to take the $m$ indices with largest $|b_i|$.

\paragraph{Definition (quantization via $\ell_2$ stick-breaking).}
We quantize by parameterizing $\vec{\gamma}$ through a stick-breaking map using
\[
u_i \in \mathcal{Q} \subset (0,1], \qquad i=1,\ldots,2^{C}-1,
\]
where $\mathcal{Q}$ is a finite grid with $2^{C-1} - 1$ elements (e.g., for $m = 3$ and $C = 2$, we have $\mathcal{Q} = \{0.33, 0.66, 1.0\}$).
The coefficients are then
\begin{align}
\gamma_1 &= \sqrt{u_1}, \\
\gamma_i &= \Big(\prod_{j=1}^{i-1}\sqrt{1-u_j}\Big)\sqrt{u_i}, \quad i=2,\ldots,m-1,\\
\gamma_m &= \prod_{j=1}^{m-1}\sqrt{1-u_j}.
\end{align}
By construction $\sum_{i=1}^m \gamma_i^2 = 1$ for any choices of $\{u_i\}$, so no final normalization is required.
The inverse map (from any feasible $\vec{\gamma}$ with $\sum \gamma_i^2=1$) is
\begin{align}
u_1 &= \gamma_1^2, \\
u_i &= \frac{\gamma_i^2}{1-\sum_{t=1}^{i-1}\gamma_t^2}, \quad i=2,\ldots,m-1.
\end{align}

\paragraph{Using the closed-form to obtain a quantized solution.}
We obtain a quantized solution directly from $\vec{\gamma}^{\,*}$ in two simple steps:
\begin{enumerate}
\item \textbf{Project $\vec{\gamma}^{\,*}$ into stick-breaking space.}
Compute $\{u_i^{*}\}_{i=1}^{m-1}$ from $\vec{\gamma}^{\,*}$ using the inverse map above.
(When $b_{(1)}\ge b_{(2)}\ge\cdots$, $\gamma_i^{*}\propto b_{(i)}$ is non-increasing, which matches the stick-breaking order.)

\item \textbf{Quantize and reconstruct.}
Independently quantize each stage by nearest-neighbor projection onto the grid,
\[
\hat{u}_i \;=\; \arg\min_{u\in\mathcal{Q}} \big|u-u_i^{*}\big|,
\qquad i=1,\ldots,m-1,
\]
then form $\hat{\vec{\gamma}}$ from $\{\hat{u}_i\}$ via the forward stick-breaking map.
\end{enumerate}
This yields $\hat{\vec{\gamma}}$ in $O(m)$ time and preserves $\sum_i \hat{\gamma}_i^2=1$ by construction.

\paragraph{Remark (stage-wise closed form and exact discrete refinement).}
If one optimizes stage-wise in the \emph{continuous} domain, the optimal fraction at stage $i$ has the closed form
\[
u_i^{\star}\;=\;\frac{b_{(i)}^{\,2}}{b_{(i)}^{\,2}+v_{i+1}^{\,2}},
\qquad
v_m=b_{(m)},\quad v_i = b_{(i)}\sqrt{u_i^{\star}}+v_{i+1}\sqrt{1-u_i^{\star}}.
\]
A discretized variant replaces $u_i^{\star}$ by the nearest grid point in $\mathcal{Q}$ at each stage (still $O(m)$).
For the \emph{exact} discrete optimum on $\mathcal{Q}$ one can use a 1D dynamic program:
\[
v_m=b_{(m)},\qquad
v_i=\max_{u\in\mathcal{Q}}\Big\{\,b_{(i)}\sqrt{u}+v_{i+1}\sqrt{1-u}\,\Big\},
\]
which selects $\hat{u}_i\in\mathcal{Q}$ per stage and remains $O\big(m\,|\mathcal{Q}|\big)$.

\section{Analysis of the Synthesized Noise Construction}
\label{app:pseudo-noise-analysis}
In this appendix, we analyze the distribution of the proposed synthesized noise vector $\vec{\epsilon}^{*}$ and justify our choice of the parameter $K$, i.e., the number of underlying noise samples used in the construction. Throughout, $d$ denotes the ambient dimension, and $\mathbf{v} \in \mathbb{R}^d$ is a fixed unit vector indicating the embedding direction. For notational simplicity, we omit the time subscript $t$ in this section, writing $\vec{\epsilon}^*$ and $\vec{\epsilon}_i$ instead of $\vec{\epsilon}_t^*$ and $\vec{\epsilon}_t^i$, since the analysis applies independently to each timestep.

\subsection{Decomposition of \texorpdfstring{$\vec{\epsilon}^{*}$}{epsilon*} into Signal and Orthogonal Gaussian Noise}

We first decompose each $\vec{\epsilon}_i$ along the direction $\mathbf{v}$ and its orthogonal complement:
\[
\vec{\epsilon}_i = s_i \mathbf{v} + \mathbf{w}_i, \quad i=1,\ldots,K,
\]
where
\[
s_i := \langle \vec{\epsilon}_i, \mathbf{v} \rangle \sim \mathcal{N}(0,1), \quad 
\mathbf{w}_i := \vec{\epsilon}_i - s_i \mathbf{v} \sim \mathcal{N}(0, I_d - \mathbf{v}\mathbf{v}^\top),
\]
and the families $\{ s_i \}_{i=1}^K$ and $\{ \mathbf{w}_i \}_{i=1}^K$ are independent, with all $s_i$ mutually independent and all $\mathbf{w}_i$ mutually independent. Since we have demonstrated the similarity between the two cosine and inner product problems in high dimensions when $K$ is much smaller than $d$, for simplicity of analysis, we can normalize $\vec{v}$ as $\vec{v}/|\vec{v}| \cdot \sqrt{d}$, where $\sqrt{d}$ is the expectation of the noise norm when $d$ is large.

Define
\[
T := \left( \sum_{i=1}^K s_i^2 \right)^{1/2}.
\]
Then $T$ follows a chi distribution with $K$ degrees of freedom, i.e., $T \sim \chi_K$, and the optimal coefficients become
\[
\gamma_i = \frac{ s_i }{ T }.
\]
Substituting into the expression for $\vec{\epsilon}^{*}$, we obtain
\begin{align*}
\vec{\epsilon}^{*} 
&= \sum_{i=1}^K \frac{s_i}{T} (s_i \mathbf{v} + \mathbf{w}_i)
= \sum_{i=1}^K \frac{s_i^2}{T} \mathbf{v} + \sum_{i=1}^K \frac{s_i}{T} \mathbf{w}_i \\
&= T \mathbf{v} + \mathbf{G},
\end{align*}
where
\[
\mathbf{G} := \sum_{i=1}^K \frac{s_i}{T} \mathbf{w}_i.
\]
We can characterize the joint distribution of $(T, \mathbf{G})$ as follows:

\begin{lemma}[Signal--noise decomposition]\label{lem:signal_noise_decomposition}
Let $T$ and $\mathbf{G}$ be defined as above. Then
\[
T \sim \chi_K, \quad \mathbf{G} \sim \mathcal{N}(0, I_d - \mathbf{v}\mathbf{v}^\top), \quad T \perp \mathbf{G},
\]
and consequently
\[
\vec{\epsilon}^{*} = T \mathbf{v} + \mathbf{G}.
\]
\end{lemma}

\begin{proof}
Conditional on $(s_1, \ldots, s_K)$, the vector $\mathbf{G}$ is a linear combination of jointly Gaussian vectors $\{ \mathbf{w}_i \}_{i=1}^K$, so $\mathbf{G} \mid s$ is Gaussian with mean zero. Its conditional covariance is
\[
\operatorname{Cov}(\mathbf{G} \mid s)
= \sum_{i=1}^K \left( \frac{s_i}{T} \right)^2 \operatorname{Cov}(\mathbf{w}_i)
= \left( \sum_{i=1}^K \frac{s_i^2}{T^2} \right) (I_d - \mathbf{v}\mathbf{v}^\top)
= I_d - \mathbf{v}\mathbf{v}^\top,
\]
which is independent of $s$. Thus, $\mathbf{G} \mid s \sim \mathcal{N}(0, I_d - \mathbf{v}\mathbf{v}^\top)$ for every realization of $s$. Integrating over $s$ shows that $\mathbf{G}$ itself is Gaussian with the same mean and covariance. 

Moreover, since the conditional law $\mathbf{G} \mid s$ does not depend on $s$, the random variables $\mathbf{G}$ and $s$ (and hence $\mathbf{G}$ and $T$) are independent. The identity $T \sim \chi_K$ follows from the standard fact that the sum of squares of $K$ independent $\mathcal{N}(0,1)$ variables is chi-square with $K$ degrees of freedom, and taking square roots yields a chi distribution.
\end{proof}

This lemma shows that $\vec{\epsilon}^{*}$ can be written as a ``signal-plus-orthogonal-Gaussian-noise'' decomposition:
\[
    \vec{\epsilon}^{*} = T \mathbf{v} + \mathbf{G},
\]
with a random amplitude $T$ along $\mathbf{v}$ and an exactly Gaussian component $\mathbf{G}$ in the orthogonal complement $\mathbf{v}^\perp$.

\subsection{Non-Gaussianity and Rank-1 Deviation from Isotropic Noise}

For an ideal isotropic Gaussian noise vector $\vec{\epsilon} \sim \mathcal{N}(0, I_d)$, the projection onto $\mathbf{v}$ satisfies
\[
\langle \vec{\epsilon}, \mathbf{v} \rangle \sim \mathcal{N}(0,1).
\]
In contrast, for the synthesized noise $\vec{\epsilon}^*$, Lemma~\ref{lem:signal_noise_decomposition} implies
\[
\langle \vec{\epsilon}^{*}, \mathbf{v} \rangle = \langle T \mathbf{v} + \mathbf{G}, \mathbf{v} \rangle = T \sim \chi_K,
\]
which is a chi distribution (non-symmetric and supported on $[0, \infty)$). Therefore, as soon as $K \geq 2$, the distribution of $\langle \vec{\epsilon}^*, \mathbf{v} \rangle$ is not Gaussian.

Using the decomposition $\vec{\epsilon}^{*} = T \mathbf{v} + \mathbf{G}$, we can compute the mean and covariance of $\vec{\epsilon}^{*}$. Since $\mathbb{E}[\mathbf{G}]=0$ and $T$ and $\mathbf{G}$ are independent,
\[
\mathbb{E}[\vec{\epsilon}^{*}] = \mathbb{E}[T]\, \mathbf{v},
\]
and
\[
\operatorname{Cov}(\vec{\epsilon}^{*}) = \operatorname{Var}(T)\, \mathbf{v}\mathbf{v}^\top + (I_d - \mathbf{v} \mathbf{v}^\top).
\]
Thus, compared to an ideal $\mathcal{N}(0, I_d)$, the synthesized noise exhibits:
\begin{itemize}
    \item a non-zero mean along $\mathbf{v}$, with magnitude $\mathbb{E}[T]$;
    \item a modified variance along $\mathbf{v}$ equal to $\operatorname{Var}(T)$, while the variance in $\mathbf{v}^{\perp}$ remains exactly $1$.
\end{itemize}

Since $\mathbf{G}$ is exactly Gaussian and lives in $\mathbf{v}^\perp$, all non-Gaussianity of $\vec{\epsilon}^{*}$ is confined to the one-dimensional subspace spanned by $\mathbf{v}$. In particular, $\vec{\epsilon}^{*}$ restricted to $\mathbf{v}^\perp$ is exactly $\mathcal{N}(0, I_d - \mathbf{v}\mathbf{v}^\top)$.

\subsection{Distribution of the Angle Between \texorpdfstring{$\vec{\epsilon}^{*}$}{epsilon*} and the Embedding Direction}

To quantify how strongly $\vec{\epsilon}^{*}$ is aligned with the embedding direction $\mathbf{v}$, we consider the angle $\theta$ between $\vec{\epsilon}^{*}$ and $\mathbf{v}$, defined by
\[
\cos \theta = \frac{ \langle \vec{\epsilon}^{*}, \mathbf{v} \rangle }{ \|\vec{\epsilon}^{*}\|\,\|\mathbf{v}\| } = \frac{ \langle \vec{\epsilon}^{*}, \mathbf{v} \rangle }{ \| \vec{\epsilon}^{*} \| },
\]
We have $\|\mathbf{v}\| \sim \sqrt{d}$. Using the decomposition from Lemma~\ref{lem:signal_noise_decomposition},
\[
\cos^2 \theta = \frac{ \langle \vec{\epsilon}^{*}, \mathbf{v} \rangle^2 }{ \| \vec{\epsilon}^{*} \|^2 } = \frac{ T^2 }{ T^2 + \| \mathbf{G} \|^2 }.
\]
We now characterize the distribution of $\cos^2 \theta$.

\begin{lemma}[Angle distribution]\label{lem:angle_distribution}
Let $\theta$ be the angle between $\vec{\epsilon}^{*}$ and $\mathbf{v}$. Then
\[
\cos^2 \theta \sim \text{Beta}\left( \tfrac{K}{2}, \tfrac{d-1}{2} \right),
\]
and, in particular,
\[
\mathbb{E}[ \cos^2 \theta ] = \frac{ K }{ K + d - 1 }.
\]
\end{lemma}

\begin{proof}
From Lemma~\ref{lem:signal_noise_decomposition}, $T^2 \sim \chi^2_K$, $\|\mathbf{G}\|^2 \sim \chi^2_{d-1}$, and $T^2$ and $\|\mathbf{G}\|^2$ are independent. It is a standard fact that if $X \sim \chi^2_{k_1}$, $Y \sim \chi^2_{k_2}$ are independent, then $\frac{X}{X+Y} \sim \text{Beta}(k_1/2, k_2/2)$. Applying this with $X = T^2$ and $Y = \|\mathbf{G}\|^2$ yields the result. The expectation of a Beta$(\alpha, \beta)$ random variable is $\alpha/(\alpha+\beta)$. Substituting $\alpha = K/2$, $\beta =(d-1)/2$ gives the second part.
\end{proof}

For comparison, if $\vec{\epsilon} \sim \mathcal{N}(0, I_d)$ is an ideal isotropic Gaussian noise vector, then its direction $\vec{\epsilon} / \| \vec{\epsilon} \|$ is uniformly distributed on the unit sphere, and the squared cosine with any fixed unit vector $\mathbf{v}$ satisfies
\[
\cos^2 \theta_\mathrm{Gauss} \sim \text{Beta}\left( \frac{1}{2}, \frac{d-1}{2} \right), \quad \mathbb{E}[ \cos^2 \theta_\mathrm{Gauss} ] = \frac{1}{d}.
\]
Hence, the synthesized noise amplifies the average squared alignment with $\mathbf{v}$ from $1/d$ (ideal isotropic Gaussian) to
\[
\mathbb{E}[ \cos^2 \theta_{\vec{\epsilon}^{*}} ] = \frac{ K }{ K + d - 1 }.
\]

A convenient way to summarize this effect is via the energy amplification factor
\[
\rho(K,d) := \frac{ \mathbb{E}[ \cos^2 \theta_{\vec{\epsilon}^{*}} ] }{ \mathbb{E}[ \cos^2 \theta_\mathrm{Gauss} ] }
= \frac{ K }{ K + d - 1 } \cdot d.
\]
When $K \ll d$, this factor behaves as $\rho(K,d) \approx K$; thus the expected energy in the embedding direction is amplified roughly by a factor of $K$ compared to an isotropic Gaussian.

\subsection{Choice of $K$ and the $K \approx \sqrt{d}$ Scaling}

The parameter $K$ (the number of codebook noises used in the combination) governs a fundamental trade-off between:
\begin{itemize}
    \item \textbf{Embedding strength:} Larger $K$ increases the energy that $\vec{\epsilon}^*$ carries along the embedding direction $\mathbf{v}$ and thus improves the effective directional SNR for decoding.
    \item \textbf{Closeness to isotropic Gaussian noise:} Larger $K$ also makes the distribution of $\vec{\epsilon}^*$ more strongly biased toward $\mathbf{v}$, and thus more easily distinguishable from an isotropic Gaussian, especially along the one-dimensional subspace $\mathrm{span}\{\mathbf{v}\}$.
\end{itemize}

Lemma~\ref{lem:angle_distribution} shows that
\[
\mathbb{E}\bigl[ \cos^2 \theta_{\vec{\epsilon}^{*}} \bigr] = \frac{ K }{ K + d - 1 }.
\]
For an isotropic Gaussian baseline $\boldsymbol{\varepsilon}\sim\mathcal{N}(0,I_d)$, we have
$\mathbb{E}[\cos^2\theta_{\boldsymbol{\varepsilon}}] \approx 1/d$.
Thus, the fraction of energy in the embedding direction under NCS is
\[
\frac{\mathbb{E}[ \cos^2 \theta_{\vec{\epsilon}^{*}} ]}{\mathbb{E}[\cos^2\theta_{\boldsymbol{\varepsilon}}]}
\approx \frac{K}{K+d-1}\cdot d,
\]
which quantifies how much more strongly $\vec{\epsilon}^*$ concentrates along $\mathbf{v}$ than an isotropic Gaussian.

More importantly for decoding, we consider the \emph{directional SNR} along $\mathbf{v}$, defined as the ratio between the expected energy in the $\mathbf{v}$-direction and the expected energy in its orthogonal complement:
\[
\mathrm{SNR}(K,d)
:= \frac{\mathbb{E}[\|\mathrm{Proj}_{\mathbf{v}}(\vec{\epsilon}^*)\|^2]}
        {\mathbb{E}[\|\mathrm{Proj}_{\mathbf{v}^\perp}(\vec{\epsilon}^*)\|^2]}
= \frac{K}{d-1}.
\]
(For comparison, an isotropic Gaussian satisfies
$\mathrm{SNR}_{\mathrm{iso}}(d) \approx 1/(d-1)$.)

We now justify choosing $K$ on the order of $\sqrt{d}$. Consider the scaling
\[
K = c \sqrt{d}
\]
for some constant $c>0$. Substituting this into Lemma~\ref{lem:angle_distribution},
\[
\mathbb{E}\bigl[ \cos^2 \theta_{\vec{\epsilon}^{*}} \bigr]
= \frac{ c \sqrt{d} }{ c \sqrt{d} + d - 1 }
= \frac{ c }{ \sqrt{d} + c - \frac{1}{\sqrt{d}} }
= O\!\left(\frac{1}{\sqrt{d}}\right)
\quad \text{as } d \to \infty.
\]
Thus, as the dimension grows, the fraction of energy in the embedding direction still vanishes, i.e.,
\[
\mathbb{E}[ \cos^2 \theta_{\vec{\epsilon}^{*}} ] \to 0 \quad \text{as } d \to \infty.
\]
From the perspective of generic low-dimensional projections and standard normality tests, the synthesized noise therefore remains hard to distinguish from an isotropic Gaussian.

At the same time, the directional SNR becomes
\[
\mathrm{SNR}(K,d)
= \frac{K}{d-1}
= \frac{c \sqrt{d}}{d-1}
= \Theta\!\left(\frac{1}{\sqrt{d}}\right).
\]
In other words, the absolute SNR along $\mathbf{v}$ still decays as $1/\sqrt{d}$ in high dimensions, but it is amplified by a factor of order $K \asymp \sqrt{d}$ compared to the isotropic Gaussian baseline, for which $\mathrm{SNR}_{\mathrm{iso}}(d) \approx 1/d$. This scaling is sufficient to significantly improve reliable decoding while keeping the synthesized noise visually close to standard Gaussian noise.

\begin{figure}[H]
    \centering
    % Title row
    \begin{tabular}{ccccc}
        \textbf{Measurement} & \textbf{DDCM:100} & \textbf{DDCM:1000} & \textbf{NCS-DPS:100} & \textbf{NCS-DPS:1000} \\
        % Row 1: 5
        \includegraphics[width=0.18\textwidth]{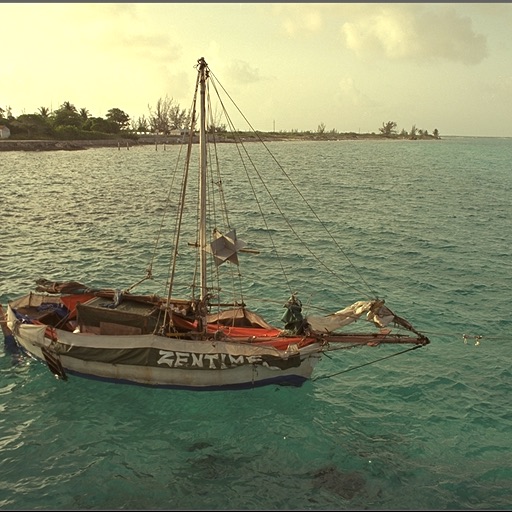} &
        \includegraphics[width=0.18\textwidth]{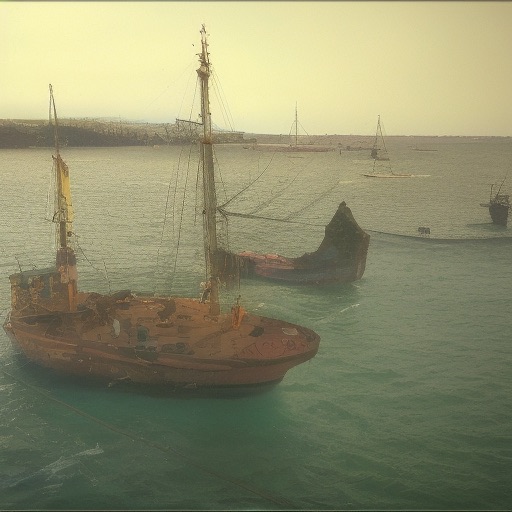} &
        \includegraphics[width=0.18\textwidth]{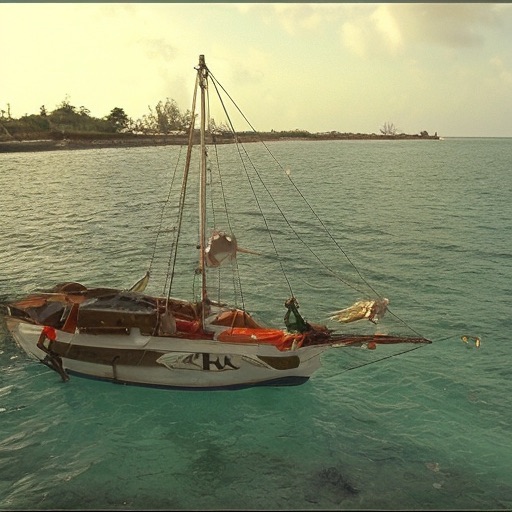} &
        \includegraphics[width=0.18\textwidth]{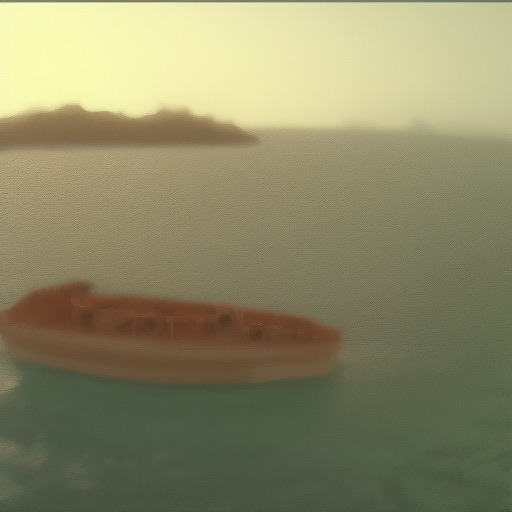} &
        \includegraphics[width=0.18\textwidth]{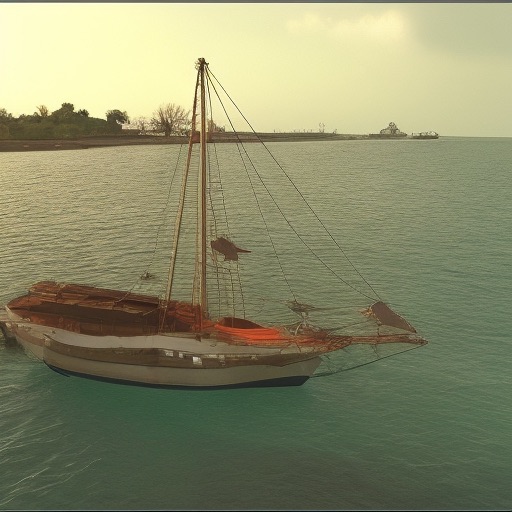} \\
        % Row 2: 14
        \includegraphics[width=0.18\textwidth]{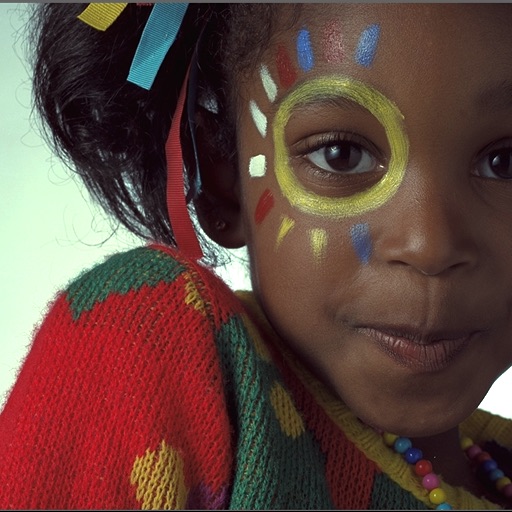} &
        \includegraphics[width=0.18\textwidth]{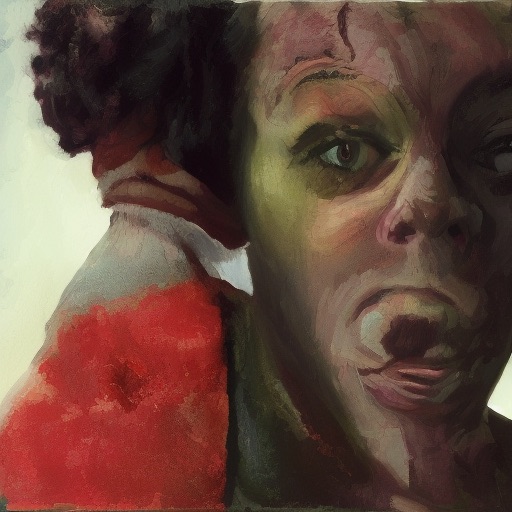} &
        \includegraphics[width=0.18\textwidth]{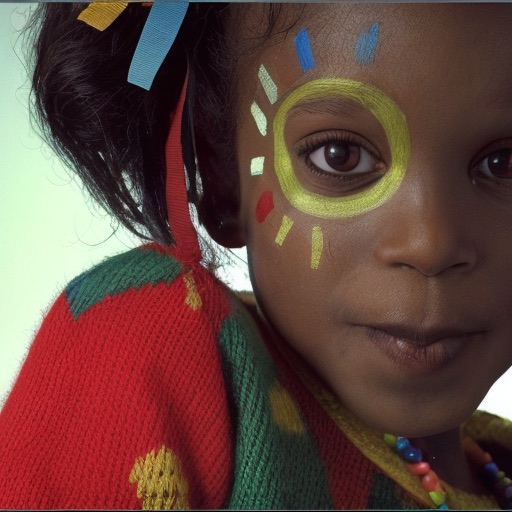} &
        \includegraphics[width=0.18\textwidth]{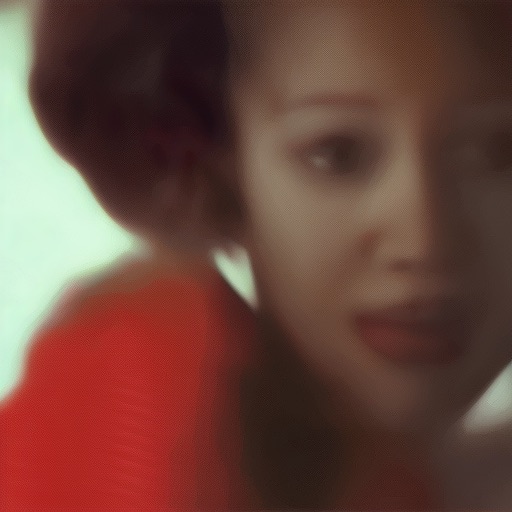} &
        \includegraphics[width=0.18\textwidth]{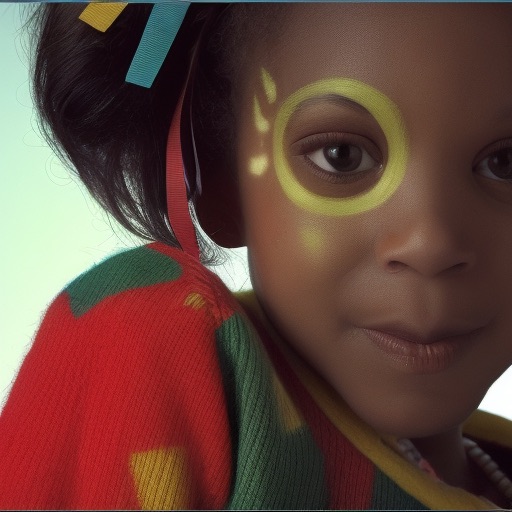} \\
        % Row 3: 16
        \includegraphics[width=0.18\textwidth]{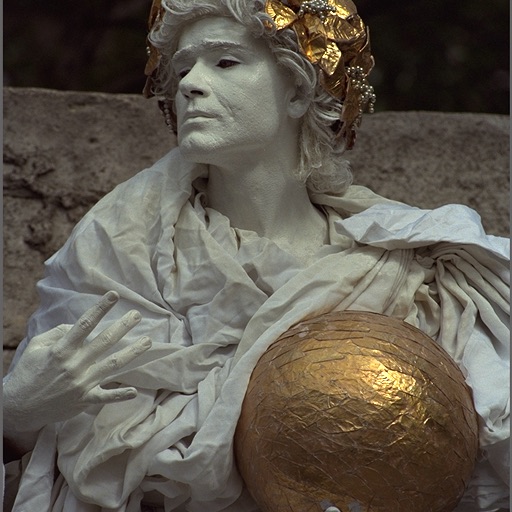} &
        \includegraphics[width=0.18\textwidth]{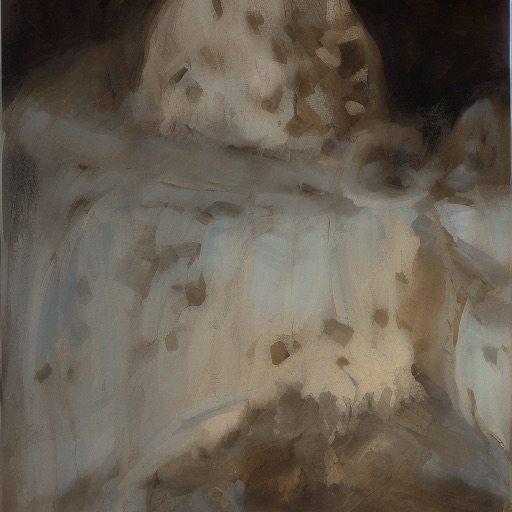} &
        \includegraphics[width=0.18\textwidth]{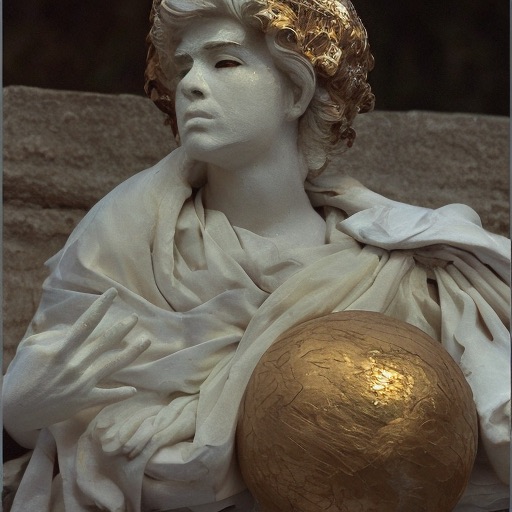} &
        \includegraphics[width=0.18\textwidth]{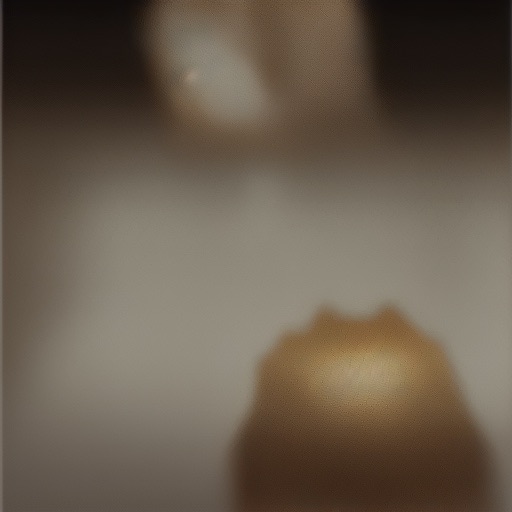} &
        \includegraphics[width=0.18\textwidth]{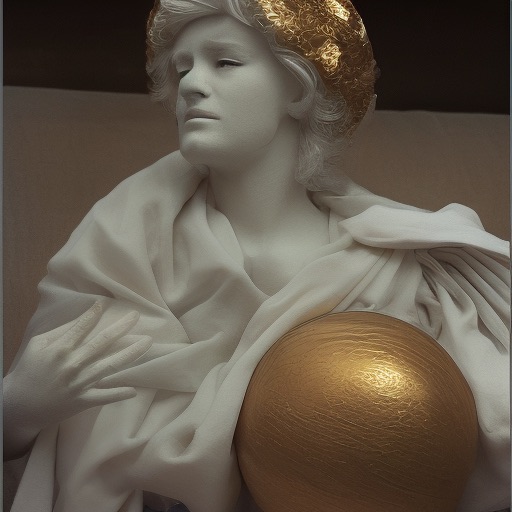} \\
    \end{tabular}
    \caption{Compression results on Kodak24 images. Each row corresponds to a different image. Columns: Original, DDCM (100 steps), DDCM (1000 steps), NCS-DPS (100 steps), NCS-DPS (1000 steps).}
    \label{fig:kodak24_samples}
\end{figure}

\section{Nonlinear Inverse Problems}
\label{app:nonlinear}

We evaluate NCS on two nonlinear inverse problems: nonlinear deblurring and phase retrieval. As shown in Table~\ref{tab:nonlinear_inverse_results}, for nonlinear deblurring, NCS yields noticeable improvements over the baseline methods, particularly in the low-step regime (e.g., 20 steps). However, as the number of sampling steps increases, the performance gap diminishes. For phase retrieval, NCS does not provide consistent improvements; in some cases, the NCS variants perform comparably or even worse than the baselines. Understanding the underlying causes of this behavior remains an interesting direction for future investigation.

\begin{table*}[!t]
    \centering
    \caption{Quantitative comparison on nonlinear inverse problems (FFHQ and ImageNet). Each cell shows PSNR\,/\,FID\,/\,LPIPS.}
    \label{tab:nonlinear_inverse_results}
    \begin{footnotesize}
    \begin{sc}
    \begin{tabular}{lllccc}
    \toprule
    & & & \multicolumn{3}{c}{PSNR(↑)\,/\,FID(↓)\,/\,LPIPS(↓)}\\
    \cmidrule(lr){4-6}
    Dataset & Task & Method & 20 & 100 & 1000 \\
    \midrule
    \multirow{14}{*}{FFHQ}
        & \multirow{7}{*}{\shortstack{Nonlinear\\Deblur}}
            & DPS              & 12.14\,/\,153.1\,/\,0.491 & 19.13\,/\,103.9\,/\,0.273 & 22.78\,/\,91.90\,/\,0.207 \\
        &   & \cellcolor{gray!20}\textbf{NCS-DPS} & \cellcolor{gray!20}\textbf{21.03}\,/\,\textbf{121.6}\,/\,\textcolor{blue}{\textbf{0.269}} & \cellcolor{gray!20}\textbf{23.03}\,/\,\textbf{97.62}\,/\,\textcolor{blue}{\textbf{0.202}} & \cellcolor{gray!20}\textbf{23.78}\,/\,\textbf{92.33}\,/\,\textbf{0.208} \\
        &   & MPGD             & \textbf{19.11}\,/\,\textbf{120.8}\,/\,\textbf{0.294} & \textbf{21.50}\,/\,127.4\,/\,\textbf{0.265} & 21.33\,/\,159.3\,/\,0.328 \\
        &   & \cellcolor{gray!20}\textbf{NCS-MPGD} & \cellcolor{gray!20}16.95\,/\,124.5\,/\,0.351 & \cellcolor{gray!20}18.95\,/\,\textbf{98.00}\,/\,0.265 & \cellcolor{gray!20}\textbf{22.80}\,/\,\textbf{95.14}\,/\,\textcolor{blue}{\textbf{0.196}} \\
        &   & DAPS             & 26.25\,/\,\textcolor{blue}{\textbf{76.81}}\,/\,0.314 & \textcolor{blue}{\textbf{27.23}}\,/\,59.52\,/\,\textbf{0.236} & \textcolor{blue}{\textbf{28.57}}\,/\,{{56.20}}\,/\,{\textbf{0.208}} \\
        &   & \cellcolor{gray!20}\textbf{NCS-DAPS} & \cellcolor{gray!20}\textcolor{blue}{\textbf{26.84}}\,/\,79.18\,/\,\textbf{0.279} & \cellcolor{gray!20}26.96\,/\,\textcolor{blue}{\textbf{55.68}}\,/\,0.239 & \cellcolor{gray!20}27.27\,/\,\textcolor{blue}{\textbf{52.39}}\,/\,0.223 \\
        &   & RED-Diff         & 13.12\,/\,140.2\,/\,1.330 & 20.68\,/\,146.1\,/\,0.582 & 22.81\,/\,138.9\,/\,0.401 \\
    \cmidrule(lr){2-6}
        & \multirow{7}{*}{\shortstack{Phase\\Retrieval}}
            & DPS              & 10.46\,/\,{145.8}\,/\,0.523 & 12.79\,/\,\textcolor{blue}{\textbf{118.4}}\,/\,0.434 & 15.07\,/\,{131.8}\,/\,0.373 \\
        &   & \cellcolor{gray!20}\textbf{NCS-DPS} & \cellcolor{gray!20}\textbf{13.03}\,/\,\textbf{143.2}\,/\,\textbf{0.447} & \cellcolor{gray!20}\textcolor{blue}{\textbf{17.05}}\,/\,168.1\,/\,\textbf{0.394} & \cellcolor{gray!20}\textcolor{blue}{\textbf{20.96}}\,/\,179.6\,/\,\textcolor{blue}{\textbf{0.282}} \\
        &   & MPGD             & 12.00\,/\,\textcolor{blue}{\textbf{138.4}}\,/\,0.473 & 12.85\,/\,\textbf{167.0}\,/\,0.457 & 12.64\,/\,347.8\,/\,0.572 \\
        &   & \cellcolor{gray!20}\textbf{NCS-MPGD} & \cellcolor{gray!20}\textcolor{blue}{\textbf{13.81}}\,/\,166.6\,/\,\textcolor{blue}{\textbf{0.426}} & \cellcolor{gray!20}\textbf{15.60}\,/\,208.7\,/\,\textcolor{blue}{\textbf{0.389}} & \cellcolor{gray!20}\textbf{15.13}\,/\,\textbf{311.6}\,/\,\textbf{0.531} \\
        &   & DAPS             & \textbf{12.72}\,/\,429.3\,/\,\textbf{0.677}& \textbf{11.54}\,/\,\textbf{205.4}\,/\,\textbf{0.722} & \textbf{13.05}\,/\,\textcolor{blue}{\textbf{66.39}}\,/\,\textbf{0.636} \\
        &   & \cellcolor{gray!20}\textbf{NCS-DAPS} & \cellcolor{gray!20}8.63\,/\,\textbf{302.9}\,/\,0.782 & \cellcolor{gray!20}8.64\,/\,285.0\,/\,0.765 & \cellcolor{gray!20}8.92\,/\,104.7\,/\,0.791 \\
        &   & RED-Diff         & 11.52\,/\,342.3\,/\,1.002 & 12.18\,/\,448.3\,/\,0.826 & 11.95\,/\,432.7\,/\,0.820 \\
    \midrule
    \multirow{14}{*}{ImageNet}
        & \multirow{7}{*}{\shortstack{Nonlinear\\Deblur}}
            & DPS              & 11.95\,/\,248.3\,/\,0.689 & 16.68\,/\,201.6\,/\,0.531 & 20.55\,/\,\textbf{169.2}\,/\,0.392 \\
        &   & \cellcolor{gray!20}\textbf{NCS-DPS} & \cellcolor{gray!20}\textbf{18.48}\,/\,\textbf{237.0}\,/\,\textbf{0.531} & \cellcolor{gray!20}\textbf{20.37}\,/\,\textbf{183.9}\,/\,\textbf{0.425} & \cellcolor{gray!20}\textbf{21.87}\,/\,175.0\,/\,\textbf{0.366} \\
        &   & MPGD             & 14.43\,/\,297.3\,/\,0.785 & 16.44\,/\,241.9\,/\,0.700 & 16.87\,/\,253.5\,/\,0.620 \\
        &   & \cellcolor{gray!20}\textbf{NCS-MPGD} & \cellcolor{gray!20}\textbf{16.49}\,/\,\textbf{225.3}\,/\,\textbf{0.556} & \cellcolor{gray!20}\textbf{17.22}\,/\,\textbf{194.5}\,/\,\textbf{0.462} & \cellcolor{gray!20}\textbf{18.95}\,/\,\textbf{188.1}\,/\,\textbf{0.472} \\
        &   & DAPS             & 23.34\,/\,140.3\,/\,0.371 & \textcolor{blue}{\textbf{24.10}}\,/\,104.5\,/\,\textcolor{blue}{\textbf{0.313}} & 25.33\,/\,\textcolor{blue}{\textbf{91.60}}\,/\,\textcolor{blue}{\textbf{0.283}} \\
        &   & \cellcolor{gray!20}\textbf{NCS-DAPS} & \cellcolor{gray!20}\textcolor{blue}{\textbf{23.98}}\,/\,\textcolor{blue}{\textbf{120.1}}\,/\,\textcolor{blue}{\textbf{0.351}} & \cellcolor{gray!20}23.83\,/\,\textcolor{blue}{\textbf{94.89}}\,/\,0.331 & \cellcolor{gray!20} \textcolor{blue}{\textbf{25.73}}\,/\,{{92.02}}\,/\,{{0.298}} \\
        &   & RED-Diff         & 12.47\,/\,342.3\,/\,1.310 & 18.67\,/\,225.0\,/\,0.713 & 19.99\,/\,202.6\,/\,0.544 \\
    \cmidrule(lr){2-6}
        & \multirow{7}{*}{\shortstack{Phase\\Retrieval}}
            & DPS              & 10.58\,/\,\textcolor{blue}{\textbf{251.9}}\,/\,0.694 & 11.36\,/\,\textcolor{blue}{\textbf{237.9}}\,/\,0.618 & 13.06\,/\,\textbf{256.5}\,/\,0.600 \\
        &   & \cellcolor{gray!20}\textbf{NCS-DPS} & \cellcolor{gray!20}\textcolor{blue}{\textbf{12.15}}\,/\,290.9\,/\,\textcolor{blue}{\textbf{0.602}} & \cellcolor{gray!20}\textcolor{blue}{\textbf{13.11}}\,/\,296.5\,/\,\textcolor{blue}{\textbf{0.589}} & \cellcolor{gray!20}\textcolor{blue}{\textbf{13.50}}\,/\,301.7\,/\,\textcolor{blue}{\textbf{0.587}} \\
        &   & MPGD             & 10.45\,/\,\textbf{268.0}\,/\,0.676 & \textbf{9.26}\,/\,\textbf{278.3}\,/\,0.829 & 7.62\,/\,\textbf{333.2}\,/\,0.897 \\
        &   & \cellcolor{gray!20}\textbf{NCS-MPGD} & \cellcolor{gray!20}\textbf{12.03}\,/\,291.8\,/\,\textbf{0.637} & \cellcolor{gray!20}7.42\,/\,309.5\,/\,\textbf{0.755} & \cellcolor{gray!20}\textbf{9.47}\,/\,335.2\,/\,\textbf{0.809} \\
        &   & DAPS             & \textbf{11.71}\,/\,429.3\,/\,\textbf{0.677}& \textbf{12.86}\,/\,\textbf{298.2}\,/\,\textbf{0.602} & \textbf{13.05}\,/\,\textcolor{blue}{\textbf{66.39}}\,/\,\textbf{0.636} \\
        &   & \cellcolor{gray!20}\textbf{NCS-DAPS} & \cellcolor{gray!20}{6.74}\,/\,\textbf{306.6}\,/\,{0.716} & \cellcolor{gray!20}9.86\,/\,301.5\,/\,0.628 & \cellcolor{gray!20}8.62\,/\,104.7\,/\,0.791 \\
        &   & RED-Diff         & 11.38\,/\,334.8\,/\,0.994 & 11.97\,/\,373.1\,/\,0.854 & 11.81\,/\,374.1\,/\,0.848 \\
    \bottomrule
    \end{tabular}
    \end{sc}
    \end{footnotesize}
\end{table*}

\section{Inverse Problems Setting}
\label{app:inverse_problems}

This appendix details the experimental setup for inverse problems evaluated in the main paper. We consider two categories: (1) image restoration tasks on FFHQ and ImageNet, and (2) scientific inverse problems from InverseBench~\citep{zheng2025inversebench}.

\subsection{Image Restoration Tasks}

We follow the problem formulations in~\citet{chung2023diffusion}. All tasks use additive Gaussian noise with $\sigma = 0.05$.

\vspace{0.3em}
\noindent\textbf{Linear tasks.}
\emph{Super-resolution} ($\times$4, $\times$8): bicubic downsampling.
\emph{Inpainting} (box, random): pixel masking.
\emph{Deblurring} (Gaussian, motion): convolution with blur kernels.

\vspace{0.3em}
\noindent\textbf{Nonlinear tasks.}
\emph{Non-uniform deblurring}: learned blur generator $\mathcal{B}_\theta(\vec{x}_0, \vec{z})$ with latent $\vec{z} \sim \mathcal{N}(\vec{0}, \sigma_z^2 \mat{I})$.
\emph{Phase retrieval}: magnitude-only Fourier measurements $\vec{y} = |\mat{F}\vec{x}_0| + \vec{n}$.

\vspace{0.3em}
\noindent\textbf{Metrics.} PSNR, FID, and LPIPS. For phase retrieval, we report the best among multiple samples due to non-uniqueness.

\subsection{Scientific Inverse Problems}

We follow InverseBench~\citep{zheng2025inversebench} and evaluate three representative tasks.

\vspace{0.3em}
\noindent\textbf{Linear inverse scattering.}
Recovering permittivity contrast from scattered light fields. Under the Born approximation: $\vec{y} = \mat{A}\vec{x}_0 + \vec{n}$, where $\mat{A}$ is derived from Green's functions. Complex-valued measurements are split into real/imaginary parts.

\vspace{0.3em}
\noindent\textbf{Compressed sensing MRI.}
Multi-coil parallel imaging with subsampled k-space: $\vec{y}_c = \mat{P}\mat{F}(\mat{S}_c \vec{x}_0) + \vec{n}_c$, where $\mat{P}$ is the sampling mask, $\mat{F}$ the Fourier transform, and $\mat{S}_c$ the coil sensitivity.

\vspace{0.3em}
\noindent\textbf{Black hole imaging (VLBI).}
Sparse visibility measurements with station-dependent errors. Inference uses \emph{closure quantities}---closure phases and log-closure amplitudes---which are nonlinear transformations of visibilities.

\vspace{0.3em}
\noindent\textbf{Metrics.} PSNR and SSIM for reconstruction fidelity. For VLBI, we additionally report $\chi^2_{\text{cp}}$ (closure phase) and $\chi^2_{\text{logca}}$ (log-closure amplitude); values near $1$ indicate good observation fitting.

\subsection{Hyperparameters}

We choose the optimal stepsize of DPS according to \citet{chung2023diffusion}. For the NCS method, the strength of the measurement operator affects the effective dimension of the noise space available for approximation. Consequently, we choose the parameter $K$ differently for each task, as summarized in Table~\ref{tab:k_settings}.

\begin{table}[H]
\centering
\caption{Codebook size $K$ for different inverse problem tasks.}
\label{tab:k_settings}
\vspace{0.5em}
\begin{tabular}{@{}cc@{\hspace{2em}}cc@{}}
\toprule
\multicolumn{2}{c}{\textbf{ImageNet / FFHQ}} & \multicolumn{2}{c}{\textbf{Scientific Inverse Problems}} \\
\cmidrule(r){1-2} \cmidrule(l){3-4}
Task & $K$ & Task & $K$ \\
\midrule
Super-resolution $\times$4 & 512 & Black hole imaging & 64 \\
Super-resolution $\times$8 & 64 & Linear inverse-scattering & 128 \\
Inpainting (box) & 512 & Compressed sensing MRI & 256 \\
Inpainting (random) & 512 & & \\
Gaussian deblurring & 256 & & \\
Motion deblurring & 512 & & \\
Phase retrieval & 256 & & \\
\bottomrule
\end{tabular}
\end{table}

\noindent The choice of $K \approx \sqrt{n}$ is motivated by the effective noise dimension $n$, which depends on the information loss induced by $\mat{A}$. For $\times$8 SR, the effective dimension is only $64 \times 64 \times 3$, so a smaller $K$ suffices. 

\begin{table}[H]
\centering
\caption{Hyperparameters for scientific inverse problems. We follow~\citet{zheng2025inversebench} for DPS and RED-Diff settings. For DAPS, we conduct a limited search over Langevin step size due to observed degradation with default parameters.}
\label{tab:scientific_hyperparams}
\vspace{0.5em}
\begin{tabular}{@{}llccc@{}}
\toprule
Method & Parameter & \shortstack{Linear Inverse\\Scattering (60)} & Black Hole & MRI (Real) \\
\midrule
DPS & Guidance scale & 625 & 0.003 & 0.428 \\
\midrule
\multirow{4}{*}{RED-Diff} 
    & Learning rate & 0.04 & 0.05 & $2.96\times10^{-2}$ \\
    & Regularization $\lambda_{\text{base}}$ & 0.0005 & 0.25 & $2.72\times10^{-3}$ \\
    & Regularization schedule & constant & constant & sqrt \\
    & Gradient weight & 1500 & 0.0004 & $1.7\times10^{-2}$ \\
\midrule
\multirow{5}{*}{DAPS} 
    & Annealing step & 200 & 100 & 200 \\
    & Diffusion step & 10 & 5 & 5 \\
    & Langevin step size & $1\times10^{-4}$ & $1\times10^{-4}$ & $1.52\times10^{-5}$ \\
    & Langevin step number & 50 & 20 & 100 \\
    & Noise level ($\tau$) & $1\times10^{-4}$ & 1 & $4.77\times10^{-3}$ \\
\bottomrule
\end{tabular}
\end{table}

\section{Computational Cost}
\label{app:computational_cost}
We benchmark the computational overhead of NCS on a single NVIDIA RTX 4090 GPU with 50 diffusion timesteps and batch size 1. As shown in Table~\ref{tab:computational_cost}, the overhead introduced by NCS is negligible in practice.

\begin{table}[H]
\centering
\caption{Computational cost comparison between baseline methods and their NCS variants.}
\label{tab:computational_cost}
\vspace{0.5em}
\begin{tabular}{@{}lccc@{}}
\toprule
Method & Total Time & Slowdown & Per-step Overhead \\
\midrule
DPS & 2.58s & -- & -- \\
NCS-DPS ($K$=256) & 2.61s & 1.01$\times$ & +1.71ms \\
NCS-DPS ($K$=512) & 2.64s & 1.02$\times$ & +2.82ms \\
\midrule
MPGD & 1.26s & -- & -- \\
NCS-MPGD ($K$=256) & 1.28s & 1.02$\times$ & +1.14ms \\
NCS-MPGD ($K$=512) & 1.31s & 1.04$\times$ & +2.46ms \\
\bottomrule
\end{tabular}
\end{table}

\begin{table}[H]
\centering
\caption{NCS computation overhead and memory usage for different $K$ values.}
\label{tab:ncs_overhead}
\vspace{0.5em}
\begin{tabular}{@{}lcc@{}}
\toprule
$K$ & Pure NCS Overhead (per step) & Additional Memory \\
\midrule
256 & $\sim$1.4ms & 192 MB \\
512 & $\sim$2.7ms & 384 MB \\
1024 & $\sim$5.4ms & 768 MB \\
\bottomrule
\end{tabular}
\end{table}

The computational and memory overhead of NCS is negligible: the slowdown is only 1--4\% because model inference (UNet forward pass) dominates the computation at $\sim$130ms per step, while NCS adds merely 1--3ms. Even with $K=1024$ noise candidates, the additional memory consumption is under 1 GB, which is insignificant compared to the memory footprint of DMs.

\section{Compression by NCS-DPS}
\label{app:ncs_compression}

Since we have unified various methods under the NCS framework, this implies that—aside from DDCM, which is a special case equivalent to NCS-MPGD, NCS-DPS can also be employed for compression tasks. To evaluate this, we conducted experiments on several images from the Kodak24 dataset, using the Stable Diffusion 2 model as the pre-trained backbone \citep{rombach2022highresolutionimagesynthesislatent}. We adopted a codebook of size 1024 but restricted the selection to a single optimal noise vector to maintain consistency with the experimental settings of the DDCM paper.

As shown in Fig.~\ref{fig:kodak24_samples}, the compression efficiency of NCS-DPS is inferior to that of DDCM. The compressed images appear overly smooth, exhibiting a significant loss of fine details. This degradation may stem from inaccuracies in the gradient used during the reverse computation. Nonetheless, when using fewer denoising steps, the images compressed by NCS-DPS appear to retain richer semantic structures.

\begin{figure}[H]
    \centering
    % Top row: titles
    \begin{tabular*}{0.83\linewidth}{@{\extracolsep{\fill}}ccccc}
        \textbf{Ground Truth} & \textbf{DPS} & \textbf{DAPS} & \textbf{Red-Diff} & \textbf{NCS-DPS} \\
    \end{tabular*}
    \vspace{0.3em}
    \\
    \includegraphics[width=0.9\linewidth]{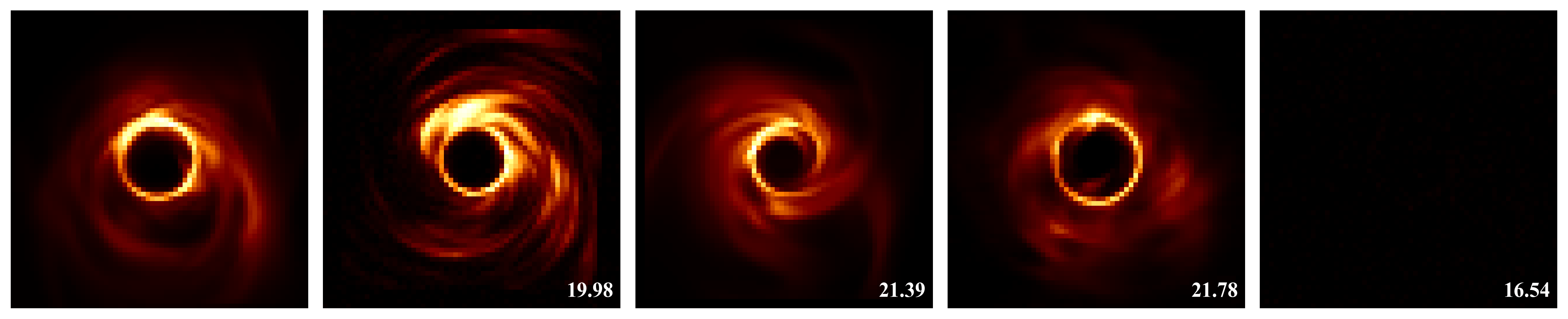}\\
    \vspace{0.5em}
    \includegraphics[width=0.9\linewidth]{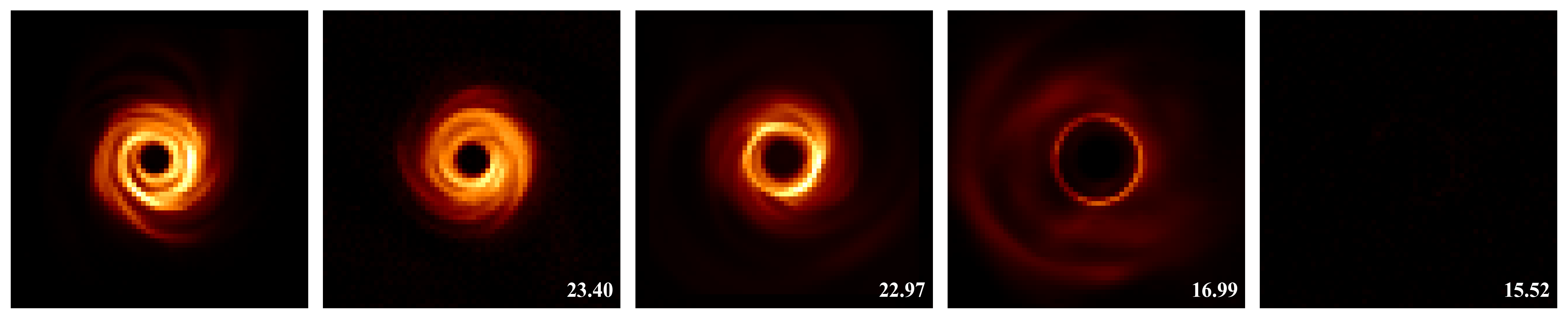}
    \caption{Failure cases of several methods. The top row shows, from left to right: Ground Truth, DPS, DAPS, Red-Diff, and NCS-DPS.}
    \label{fig:fail:image}
\end{figure}

\section{Failure Cases}

We observe that NCS-DPS can still encounter unstable behavior in certain cases, similar to DPS. In particular, during the sampling process, the reconstruction may collapse to overly dark images or become dominated by high-frequency noise, as illustrated in Fig.~\ref{fig:fail:image}. Empirically, this issue can often be substantially alleviated by slightly reducing the codebook size $K$, suggesting that the instability is sensitive to the strength of the injected conditional signal.

We hypothesize that this behavior is closely related to the inherent instability of gradient-based guidance in DPS-style methods, rather than being specific to NCS itself. In fact, when excluding these unstable samples, DPS-based methods achieve a significantly larger margin over competing approaches on scientific inverse problems in terms of PSNR. Addressing this instability more fundamentally remains an open problem, and we leave a systematic investigation of its causes and potential remedies for future work.

\end{document}